\documentclass[12pt]{article}
\usepackage[margin=1in]{geometry}
 
\usepackage{times,upgreek,morenotations,rotating}
\usepackage{graphicx} 
\usepackage{subfigure} 

\usepackage{natbib}

\usepackage{algorithm}
\usepackage{algorithmic}

\usepackage{makecell}
\usepackage{tabularx,cleveref}

\makeatletter
\newcommand\footnoteref[1]{\protected@xdef\@thefnmark{\ref{#1}}\@footnotemark}
\makeatother

\usepackage{tikz}
\usetikzlibrary{bayesnet}

\usepackage[framemethod=TikZ]{mdframed}
\mdfdefinestyle{MyFrame}{%
    linecolor=gray,
    outerlinewidth=2pt,
    roundcorner=10pt,
    innertopmargin=0.5\baselineskip,
    innerbottommargin=0.5\baselineskip,
    innerrightmargin=10pt,
    innerleftmargin=10pt}

\date{\today}

\author{
  Richard Nock\\
Nicta \& The Australian National University\\
  \texttt{richard.nock@nicta.com.au}\\
\and
  Rapha\"el Canyasse\\
Ecole Polytechnique \& The Technion\\
  \texttt{raphael.canyasse@polytechnique.edu}\\
\and
  Roksana Boreli\\
Nicta \& The University of New South Wales\\
  \texttt{roksana.boreli@nicta.com.au}\\
\and
  Frank Nielsen\\
Ecole Polytechnique \& Sony Computer Science Laboratories, Inc.\\
  \texttt{Frank.Nielsen@acm.org}\\
}

\begin{document} 

\title{\thetitle}

\maketitle 

\begin{abstract} 
$k$-means++ seeding has become a de facto standard
for hard clustering algorithms. In this paper, our first
contribution is a two-way generalisation of this seeding, $k$-\means, that includes the sampling of general densities
rather than just a discrete set of Dirac densities anchored at the point
locations, \textit{and} a generalisation of the
well known Arthur-Vassilvitskii (AV) approximation
guarantee, in the form of a \textit{bias+variance} approximation bound of
the \textit{global} optimum. This approximation exhibits a reduced dependency on the "noise"
component with respect to the optimal potential --- actually approaching the
statistical lower bound. We show that $k$-\means~\textit{reduces} to efficient (biased seeding) clustering algorithms tailored to specific frameworks; these include distributed, streaming and on-line clustering, with \textit{direct} approximation results for these algorithms. 
Finally, we present a novel 
application of $k$-\means~to differential privacy.
For either the specific frameworks considered here, or for the
differential privacy setting, there is little to no prior results on the
direct application of $k$-means++ and its approximation bounds --- 
state of the art contenders appear to be significantly more complex
and / or display less favorable (approximation) properties. 
We stress that our algorithms can still be run in
cases where there is \textit{no} closed form solution for the population minimizer. 
We demonstrate the applicability of our analysis via experimental
evaluation on several domains and settings, displaying competitive performances vs state of the art. 
\end{abstract} 

\section{Introduction}

Arthur-Vassilvitskii's (AV) $k$-means++ algorithm has been
extensively used to address the hard membership clustering problem, due to its simplicity, experimental performance and guaranteed approximation of the \textit{global} optimum; the goal being the $k$-partitioning of a dataset so as to minimize the sum
of within-cluster squared distances to the cluster center
\citep{avKM}, \textit{i.e.}, a centroid or a 
\textit{population minimizer} \citep{nnaOCD}.

The $k$-means++ non-uniform seeding approach has also been utilized in
more complex settings, including
tensor clustering, distributed, data stream, on-line and parallel clustering, clustering with
non-metric distortions and even clustering with distortions not 
allowing population
minimizers in closed form \citep{ajmSK,belDK,jsbAA,lssAA,nlkMB,nnTJ}. However, apart from the non-uniform seeding, all these algorithms are distinct and (seemingly) do not share many common properties.

Finally, the application of $k$-means++ in some scenarios is still an open research topic, due to the related constraints -- e.g., there is limited prior work in a differentially private setting \citep{nrsSS,wwsDP}. 

\noindent \textbf{Our contribution} --- In a nutshell, we describe a generalisation of the $k$-means++ seeding process, 
$k$-\means, which still delivers an efficient approximation of the
global optimum, and can be used to obtain \textit{and} analyze efficient algorithms for a wide range of settings, including: distributed, streamed, on-line clustering,
(differentially) private clustering, etc. . We proceed in two steps.

First, we describe $k$-\means~and analyze its approximation properties. We leverage two major
components of $k$-means++: (i) data-dependent \textit{probes} (specialized to observed data
in the $k$-means++) are used to compute the weights for selecting centers, and (ii) selection of centers is based on an \textit{arbitrary} family of
densities (specialized to Diracs in the $k$-means++). Informally, the approximation properties (when only
(ii) is considered), can be shown as:
\begin{eqnarray*}
\mbox{expected$\_$cost($k$-\means)} & \leq &
(2 + \log k) \cdot \Phi\:\:, \mbox{ with}
\end{eqnarray*}
$\Phi \defeq 6\cdot \mbox{optimal$\_$\textbf{noise-free}$\_$cost}
 + 2\cdot
  \mbox{noise$\_$(bias + variance)}$,
where ``noise'' refers to the family of densities (note that constants
are explicit in the bound). The dependence on
these densities is arguably smaller than
expectable (factor 2 for noise vs 6 for global optimum). There is also not much room for
improvement: we show that the guarantee approaches the
Fr\'echet-Cram\'er-Rao-Darmois lowerbound. 

Second, we use this general algorithm in two ways. We use it directly in a differential
privacy setting, addressing a conjecture of
\citep{nrsSS} with weaker assumptions. We also demonstrate the
use of this algorithm for a \textit{reduction} to other biased seeding
algorithms for distributed, streamed or on-line clustering, and obtain 
the approximation bounds for these
algorithms. This simple reduction technique allows us to analyze lightweight
algorithms that compare favorably to the state of the art in the related
domains \citep{ajmSK,belDK,lssAA}, from the approximation, assumptions and / or
complexity aspects. Experiments against state of the art
for the distributed and differentially private settings display that 
solid performance improvement can be obtained.

The rest of this paper is organised as follows: Section \ref{sec-km}
presents $k$-\means. Section \ref{sec-red} presents approximation
properties for distributed, streamed and on-line clustering that use a
reduction from $k$-\means. Section \ref{sec-dir} presents direct
applications of $k$-\means~to differential privacy. Section
\ref{sec-em}  presents experimental results. Last Section discusses
extensions (to more distortion measures) and conclude. In order not to
laden the paper's body, an Appendix, starting page \pageref{secappendix},
provides all proofs, extensive experiments and additional remarks on
the paper's content.

\section{$k$-\means}\label{sec-km}

\begin{figure*}[t]
\begin{mdframed}[style=MyFrame]
{\colorbox{gray!20}{\fbox{\textbf{Algorithm 0} $k$-\means}}}\\
 \textbf{Input}: data ${\mathcal{A}}\subset {\mathbb{R}}^d$
 with $|{\mathcal{A}}| = m$, $k\in {\mathbb{N}}_*$, densities
 $\left\{{\color{red} p_{(\ve{\mu}_{\ve{a}},\ve{\theta}_{\ve{a}})}}, \ve{a}\in
   {\mathcal{A}}\right\}$,  probe functions ${\color{red}\probe_t} : {\mathcal{A}}
 \rightarrow {\mathbb{R}}^d$ ($t\geq 1$);\\
Step 1: Initialise centers ${\mathcal{C}} \leftarrow \emptyset$; \\
Step 2: \textbf{for} $t=1, 2, ..., k$

\hspace{1.02cm} 2.1: randomly sample $\ve{a} \sim_{q_t}
{\mathcal{A}}$, with $q_1 \defeq u_m$ and, for $t>1$,
\begin{eqnarray}
q_t(\ve{a}) & \defeq & 
D_t(\ve{a})\left(\sum_{\ve{a}' \in {\mathcal{A}}}
  D_t(\ve{a}')\right)^{-1} \:\:, \mbox{ where } D_t(\ve{a}) \defeq \min_{\ve{x} \in {\mathcal{C}}}
\|{\color{red}\probe_t}(\ve{a})-\ve{x}\|_2^2\:\:;\label{choosea}
\end{eqnarray}

\hspace{1.02cm} 2.2: randomly sample $\ve{x} \sim {\color{red} p_{(\ve{\mu}_{\ve{a}},\ve{\theta}_{\ve{a}})}}$; 

\hspace{1.02cm} 2.3: ${\mathcal{C}} \leftarrow {\mathcal{C}} \cup
\left\{\ve{x} \right\}$; 

\textbf{Output}: ${\mathcal{C}}$;
\end{mdframed}
\end{figure*}

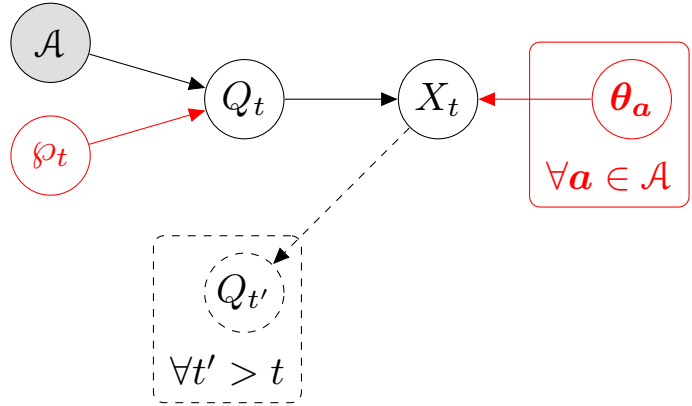
\begin{figure}
      \centering
 \begin{tikzpicture}[scale=1.5, transform shape]  %
        \node[latent] (distc) {$Q_t$} ; %
        \node[obs, left=of distc, yshift=0.5cm] (ls) {${\mathcal{A}}$} ; %
        \node[latentred, left=of distc, yshift=-0.5cm] (probe) {$\color{red}\probe_t$} ; %
        \node[latent, right=of distc] (xt) {$X_t$} ; %
       \node[latentred, right=of xt] (thetaj) {$\color{red}{\ve{\theta}}_{\ve{a}}$} ; %
      \node[latentdash, below=of distc] (dstep) {$Q_{t'}$} ; %
        \edge {ls} {distc} ; %
        \edgered {probe} {distc} ; %
        \edge {distc} {xt} ; %
         \edgered {thetaj} {xt} ; %
        \edgedash {xt} {dstep} ; %
  \platered {} { %
   (thetaj) %
  } {$\forall \ve{a} \in {\mathcal{A}}$} ; %
  \platedashed {} { %
   (dstep) %
  } {$\forall t' > t$} ; %
\end{tikzpicture}
\caption{Graphical model for the $k$-means++ seeding process
  (black) and our generalisation (black + red, best
  viewed in color).}\label{f-kmpp}
    \end{figure}

We consider the hard clustering problem \citep{bmdgCWj,nnaOCD}: given set ${\mathcal{A}} \subset
{\mathbb{R}}^d$ and integer $k>0$, find centers ${\mathcal{C}}\subset
{\mathbb{R}}^d$ which minimizes the $L_2^2$ potential to the centers
(here, $\ve{c}(\ve{a}) \defeq \arg\min_{\ve{c} \in {\mathcal{C}}} \|\ve{a} - \ve{c}\|_2^2$):
\begin{eqnarray}
\phi({\mathcal{A}} ; {\mathcal{C}}) & \defeq & \sum_{\ve{a} \in {\mathcal{A}}} \|\ve{a} - \ve{c}(\ve{a})\|_2^2\:\:,\label{costm}
\end{eqnarray}
Algorithm~0 describes $k$-\means. $u_m$ denotes the uniform
distribution over ${\mathcal{A}}$ ($|{\mathcal{A}}| = m$). The parenthood with
$k$-means++ seeding, which we name ``$k$-means++'' for short\footnote{Both
approaches can be completed with the same further local monotonous
optimization steps like Lloyd or Hartigan
iterations; furthermore, it is the biased seeding which holds the
approximation properties of $k$-means++.} \citep{avKM} can be best understood using Figure
\ref{f-kmpp} (the {\color{red}red} parts in Figure \ref{f-kmpp} are
pinpointed in
Algorithm~0). $k$-means++ is a random process that generates cluster centers
from observed data ${\mathcal{A}}$. It can be modelled using a
two-stage generative process for a mixture of Dirac distributions: the
first stage involves random
variable $Q_t\sim \mathrm{Mult}(m, \ve{\piup}_t)$ whose parameters
$\ve{\piup}_t \in \bigtriangleup_m$ (the $m$-dim probability simplex) are computed from
the data and previous centers; sampling $Q_t$ chooses the Dirac
distribution, which is then ``sampled'' for one center (and the process iterates). All the crux of the technique is the design of
$\ve{\piup}_t$, which, under \textit{no} assumption of the data,
yield in expectation a $k$-means potential for the centers chosen that
is within $8(2+\log k)$ of the global optimum \citep{avKM}.

$k$-\means~generalize the process in two ways: first, the update of
$\ve{\piup}_t$ depends on data and previous \textit{probes}, using a
sequence of \textit{probe
  functions} $\probe_t : {\mathcal{A}}\rightarrow {\mathbb{R}}^d$ ($\probe = \mathrm{Id}, \forall t$ in $k$-means++). Second, Diracs are replaced by arbitrary but
fixed \textit{local} (sometimes also called \textit{noisy}) distributions with parameters\footnote{Because expectations are the major
  parameter for clustering, we
split the parameters in the form of $\ve{\mu}_{\ve{a}}$
(expectation) and $\ve{\theta}_{\ve{a}}$ (other parameters,
\textit{e.g.} covariance matrix).}
$(\ve{\mu}_{\ve{a}}, \ve{\theta}_{\ve{a}})$ that depend
on ${\mathcal{A}}$.

Let
${\mathcal{C}}_{\opt} \subset {\mathbb{R}}^d$ denote the set of $k$
centers minimizing (\ref{costm}) on ${\mathcal{A}}$. Let
$\ve{c}_{\opt}(\ve{a}) \defeq \arg\min_{\ve{c}\in {\mathcal{C}}_{\opt}}
\|\ve{a} - \ve{c}\|_2^2$ ($\ve{a}\in {\mathcal{A}}$), and
\begin{eqnarray}
\phi_{\opt} & \defeq & \sum_{\ve{a} \in {\mathcal{A}}} \|\ve{a} - \ve{c}_{\opt}(\ve{a})\|_2^2\:\:,\label{defopt}\\
\phi_{\bias} & \defeq & \sum_{\ve{a} \in {\mathcal{A}}} \|\ve{\mu}_{\ve{a}} - \ve{c}_{\opt}(\ve{a})\|_2^2\:\:,\label{defopt2}\\
\phi_{\variance} & \defeq & \sum_{a\in {\mathcal{A}}} \trace{\Sigma_{\ve{a}}}\:\:.\label{defopt3}
\end{eqnarray}
$\phi_{\opt}$ is the optimal \textbf{noise-free} potential,
$\phi_{\bias}$ is the bias of the noise\footnote{We term it
  \textit{bias} by analogy with supervised classification, considering
  that the expectations of the densities could be used as models for
  the cluster centers \citep{kwBP}.}, and $\phi_{\variance}$ its
variance, with
$\Sigma_{\ve{a}}\defeq \expect_{\ve{x} \sim
  p_{\ve{a}}}[(\ve{x}-\ve{\mu}_{\ve{a}})(\ve{x}-\ve{\mu}_{\ve{a}})^\top]$
the covariance matrix of $p_{\ve{a}}$. 
Notice that when $\ve{\mu}_{\ve{a}} = \ve{a}$, $\phi_{\bias} =
\phi_{\opt}$. Otherwise, it \textit{may} hold that
$\phi_{\bias} < \phi_{\opt}$, and even $\phi_{\bias} = 0$ if 
expectations coincide with ${\mathcal{C}}_{\opt}$. Let $C_{\opt}$
denote the partition of ${\mathcal{A}}$ according to the centers in
${\mathcal{C}}_{\opt}$. 
We say that probe function $\probe_t$ is $\upeta$-stretching if,
informally, replacing points by their probes does not distort significantly
the observed potential of an optimal cluster, with respect to its
actual optimal potential. The formal definition follows.
\begin{definition}\label{defstretch}
Probe functions $\probe_t$ are said $\upeta$-stretching on ${\mathcal{A}}$, for some
$\upeta \geq 0$, iff the following holds: for any cluster $A \in
C_{\opt}$ and any
$\ve{a}_0\in A$ such that $\phi(\probe_t(A); \{\probe_t(\ve{a}_0)\})
\neq 0$, for any set
of at most $k$ centers ${\mathcal{C}} \subset {\mathbb{R}}^d$,
\begin{eqnarray}
\frac{\phi(A ; {\mathcal{C}})}{\phi(A; \{\ve{a}_0\})} & \leq &
(1+\upeta) \cdot \frac{\phi(\probe_t(A) ;
  {\mathcal{C}})}{\phi(\probe_t(A); \{\probe_t(\ve{a}_0)\})},
\forall t\:\:. \label{defeta}
\end{eqnarray}
\end{definition}
Since $\phi(A ; {\mathcal{C}}_{\opt}) =
\sum_{\ve{a}_0 \in A} \phi(A; \{\ve{a}_0\})$ \citep{avKM} (Lemma
3.2), Definition \ref{defstretch} roughly states that the
potential of an optimal cluster with respect to a set of cluster
centers, relatively to its potential with respect to the optimal set
of centers, does not blow up through probe function $\probe_t$.
The identity function is trivially $0$-stretching, for any
${\mathcal{A}}$. Many local transformations would be eligible for
$\upeta$-stretching probe functions with $\upeta$ small, including local
translations, mappings to core-sets \citep{hmOC}, mappings to Voronoi
diagram cell centers \citep{bnnBV}, etc. Notice that ineq. (\ref{defeta}) has to hold only for optimal clusters and
\textit{not} any clustering of ${\mathcal{A}}$.
Let
$\expect[\phi({\mathcal{A}} ; {\mathcal{C}})] \defeq \int \phi({\mathcal{A}} |
{\mathcal{C}}) \mathrm{d}p({\mathcal{C}})$ denote the expected potential over the
random sampling of ${\mathcal{C}}$ in $k$-\means.
\begin{theorem}\label{thmain} 
For any dataset ${\mathcal{A}}$, any
  sequence of
  $\upeta$-stretching probe functions $\probe_t$ and any density
  $\{p_{\ve{a}}, \ve{a}\in {\mathcal{A}}\}$, the expected potential of
  $k$-\means~satisfies:
\begin{eqnarray}
\expect[\phi({\mathcal{A}} ; {\mathcal{C}})]
  & \leq & (2 + \log
k)\cdot \Phi\:\:,\label{boundsup}
\end{eqnarray}
with $\Phi \defeq (6+4\upeta)\phi_{\opt} + 2\phi_{\bias} +
2 \phi_{\variance}$.
\end{theorem}
(Proof in page \pageref{sec-proof-thmain})
Five remarks are in order. First, we retrieve the result of
\citep{avKM} in their setting ($\upeta=\phi_{\variance} = 0$,
$\phi_{\bias} = \phi_{\opt}$). Second, in the case where 
$\phi_{\bias} < \phi_{\opt}$, we \textit{may} beat AV's bound. This is not due to an improvement of the algorithm,
but to a finer analysis which shows that special settings may
``naturally'' favor the improvement. We shall see one example in the
distributed clustering case.
Third, apart from being
$\upeta$-stretching, there is no constraint on the
choice of probe functions $\probe_t$: it can be randomized, iteration
dependent, etc.
Fourth, the algorithm can easily be
generalized to the case where points are weighted. Last, as we show
in the following Lemma, the dependence in noise in
ineq. (\ref{boundsup}) can hardly be improved in our framework.
\begin{lemma}\label{lembinf}
Suppose each point in ${\mathcal{A}}$ is replaced (i.i.d.) by a point
sampled in $p_{\ve{a}}$ with $\Sigma_{\ve{a}} = \Sigma$. Then any clustering algorithm suffers: $\expect[\phi({\mathcal{A}} ;
{\mathcal{C}})] = \Omega(|{\mathcal{A}}| \trace{\Sigma})$.
\end{lemma}
(Proof in page \pageref{proof_lembinf})
We make use of $k$-\means~in two different ways. First, we show that
it can be used to prove approximation properties for 
algorithms operating in different clustering settings: distributed clustering,
streamed clustering and on-line clustering. The proof involves a
\textit{reduction} (see page \pageref{proof_all_thDKM}) from
$k$-\means~to each of these algorithms. By reduction, we mean there
exists distributions and probe functions (even non poly-time
computable) for which $k$-\means~yields the same result in expectation
as the other algorithm, thus directly yielding an approximability ratio of the
global optimum for this latter algorithm via Theorem \ref{thmain}. Second, we show how $k$-\means~can
\textit{directly} be specialized to address settings for which no efficient
application of $k$-means++ was known.

\section{Reductions from $k$-\means}\label{sec-red}

\begin{table}[t]
\begin{center}
{\tiny
\begin{tabular}{ccccc}\\ \hline \hline
 & Ref. & Property & Them & Us\\ \hline
(1) & \citep{bmvkvSK} & Communication complexity & $O(n^2 \ell \cdot \log
\phi_1)$ (expected) & $O(n^2k)$ \\
 (2) & \citep{bmvkvSK} & $\#$ data to compute one center & $m$ & $\leq
 \max_{i\in [n]}(m/m_i)$ \\
(3) & \citep{bmvkvSK} & Data points shared & $O(\ell \cdot \log \phi_1)$
 (expected) & $k$ \\
 (4) & \citep{bmvkvSK} & Approximation bound & $O((\log k) \cdot \phi_{\opt})$
 & $(2 + \log
k)\cdot \left(10\phi_{\opt} + 6\phi^F_{s}\right)$ \\ \hline
 (I) & \citep{belDK} & Communication complexity & $\Omega( (nkd/\varepsilon^4) + n^2k\ln
(nk))$ & $O(n^2k)$\\ 
(II) & \citep{belDK} & Data points shared & $\Omega( (kd/\varepsilon^4) + nk\ln
(nk))$ & $k$\\ 
(III) & \citep{belDK} & Approximation bound & $(2+\log k) (1+\varepsilon) \cdot 8 \phi_{\opt}$ & $(2 + \log
k)\cdot \left(10\phi_{\opt} + 6\phi^F_{s}\right)$\\ \hline
(i) & \citep{ajmSK} & Time complexity (outer loop) &
\multicolumn{2}{c}{--- identical ---}\\ 
(ii) & \citep{ajmSK} & Approximation bound & $(2+\log k) (1+\upeta)
\cdot 32 \phi_{\opt}$ & $(2+\log k) \cdot ( (8+4\upeta) \phi_{\opt} +
2\phi^{\probe}_s)$\\ \hline
(a) & \citep{lssAA} & Knowledge required & Lowerbound
$\phi^* \leq \phi_{\opt}$ & None\\ 
(b) & \citep{lssAA} & Approximation bound & $O(\log m \cdot \phi_{\opt})$ & $(2 + \log
k)\cdot \left(4+ (32/\varsigma^2)\right)\phi_{\opt}$\\ \hline
(A) & \citep{nrsSS} & Knowledge required & 
$\uplambda(\phi_{\opt})$ & None\\ 
(B) & \citep{nrsSS} & Noise variance ($\sigma$) & $O(\uplambda k
R/\epsilon)$ & $O(R/(\epsilon + \log m))$\\ 
(C) & \citep{nrsSS} & Approximation bound & $O^*(\phi_{\opt} +
m\uplambda^2kR^2/\epsilon^2)$ & $O(\log k (\phi_{\opt} +
mR^2/(\epsilon + \log m)^2))$\\ \hline
($\alpha$) & \citep{wwsDP} & Assumptions on $\phi_{\opt}$ & 
Several (separability, size of clusters, etc.) & None\\ 
($\beta$) & \citep{wwsDP} & Approximation bound & $O^*(\phi_{\opt} +
km\log(m) R^2/\epsilon^2)$ & $O(\log k (\phi_{\opt} +
mR^2/(\epsilon + \log m)^2))$\\ \hline \hline
\end{tabular}
}
\end{center}
\caption{Comparison with state of the art approaches for distributed
  clustering (1-4, I-III), streamed clustering (i, ii), on-line clustering
  (a, b) and differential privacy (A-C, $\alpha$, $\beta$). Notations
  used for the "Them" column are as follows. $\phi_1$ is the expected
  potential of a clustering with a single cluster over the
  \textit{whole} data and $\ell$ is in general $\Omega(k)$
  \citep{bmvkvSK}. $\varepsilon$ is the coreset approximation factor in
  \citep{belDK}. $\upeta$ is the approximation factor of the optimum in
  \citep{ajmSK}. $\uplambda$ is the separability factor in Definition
  5.1 in \citep{nrsSS}.}
  \label{tcomp}
\end{table}

\begin{algorithm}[t]
\caption{\protectedKMPP~(// \privateKMPP)}\label{algoDKM}
\begin{algorithmic}
\STATE  \textbf{Input:} Forgy nodes $(\F_i,
{\mathcal{A}}_i), i \in [n]$,
\STATE  \textbf{for} $t = 1, 2, ..., k$ 
\STATE  \hspace{0.04cm} Round 1 : $\N^*$ picks $i^* \sim_{q^D_t} [n]$
and asks $\F_{i^*}$ for a center;

\STATE  \hspace{0.04cm} Round 2 : $\F_{i^*}$ picks $\ve{a}
\sim_{u_{i^*}} {\mathcal{A}}_{i^*}$ and sends $\ve{a}$ to $\F_{i},
\forall i$;

\hspace{0.8cm} // \privateKMPP: $\F_{i^*}$ sends $\ve{x} \sim
p_{(\ve{\mu}_{\ve{a}},\ve{\theta}_{\ve{a}})}$ to $\F_{i},
\forall i$;

\STATE  \hspace{0.04cm} Round 3 : $\forall i, \F_{i}$ updates $D_t({\mathcal{A}}_i)$ and sends it to $\N^*$;
\STATE \textbf{Output:} ${\mathcal{C}} = $ set of broadcasted
$\ve{a}$s (or $\ve{x}$s);
\end{algorithmic}
\end{algorithm}

Despite tremendous advantages, $k$-means++
has a serious downside: it is difficult to
parallelize, distribute or stream it under relevant communication, space, privacy
and/or time resource constraints \citep{bmvkvSK}. Although extending 
$k$-means clustering to these settings has been a major research area in recent 
years, there has been no obvious solution to tailoring
$k$-means++ 
\citep{almrssSA,ajmSK,bmvkvSK,belDK,lssAA,swmFA} (and others).

\paragraph{Distributed clustering} We consider horizontally
partitioned data among \textit{peers}, in line with \citep{bmvkvSK}, and a
setting significantly more
restrictive than theirs: each peer can only
locally run the standard operations of Forgy initialisation (that is, uniform random seeding) on its own
data, unlike for example the biased distributions
of \citep{bmvkvSK}. This is consistent with the notion that data
handling peers are not necessarily computationally
intensive resources. Additionally, due to privacy constraints, we limit the data sharing between nodes.
We denote the nodes handling the
data \textit{Forgy nodes}. We have $n$ such nodes, $(\F_i,
{\mathcal{A}}_i), i \in [n]$, where ${\mathcal{A}}_i$ is the dataset
held by $\F_i$. To enable more complex operations necessary to implement $k$-\means, we introduce a special node, $\N^*$, that has high computation power, \textit{but} is not allowed to handle 
\textit{any} data (points) from the Forgy nodes. We therefore split the
location of the computational power from the location of the data. We
also prevent the Forgy nodes from exchanging \textit{any} data between
themselves, with the sole exception of cluster centers. We note that none of the
algorithms of \citep{ajmSK,belDK,bmvkvSK} would be applicable to this
setting without non-trivial modifications affecting their properties.

Algorithm \ref{algoDKM}
defines the mechanism that is consistent with our setting. It includes two variants: a protected version \protectedKMPP~where Forgy nodes directly share local centers and a private version \privateKMPP~where the nodes share noisy centers,
such as to ensure a differentially private release of centers (with
relevant noise calibration). Notations used in Algorithm
\ref{algoDKM} are as follows. Let
$D_t({\mathcal{A}}_i) \defeq \sum_{\ve{a} \in {\mathcal{A}}_i}
D_t(\ve{a})$ and $q^D_{ti} \defeq D_t({\mathcal{A}}_i) \cdot (\sum_{j}
  D_t({\mathcal{A}}_j))^{-1}$ if $t>1$ and $q^D_{ti} \defeq 1/n$ otherwise.
Also, $u_i$ is uniform distribution on $[m_i]$, with $m_i
\defeq |{\mathcal{A}}_i|$. 

\begin{theorem}\label{thDKM}
Let $\phi^F_s \defeq \sum_{i \in [n]}\sum_{\ve{a} \in
  {\mathcal{A}}_i} \|\ve{c}({\mathcal{A}}_i) - \ve{a} \|_2^2$ 
be the total \textsl{spread} of the Forgy nodes
($\ve{c}({\mathcal{A}}_i) \defeq (1/m_i)\cdot \sum_{{\mathcal{A}}_i}\ve{a}$). At
iteration $k$, the expected potential on the \textbf{total} data ${\mathcal{A}} \defeq \cup_i
{\mathcal{A}}_i$ satisfies ineq. (\ref{boundsup}) with
\begin{eqnarray}
\Phi & \defeq & 
\left\{\begin{array}{ll}
10\phi_{\opt} + 6\phi^F_{s} & (\mbox{\protectedKMPP})\\
10\phi_{\opt} + 4\phi^F_{s} + 2\phi_{\variance} &
(\mbox{\privateKMPP})
\end{array}\right. \:\:.\label{boundsupD}
\end{eqnarray}
Here, $\phi_{\opt}$ is the
optimal potential on \textbf{total} data ${\mathcal{A}}$.
\end{theorem}
(Proof in page \pageref{proof_thDKM}) We note that the optimal
potential is defined on the total data. The dependence on $\phi^F_{s}$, which is just the peer-wise
variance of data, is thus rather intuitive. A positive point is that $\phi^F_{s}$ is
weighted by a factor smaller than the factor that weights the optimal
potential. Another positive point is that this parameter \textit{can} be
computed from data, and among peers, without disclosing more
data. Hence, it may be possible to estimate the loss against the
centralized, $k$-means++ setting, taking as reference eq. (\ref{boundsupD}).
To gain insight in the leverage that Theorem \ref{thDKM} provides, 
Table \ref{tcomp} compares \protectedKMPP~to \citep{belDK}'s ($\varepsilon$ is the coreset approximation parameter),
even though the latter approach would not be applicable to our restricted
framework. To be fair, we
assume that the algorithm used to cluster the coreset in
\citep{belDK} is $k$-means++. We note that, considering the communication
complexity and the number of data points shared, Algorithm \ref{algoDKM}
is a clear winner. In fact, Algorithm \ref{algoDKM} can also win from the approximability
standpoint. The dependence in $\varepsilon$ prevents to fix it too
small in \citep{belDK}. Comparing the
bounds in row (III) shows that if $\varepsilon > 1/4$, then we can also be
better from the approximability standpoint if the spread satisfies
$\phi^F_{s} = O (\phi_{\opt})$. While this may not be feasible over arbitrary
data, it becomes more realistic on several real-world scenarii, when
Forgy nodes aggregate ``local'' data with respect to features, \textit{e.g.}, state-wise insurance data, city-wise financial
data, etc. When $n$ increases, this also becomes more realistic.

\paragraph{Streaming clustering} 
\begin{algorithm}[t]
\caption{\SKM}\label{algoSKM}
\begin{algorithmic}
\STATE  \textbf{Input:} Stream $\stream$
\STATE Step 1: ${\mathcal{S}} \defeq \{(\ve{s}_j, m_j),
i\in [n]\}
\leftarrow \PA(\stream,n)$;
\STATE Step 2: \textbf{for} $t = 1, 2, ..., k$ 
\STATE  \hspace{0.99cm} 2.1: \textbf{if} $t=1$ then let $\ve{s}_j \sim_{u_n}
{\mathcal{S}}$ \textbf{else} $\ve{s}_j \sim_{q_t^S}
{\mathcal{S}}$ s.t.
\begin{eqnarray}
q^S_t(\ve{s}_j) & \defeq & m_j D_t(\ve{s}_j)\left(\sum_{j' \in [n]} m_{j'} D_t(\ve{s}_{j'})\right)^{-1}\label{qw}\:\:;
\end{eqnarray}
\STATE  \hspace{1.65cm} // $D_t(\ve{s}_j) \defeq \min_{\ve{c}\in
  {\mathcal{C}}} \|\ve{s}_j - \ve{c}\|_2^2$;
\STATE  \hspace{0.99cm} 2.2: ${\mathcal{C}} \leftarrow {\mathcal{C}} \cup \{\ve{s}_{j}\}$;
\STATE \textbf{Output:} Cluster centers ${\mathcal{C}}$;
\end{algorithmic}
\end{algorithm}

We have access to a stream $\stream$, with an assumed finite size:
$\stream$ is a sequence of points $\ve{a}_1, \ve{a}_2, ...,
\ve{a}_m$. We authorise the computation / output of the clustering at
the end of the stream, \textit{but} the memory $n$ allowed for all
operations satisfies $n < m$, such as $n = m^\alpha$ with $\alpha<1$
in \citep{ajmSK}. We assume for simplicity that each point can be stored in one storage memory unit.
Algorithm \ref{algoSKM} (\SKM)~presents our
approach. It relies on the standard ``trick'' of summarizing massive
datasets via compact representations (synopses) before processing them
\citep{immmCC}. The approximation properties of \SKM, proven using a
reduction from $k$-\means, hold regardless of the way synopses
are built. They show that two key parameters
may guide its choice: the spread of the synopses, analogous
to the spread of Forgy nodes for distributed clustering, and the
stretching properties of the synopses used as centers.

\begin{theorem}\label{thSKM}
Let $\probe(\ve{a}) \defeq \arg\min_{\ve{s}' \in {\mathcal{S}}}
\|\ve{a}-\ve{s}'\|_2^2, \forall \ve{a} \in \stream$. Let $\phi^\probe_s \defeq \sum_{\ve{a} \in \stream}
\|\probe({\ve{a}}) - \ve{a}\|_2^2$ be the \textsl{spread} of $\probe$
on synopses set ${\mathcal{S}}$.
Let $\upeta>0$
such that  $\probe$ is $\upeta$-stretching on $\stream$. Then the expected potential
of $\SKM$ on stream $\stream$ satisfies ineq. (\ref{boundsup}) with 
\begin{eqnarray}
\Phi & \defeq & (8+4\upeta)\phi_{\opt} +
2\phi^{\probe}_s\:\:,\nonumber
\end{eqnarray}
Here, $\phi_{\opt}$ is the optimal potential on \textbf{stream}
$\stream$.
\end{theorem}
(Proof in page \pageref{proof_thSKM}) 
It is not surprising to see that $\SKM$ looks like a generalization of \citep{ajmSK} and almost matches it (up to the number
of centers delivered) when $k'\gg k$ synopses are learned
from $k'$-means{\tiny $\#$}. Yet, we rely on a different --- and more general --- analysis of its
approximation properties. Table \ref{tcomp}
compares properties of \SKM~to \citep{ajmSK} ($\upeta$ relates to approximation of the
$k$-means objective in inner loop).

\begin{algorithm}[t]
\caption{\OKM}\label{algoOKM}
\begin{algorithmic}
\STATE  \textbf{Input:} Minibatch $\stream_j$, current weighted centers ${\mathcal{C}}$;
\STATE Step 1:  \textbf{if} $j=1$ then let $\ve{s}\sim_{u_1}
\stream_1$ \textbf{else} $\ve{s}\sim_{q^O_j} \stream_j$ s.t.
\begin{eqnarray}
q_j^O(\ve{s}) & \defeq & D_t(\ve{s})\left(\sum_{\ve{s}'\in \stream_j} D_t(\ve{s}')\right)^{-1}\label{qwol}\:\:;
\end{eqnarray}
\STATE  \hspace{0.99cm}// $D_t(\ve{s}) \defeq \min_{\ve{c}\in
  {\mathcal{C}}} \|\ve{s} - \ve{c}\|_2^2$;
\STATE  Step 2: ${\mathcal{C}} \leftarrow {\mathcal{C}} \cup \{\ve{s}\}$;
\end{algorithmic}
\end{algorithm}

\paragraph{On-line clustering} This setting is probably
the farthest from the original setting of the $k$-means++ algorithm. Here, points
arrive in a sequence, finite, but of unknown size and too large to fit
in memory \citep{lssAA}. We make no other assumptions -- the sequence can be
random, or chosen by an adversary. Therefore, the expected analysis we make is only
with respect to the internal randomisation of the algorithm, \textit{i.e.}, for
the fixed stream sequence as it is observed. We do not assume a
feedback for learning (common for supervised learning); so,
we do not assume that
the algorithm has to predict a cluster for each point that arrives,
yet it has to be easily modifiable to do so. 

Our approach
is summarized in Algorithm \ref{algoOKM} (\OKM), a variation of
$k$-means++ which consists of splitting
the stream $\stream$ into
minibatches $\stream_j$ for $j=1, 2, ...$,
each of which is used to sample one center. 
$u_1$ denotes the uniform
distribution with support $\stream_1$.
Let $R \defeq
\max_{\ve{a}, \ve{a}' \in \stream} \|\ve{a}-\ve{a}'\|_2 (\ll \infty)$ be the
diameter of $\stream$.
\begin{theorem}\label{thOKM}
Let $\varsigma>0$ be
the largest real such that the following conditions are met (for any $A \in
C_{\opt}, j\geq 1$): for any set of at most $k$ centers ${\mathcal{C}}$, $\sum_{\ve{a}, \ve{a}' \in A} 
\|\ve{a}-\ve{a}'\|_2^2 \geq \varsigma \cdot
{|A|\choose 2} R^2$ and $\sum_{\ve{a}\in
  A\cap \stream_j} {\|\ve{a} - \ve{c}(\ve{a})\|_2^2} \geq \varsigma
\cdot \sum_{\ve{a}\in
  A} {\|\ve{a} - \ve{c}(\ve{a})\|_2^2} $ (with $\ve{c}(\ve{a})$
defined in eq. (\ref{costm})).
Then the expected potential
of \OKM on stream $\stream$ satisfies ineq. (\ref{boundsup}) with 
\begin{eqnarray}
\Phi & \defeq & \left(4+\frac{32}{\varsigma^2}\right) \cdot
\phi_{\opt}\:\:,\nonumber
\end{eqnarray}
where $\phi_{\opt}$ is the optimal potential on \textbf{stream} $\stream$.
\end{theorem}
(Proof in page \pageref{proof_thOKM})
Notice that loss function $\phi(\stream, {\mathcal{C}})$ in
eq. (\ref{costm}) implies the finiteness of $\stream$, and the
existence of $\varsigma>0$; also, the second condition implies
$\varsigma \leq 1$. In \citep{lssAA}, 
the clustering algorithm is required to have space and time at most
polylog in the length of the stream. Hence, each minibatch
can be reasonably large with respect to the stream --- the larger they
are, the larger $\varsigma$. 
The knowledge of $\varsigma$ is not
necessary to run \OKM; it is just a part of the approximation bound
which quantifies the loss in approximation due to the fact that
centers are computed from the \textit{partial} knowledge of the stream. Table \ref{tcomp}
compares properties of \OKM~to \citep{lssAA} (we picked the fully
on-line, non-heuristic algorithm). To compare the bounds, suppose that
batches have the same size, $b$, so that $\log k = \log(m/b)$. If
batches are at least polylog size, up to
what is hidden in the big-Oh notation, our
approximation can be quite competitive when $\varsigma$ is large, \textit{e.g.}, if $d$ is large and optimal
clusters are not too small.

\section{Direct use of $k$-\means}\label{sec-dir}

The most direct application domain of
$k$-\means~is differential privacy.
Several algorithms have independently emphasised the idea that
powerful mechanisms may be amended via a carefully designed noise mechanism
to broaden their scope with new capabilities, without overly challenging their original properties. 
Examples abound \citep{hpTN,kvEA,cmsDP,clAN}, etc. Few approaches are
related to clustering, yet noise injected is big --- the existence of
a smaller, sufficient noise, 
was conjectured in \citep{nrsSS} --- and approaches rely on a variety
of assumptions or knowledge about the optimum (See Table \ref{tcomp})
\citep{nrsSS,wwsDP}. 
To apply $k$-\means, we consider that $\probe_t = \mathrm{Id}, \forall
t$, and assume $0<R\ll\infty$ s.t. $\max_{\ve{a}, \ve{a}' \in
  {\mathcal{{A}}}} \|\ve{a}-\ve{a}'\|_2  \leq  R $ (a current
assumption in the field \citep{drTA}). 

\paragraph{A general likelihood ratio bound for $k$-\means}\label{sec-lr}

We show that the likelihood ratio of the same clustering for two
``close'' instances is governed by two quantities that rely on the
neighborhood function. Most importantly for differential privacy, when
densities $p_{(\ve{\mu}_{\ve{a}}, \ve{\theta}_{\ve{a}})}$ are
carefully chosen, this ratio \textit{always} $\rightarrow 1$ as a function
of $m$, which is highly desirable for differential privacy.  We let $\nn_{\mathcal{N}}(\ve{a}) \defeq
\arg\min_{\ve{a}' \in {\mathcal{N}}}\|\ve{a}-\ve{a}'\|_2$ denote the
nearest neighbour of $\ve{a}$ in ${\mathcal{N}}$, and let $\ve{c}(A) \defeq (1/|A|) \cdot \sum_{\ve{a} \in A}
\ve{a}$.
\begin{definition} \label{defspread}
We say that neighborhood in ${\mathcal{A}}$ is
  $\updeltaw$-spread for some $\updeltaw>0$ iff for any
  ${\mathcal{N}} \subseteq {\mathcal{A}}$ with $| {\mathcal{N}}| = k-1$, and any ${\mathcal{B}} \subseteq
  {\mathcal{A}}$ with $| {\mathcal{B}}| = |{\mathcal{A}}| - 1$,
\begin{eqnarray}
\sum_{\ve{a} \in {\mathcal{B}}} \|\ve{a} - \nn_{\mathcal{N}}(\ve{a})\|_2^2 & \geq &
\frac{R^2}{\updeltaw}\:\:.\label{assum1eq}
\end{eqnarray}
\end{definition}
\begin{definition} \label{defmon}
We say that neighborhood in ${\mathcal{A}}$ is
  $\updeltas$-monotonic for some $\updeltas>0$ iff the following holds. $\forall {\mathcal{N}} \subseteq {\mathcal{A}}$ with $| {\mathcal{N}}| \in
  \{1, 2, ..., k-1\}$, for any $A \subseteq
  {\mathcal{A}}\backslash {\mathcal{N}}$ which is ${\mathcal{N}}$-packed, we have:
\begin{eqnarray}
\lefteqn{\sum_{\ve{a} \in {\mathcal{A}}} \|\ve{a} - \nn_{\mathcal{N}}(\ve{a})\|_2^2}\nonumber\\
& \leq & (1+\updeltas) \cdot \sum_{\ve{a} \in {\mathcal{A}}} \|\ve{a} -
\nn_{\mathcal{N}\cup \{\ve{c}(A)\}}(\ve{a})\|_2^2\:\:.\label{defpack}
\end{eqnarray}
Set $A$ is said ${\mathcal{N}}$-packed iff there exists
  $\ve{x}\in {\mathbb{R}}^d$ satisfying $\ve{x} =
\arg\min_{\ve{c} \in {\mathcal{N}}\cup \{\ve{x}\}} \|\ve{a} -
\ve{c}\|_2^2$, $\forall \ve{a}\in A$.
\end{definition}
It is worthwhile remarking that as long as $k<|{\mathcal{A}}|\ll
\infty$, both $0<\updeltaw\ll \infty$ and
$0<\updeltas\ll \infty$ always exist. Informally, $\updeltaw$ 
brings that the sum of squared distances to any subset of $k-1$
centers in ${\mathcal{A}}$ must not be negligible against the diameter
$R$.
$\updeltas$ yields a statement a bit more technical, but it roughly
reduces to stating that 
adding one center to any set of at most $k-1$ points that are already
close to each other should not
decrease significantly the overall potential to the set of
centers. Figure \ref{figA1} provides a schematic view of the property,
showing that the modifications of the potential can be very local,
thus yielding small $\updeltas$ in ineq. (\ref{defpack}). 
\begin{figure}[t]
\centering
\begin{tabular}{c} 
\includegraphics[width=0.40\linewidth]{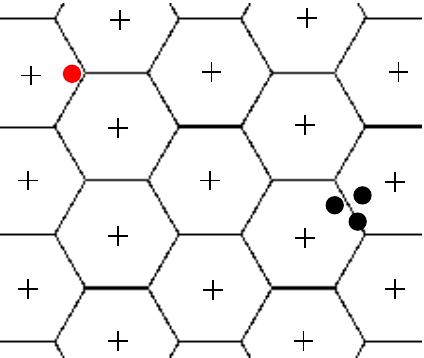} 
\end{tabular}
\caption{Checking that $\updeltas$ is small, for ${\mathcal{N}}$
  the set of crosses ({\textbf{+}}). Any set $A$ of points close to each other, such as
  the black dots ($\bullet$), would be ${\mathcal{N}}$-packed (pick $\ve{x} =
  \ve{c}(A)$ in this case), but would fail to be
  ${\mathcal{N}}$-packed if too spread (\textit{e.g.}, red dot
  ({\color{red}$\bullet$}) plus black dots). Segments depict the
  Voronoi diagram of ${\mathcal{N}}$. Best viewed in color.}
\label{figA1}
\end{figure}
The following Theorem uses the definition of neighbouring samples: 
samples ${\mathcal{A}}$ and ${\mathcal{A}}'$ are 
neighbours, written ${\mathcal{A}}\approx {\mathcal{A}}'$, iff
they differ by one point. We also define $\pr[{\mathcal{C}} |
{\mathcal{A}}]$ to be the density of output ${\mathcal{C}}$ given
input data ${\mathcal{A}}$.
\begin{theorem}\label{thmdp}
Fix $\probe_t = \mathrm{Id}$ ($\forall t$) and densities $p_{(\ve{\mu}_{.}, \ve{\theta}_{.})}$ having the same
support $\Omega$ in $k$-\means. 
Suppose there exists $\varrho(R)>0$ such that densities
$p_{(\ve{\mu}_{.}, \ve{\theta}_{.})}$ satisfy the following pointwise likelihood ratio constraint: 
\begin{eqnarray}
\frac{p_{(\ve{\mu}_{\ve{a}'}, \ve{\theta}_{\ve{a}'})}( \ve{x})}{p_{(\ve{\mu}_{\ve{a}}, \ve{\theta}_{\ve{a}})}( \ve{x})} & \leq &
\varrho(R)\:\:,\forall {\ve{a}}, {\ve{a}'} \in {\mathcal{A}}, \forall
\ve{x}\in \Omega\:\:.\label{pwc}
\end{eqnarray}
Then, there exists a function $f(.)$ such that, for any $\updeltaw, \updeltas > 0$ such that ${\mathcal{A}}$ is
$\updeltaw$-spread and $\updeltas$-monotonic, for any ${\mathcal{A}}'
\approx {\mathcal{A}}$, for any $k>0$ and any ${\mathcal{C}}
\subset \Omega$ of size $k$ output by Algorithm $k$-\means~on
whichever of ${\mathcal{A}}$ or ${\mathcal{A}}'$, the likelihood ratio
of ${\mathcal{C}}$ given ${\mathcal{A}}$ and ${\mathcal{A}}'$ is
upperbounded as:
\begin{eqnarray}
\frac{\pr[{\mathcal{C}} | {\mathcal{A}}']}{\pr[{\mathcal{C}} |
  {\mathcal{A}}]}  \hspace{-0.2cm} & \leq &\hspace{-0.2cm} 
(1+\updeltaw)^{k-1} \hspace{-0.1cm}  +  \hspace{-0.1cm}  f(k) \hspace{-0.1cm}\cdot \hspace{-0.1cm}\updeltaw  \hspace{-0.1cm}\cdot \hspace{-0.1cm}\left(1+\updeltas\right)^{k-1}  \hspace{-0.2cm}\cdot \hspace{-0.1cm}
 \varrho(R)  \:\:.\label{defppp}
\end{eqnarray}
\end{theorem}
(Proof in page \pageref{proof_thmdp}) 
Notice that Theorem \ref{thmdp} makes just one assumption (\ref{pwc}) about the
densities, so it
can be applied in fairly general settings, such as for regular exponential families \citep{bmdgCWj}. These are a key choice because they extensively cover the domain of
distortions for which the average is the
population minimiser. 

\paragraph{An (almost) distribution-free $1+o(1)$ likelihood ratio}\label{sec-adf} 
We now show that if
${\mathcal{A}}$ is sampled i.i.d. from any distribution
${\mathcal{D}}$ which satisfies the mild assumption that it is locally
bounded everywhere (or almost surely) in a ball, then with high probability the
right-hand side of ineq. (\ref{defppp}) is $1+o(1)$ where the
little-oh vanishes with $m$. The proof, of independent interest,
involves an explicit bound on $\updeltaw$ and
$\updeltas$.
\begin{theorem}\label{thlabf}
Suppose ${\mathcal{A}}$ with $|{\mathcal{A}}| = m > 1$ sampled i.i.d. from distribution ${\mathcal{D}}$
whose support contains a $L_2$ ball $\mathscr{B}_2(\ve{0},R)$ with density
inside in between $\upepsilon_m>0$ and
$\upepsilon_M \geq \upepsilon_m$. Let $\uprho_{{\tiny \mbox{$\mathcal{D}$}}} \defeq \upepsilon_M /
\upepsilon_m$ ($\geq 1$). For any $0<\updelta <1/2$, if
(i) ${\mathcal{A}} \subset \mathscr{B}(\ve{0},R)$ and (ii) the
number of clusters $k$ meets:
\begin{eqnarray}
k & \leq & \frac{\updelta^2}{4 \uprho_{{\tiny \mbox{$\mathcal{D}$}}}} \cdot \sqrt{m}\:\:,\label{condk}
\end{eqnarray}
then there is probability $1 - \updelta$ over the sampling of
${\mathcal{A}}$ that $k$-\means, instantiated as in Theorem
\ref{thmdp}, satisfies $\pr[{\mathcal{C}} | {\mathcal{A}}'] / \pr[{\mathcal{C}} |
  {\mathcal{A}}] \leq 1 +\uprho_{{\tiny \mbox{$\mathcal{D}$}}}^k \cdot g(m,k,d,R)$, $\forall {\mathcal{A}}'\approx {\mathcal{A}}$,
with 
\begin{eqnarray}
g(m,k,d,R) & \defeq & \frac{4}{m^{\frac{1}{4}+\frac{1}{d+1}}} + \left(\frac{64}{
  k^\frac{2}{d}}\right)^k 
\cdot \frac{\varrho(2R)}{m}\label{defg}\:\:.
\end{eqnarray}
\end{theorem}
(Proof in page \pageref{proof_thlabf}) The key informal statement 
of Theorem \ref{thlabf} is that one may obtain with high
probability some ``good'' datasets ${\mathcal{A}}$, \textit{i.e.}, for which
$\updeltaw, \updeltas$ are small, under very weak assumptions about
the domain at hand. The key point is that if one has
access to the sampling, then one can resample datasets
${\mathcal{A}}$ until a good one comes. 

\paragraph{Applications to differential privacy}\label{sec-dp}

Let $\mathscr{M}$ be any algorithm which takes as input
${\mathcal{A}}$ and $k$, and returns a set of $k$ centers
${\mathcal{C}}$. Let $\pr_{\mathscr{M}}[{\mathcal{C}} | {\mathcal{A}}]$
denote the probability, over the internal randomisation of
$\mathscr{M}$, that $\mathscr{M}$
returns ${\mathcal{C}}$ given ${\mathcal{A}}$ and $k$ ($k$, fixed, is omitted
in notations). Following is the definition of differential privacy \citep{dmnsCN},
tailored for conciseness to our clustering problem.
\begin{definition}\label{defEDPF}
$\mathscr{M}$ is $\epsilon$-differentially private (\DP) for $k$ clusters iff for any
neighbors ${\mathcal{A}}\approx {\mathcal{A}}'$, set ${\mathcal{C}}$
of $k$ centers,
\begin{eqnarray}
\pr_{\mathscr{M}}[{\mathcal{C}} | {\mathcal{A}}'] /
\pr_{\mathscr{M}}[{\mathcal{C}} | {\mathcal{A}}] & \leq & \exp
\epsilon\:\:.\label{defEDP}
\end{eqnarray}
\end{definition}
A relaxed version of $\epsilon$-\DP~is $(\epsilon,\delta)$-\DP, in
which we require $\pr_{\mathscr{M}}[{\mathcal{C}} | {\mathcal{A}}'] \leq 
\pr_{\mathscr{M}}[{\mathcal{C}} | {\mathcal{A}}]\cdot \exp \epsilon +
\delta$; thus, $\epsilon$-\DP~$ = (\epsilon, 0)$-\DP~\citep{drTA}. We
show that low noise may be affordable to satisfy
ineq. (\ref{defEDP}) using Laplace
distribution, $Lap(\sigma/\sqrt{2})$. We refer to the
\textit{Laplace mechanism} as a popular mechanism which adds to the
output of an algorithm a sufficiently large amount of Laplace noise to
be $\epsilon$-\DP. We refer to \citep{dmnsCN} for details, and assume
from now on
that data belong to a $L_1$ ball ${\mathscr{B}}_1(\ve{0}, R)$.

\begin{theorem}\label{thdp1}
Using notations and setting of Theorem \ref{thmdp}, let
\begin{eqnarray}
\tilde{\epsilon} & \defeq & \log\left(\frac{\exp(\epsilon) - (1+\updeltaw)^{k-1}}{f(k)\cdot \updeltaw  \cdot \left(1+\updeltas\right)^{k-1}}\right)\:\:.\label{defepsilont}
\end{eqnarray}
Then, $k$-\means~with $p_{(\ve{\mu}_{.}, \ve{\theta}_{.})}$ a
product of $Lap(\sigma_1/\sqrt{2})$, for $\sigma_1 \defeq
2\sqrt{2}R / \tilde{\epsilon}$, both meets ineq. (\ref{defEDP}) \textbf{and} its expected potential satisfies ineq. (\ref{boundsup}) with
\begin{eqnarray}
\Phi  & = & \Phi_1 \defeq 8 \cdot\left(\phi_{\opt} + \frac{mR^2}{\tilde{\epsilon}^2}\right)\:\:.\label{boundsupDP1}
\end{eqnarray}
On the other hand, if we opt for $\sigma_2 \defeq
2 \sqrt{2}kR/\epsilon$, then $k$-\means~is an instance of the Laplace
mechanism \textbf{and} its expected potential satisfies ineq. (\ref{boundsup}) with
\begin{eqnarray}
\Phi  & = & \Phi_2 \defeq 8 \cdot\left(\phi_{\opt} +  \frac{m k^2 R^2}{\epsilon^2}\right)\:\:.\label{boundsupDP2}
\end{eqnarray}
\end{theorem}
(Proof in page \pageref{proof_thdp1}) A question is how do
$\sigma_1$ (resp. $\Phi_1$) and $\sigma_2$ (resp. $\Phi_2$) compare with each
other, and how do they compare to the state of the art \citep{nrsSS,wwsDP} (we only consider methods with provable approximation
bounds of the global optimum). The key fact is that, if $m$ is
sufficiently large, then it happens that we can fix $\updeltaw =
O(1/m)$ and $\updeltas = O(1)$. The proof of Theorem  \ref{thlabf} (page
\pageref{proof_thlabf}) and the experiments (page
\pageref{exp_expes}) display that such regimes \textit{are} indeed observed. In this case, it is not hard to show that
$\tilde{\epsilon} = \Omega(\epsilon + \log m)$, granting $\sigma_1 = o(\sigma_2)$ since
\begin{eqnarray}
\sigma_1 & = & O\left(\frac{R}{\epsilon + \log(m)}\right)\:\:,
\end{eqnarray}
\textit{i.e.} the noise guaranteeing ineq. (\ref{defEDP})
\textit{vanishes} at $1/\log(m)$ rate. Consequently, in this regime,
$\Phi_1$ in eq. (\ref{boundsupDP1}) becomes:
\begin{eqnarray}
\Phi_1 & = & \tilde{O}\left(\phi_{\opt} + \frac{m R^2}{\left(\epsilon + \log m\right)^2}\right)\:\:,\label{boundsupDP22}
\end{eqnarray}
ignoring all factors other than those noted. Thus, the noise
dependence grows \textit{sublinearly} in $m$. Since in this setting,
unless all datapoints are the same,
$\updeltaw$ and $\updeltas$ for ${\mathcal{A}}$ and \textit{any} possible
neighbor ${\mathcal{A}}'$ are within $1+o(1)$, it is
also possible to overestimate $\updeltaw$ and $\updeltas$ to still
have $\updeltaw =
O(1/m)$ and $\updeltas = O(1)$ \textbf{and} grant
$\epsilon$-\DP~for $k$-\means. Otherwise, the setting of Theorem
\ref{thlabf} can be used to grant $(\epsilon, \updelta)$-\DP~without
any tweak. Table \ref{tcomp} compares $k$-\means~to
\citep{nrsSS,wwsDP} in this large sample regime, which is actually a
prerequisite for \citep{nrsSS,wwsDP}. Notation $O^*$ removes all
dependencies in their model parameters (assumptions, model
parameters, and $\delta$ for the $(\epsilon,\delta)$-\DP~in
\citep{wwsDP}), and $\uplambda$ is the separability assumption
parameter \citep{nrsSS}\footnote{$\uplambda$ is named
  $\phi$ in \citep{nrsSS}. We use $\uplambda$ to avoid confusion with
  clustering potentials.}. The approximation bounds in \citep{nrsSS} consider
Wasserstein distance between (estimated / optimal) centers, and not the potential involving
data points like us. To obtain bounds that can be compared, we have
used the simple trick that the observed potential is, up to a
constant, no more than the optimal potential plus a fonction of the
distance between (estimated / optimal) centers. This somewhat degrades 
the bound, but not enough for the observed discrepancies with our
bound to reverse or even vanish. It is clear from the bounds that the
noise dependence is significantly in our favor, and our bound is also
significantly better at least when $k$ is not too large.

\section{Experiments}\label{sec-em}

\begin{figure}[t]
    \centering
\begin{tabular}{cc}
\hspace{-1.1cm}\includegraphics[width=0.6\textwidth]{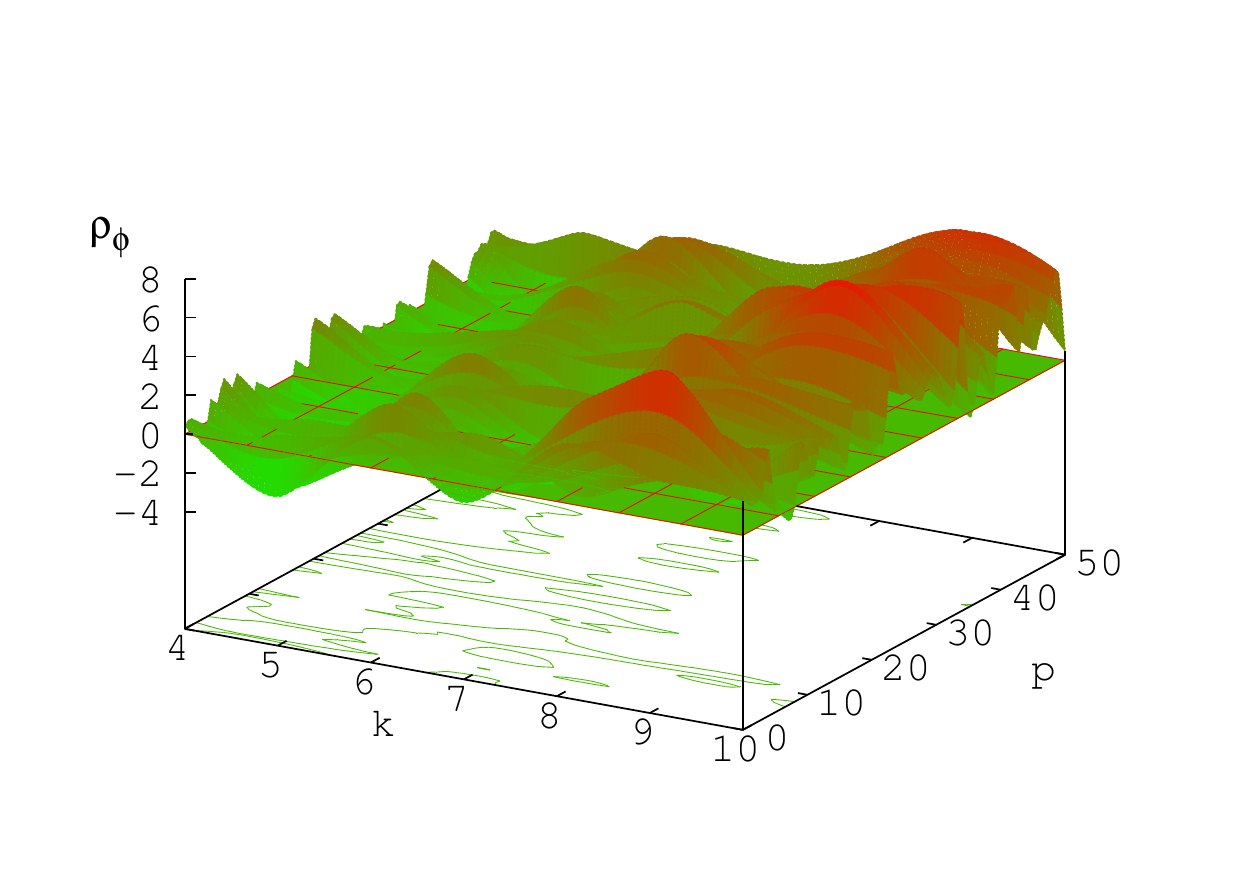}
\hspace{-1.3cm} & \hspace{-0.9cm} \includegraphics[width=0.6\textwidth]{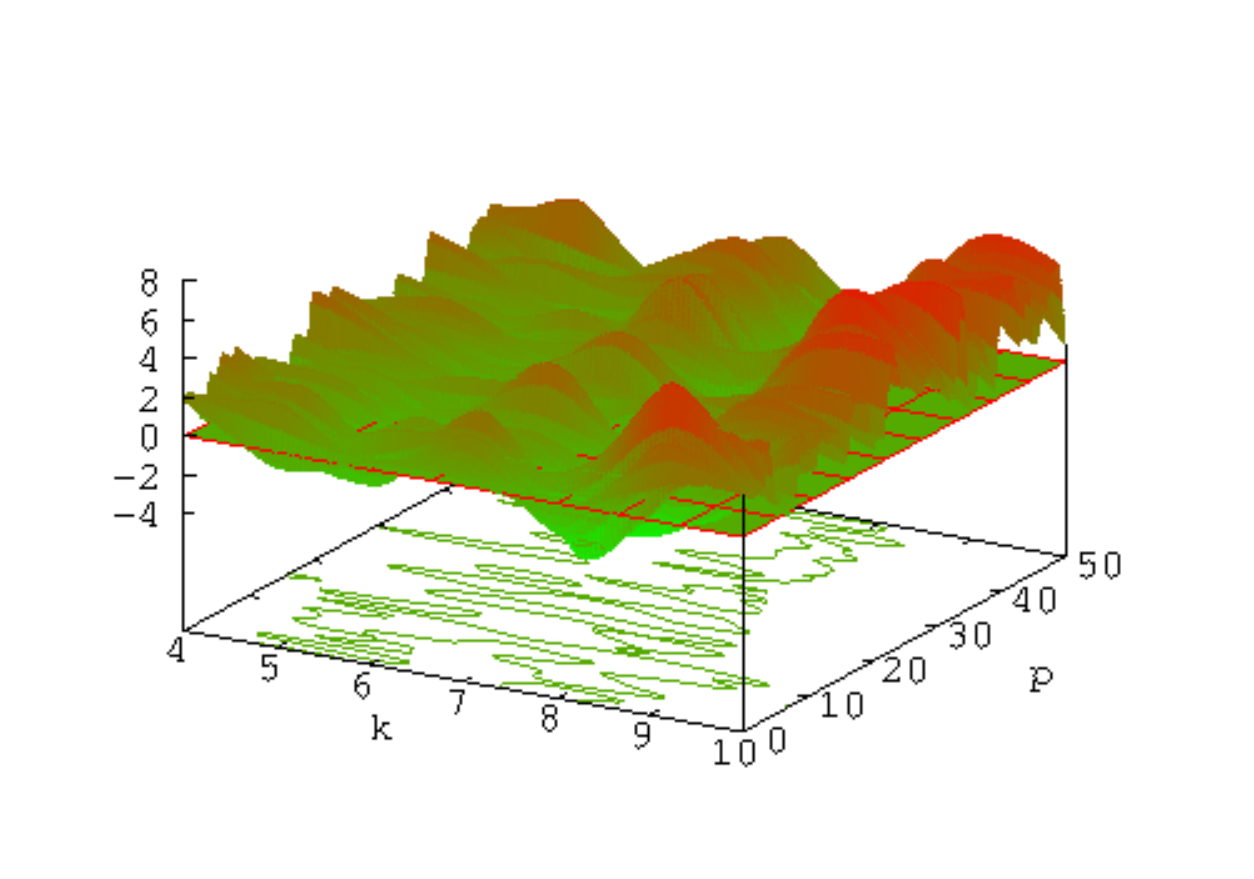}
\end{tabular}
\caption{Plot of $\uprho_\phi({\mathcal{H}})=f(k,p)$ (points below $z=0$ --- isocontour shown --- correspond to
  superior performances for \protectedKMPP). Left:
  ${\mathcal{H}}$=$k$-means++; right:
  ${\mathcal{H}}$=$k$-means$_{\tiny{\mbox{$\|$}}}$ (best viewed in color).}\label{f-par1}
\end{figure}

The experiments carried out are provided \textit{in extenso} in the
Appendix (from page \pageref{exp_expes}).

\noindent \textbf{\protectedKMPP~vs $k$-means++ and
  $k$-means$_{\tiny{\mbox{$\|$}}}$} \citep{bmvkvSK} To address algorithms that can be
reduced from $k$-\means~(Section \ref{sec-red}), we have tested
\protectedKMPP~vs state of the art approach 
$k$-means$_{\tiny{\mbox{$\|$}}}$; to be fair with \protectedKMPP,
we use $k$-means++ seeding as the
reclustering algorithm in $k$-means$_{\tiny{\mbox{$\|$}}}$. Parameters
are in line with
\citep{bmvkvSK}. To control the spread of Forgy nodes $\phi_s^F$
(Theorem \ref{thDKM}), each peer's initial data consists of points
uniformly sampled in a random hyperrectangle in a space of $d=50$
(expected number of peers points $m_i = 500, \forall i$). We sample peers until 
a total of
$m\approx$ 20000 point is sampled. Then, each point moves with
$p\%$ chances to a uniformly sampled peer. We checked that
$\phi_s^F$ blows up with $p$, \textit{i.e.}, $>$20 times for
$p=50\%$ with respect to $p=0$. A remarkable phenomenon was the fact
that, even when the number of peers $n$ is quite large (dozens
on average), \protectedKMPP~is able to
\textit{beat} both $k$-means++ and $k$-means$_{\tiny{\mbox{$\|$}}}$,
even for large values of $p$, as computed by ratio
$\uprho_\phi({\mathcal{H}}) \defeq 100 \cdot
(\phi(\mbox{\protectedKMPP}) -
\phi({\mathcal{H}}))/\phi({\mathcal{H}})$ for ${\mathcal{H}} \in
\{k\mbox{-means++}, \mbox{$k$-means$_{\tiny{\mbox{$\|$}}}$}\}$ (Figure \ref{f-par1}). Another positive point is that the
amount of data to compute a center for \protectedKMPP~is in average
$\approx n$ times smaller than $k$-means$_{\tiny{\mbox{$\|$}}}$. 

The fact that \protectedKMPP, which locally implements the biased
seeding, may be able to beat $k$-means++, which globally implements
this seeding technique, is not surprising, and in fact may come from
the leverage brought by the compartmentalization of 
distributed data: as discussed in deeper details in page
\pageref{exp_dkm1}, this may even improve the approximability ratio of
\protectedKMPP~so that it \textit{beats} the AV bound.

\noindent \textbf{$k$-\means~vs Forgy-\DP~and GUPT} To address algorithms that can be
obtained via a direct use of $k$-\means~(Section \ref{sec-dir}), we
have tested it in a differential privacy framework vs state of the art
approach GUPT \citep{mtsscGP}.  We let $\epsilon = 1$ in our
experiments. We also compare it to Forgy~\DP~(F-\DP), which is just Forgy
initialisation in the Laplace mechanism, with noise rate (standard dev.) 
$\propto kR/\epsilon$. In comparison, the noise rate for GUPT is
$\propto kR/(\ell \epsilon)$ at the end of its aggregation
process, where $\ell$ is the number of blocks. Table \ref{t-comp-sum} gives results for the
average (over the choices of $k$) parameters used, $k$, $\overline{\tilde{\epsilon}}$, and ratio
$\overline{\uprho'_\phi}$ where $\uprho'_\phi ({\mathcal{H}}) \defeq
\phi({\mathcal{H}}) / \phi(\mbox{$k$-\means})$ --- values above 1
indicate better results for $k$-\means. We use $\tilde{\epsilon}$ as
the equivalent $\epsilon$ for $k$-\means, \textit{i.e.} the value that
guarantees ineq. (\ref{defEDP}). From Theorem \ref{thdp1}, when
$\tilde{\epsilon} > \epsilon$, this brings a smaller noise magnitude,
desirable for clustering. 
The obtained results show that
$k$-\means~becomes more of a contender with increasing $m$, but
its relative performance tends to decrease with increasing $k$. This is in
accordance with the ``good'' regime of Theorem \ref{thdp1}. Results on
synthetic domains display the same patterns, along with the fact that
relative performances of $k$-\means~improves with $d$, making it a
relevant choice for "big" domains. 

In fact,
extensive experiments on synthetic data (page \pageref{exp_thdp1}) show
that intuitions regarding the sublinear noise regime in
eq. (\ref{boundsupDP22}) are experimentally observed, and
furthermore they may happen for quite small values of $m$. 

\begin{table}[t]
    \centering
\begin{center}
{
\begin{tabular}{|c|r|r|c|r|r|r|}
\hline 
Dataset & $m$ & $d$ & $\overline{k}$ & \multicolumn{1}{|c|}{$\overline{(\tilde{\epsilon}/\epsilon)}$} &
\multicolumn{1}{|c|}{{\tiny $\overline{\uprho'_\phi(\mbox{F-\DP})}$}} &
\multicolumn{1}{|c|}{{\tiny $\overline{\uprho'_\phi(\mbox{GUPT})}$}} \tabularnewline
\hline 
\hline 
 LifeSci & 26 733 & 10 & 3 & 4.5 & 163.0 & 0.7\\ 
Image & 34 112 & 3 & 2.5 & 7.9 & 188.5 & 2.9\\
EuropeDiff & 169 308 & 2 & 5 & 13.0 & 2857.1 & 40.4\\ 
\hline\hline 
\end{tabular}
}
\end{center}
\caption{$k$-\means~vs F-\DP~and GUPT~(see text).}\label{t-comp-sum}
\end{table}

\section{Discussion and Conclusion}

We first show in this paper that the $k$-means++ analysis of
Arthur and Vassilvitskii can be carried out on a significantly more general scale, aggregating various clustering frameworks of interest and
for which no trivial adaptation of $k$-means++ was previously
known. Our contributions stand at two levels: (i) we provide the
``meta'' algorithm, $k$-\means, and two key results, one on its
approximation abilities of the \textit{global} optimum, and one on the
\textit{likelihood ratio} of the centers it delivers. We do expect further
applications of these results, in particular to address several other key
clustering problems: stability,
generalisation and smoothed analysis \citep{amrSA,vCS}; (ii) we provide two examples of
application. The first is a reduction technique from $k$-\means, which
shows a way to obtain straight approximabilty results for other
clustering algorithms, some being efficient proxies for the
generalisation of existing approaches \citep{ajmSK}. The second is a direct application of
$k$-\means~to differential privacy, exhibiting a noise component
significantly better than existing approaches \citep{nrsSS,wwsDP}.

We have not discussed here the possibility to
replace the $L_2^2$ distortion which computes the potential
by elements from large and interesting classes --- clustering being a
huge practical problem, it is indeed reasonable to tailor the distortion to
the application at hand. 
One example are
Bregman divergences, that fail simple metric transforms
\citep{abbBD}. Another example are total divergences, that fail the simple computation of the population minimizers
\citep{nnaOCD,lvanSR}. Some do not even admit population minimizers in
closed form \citep{nnTJ}. It turns out that $k$-\means, and its
good approximation properties, \textit{can} be extended to such
cases (see page \pageref{proof_ext}) for total Jensen divergence
\citep{nnTJ}.

\section{Acknowledgments}

Thanks are due to Stephen
Hardy, Guillaume Smith, Wilko Henecka and Max Ott for stimulating
discussions and feedback on the subject. Nicta is funded by the Australian Government through the Department of Communications and the Australian Research Council through the ICT Center of Excellence Program.

\bibliography{bibgen}
\bibliographystyle{icml2016} 

\newpage

\section*{Appendix --- Table of contents}\label{secappendix}

\noindent \textbf{Appendix on proofs} \hrulefill Pg \pageref{proof_proofs}\\
\noindent Proof of Theorem \ref{thmain}\hrulefill Pg \pageref{sec-proof-thmain}\\
\noindent Proof of Lemma \ref{lembinf}\hrulefill Pg
\pageref{proof_lembinf}\\
\noindent Comments on Table \ref{tcomp}\hrulefill Pg
\pageref{comments_table}\\
\noindent Proofs of Theorems \ref{thDKM}, \ref{thSKM} and \ref{thOKM}\hrulefill Pg \pageref{proof_all_thDKM}\\
\noindent $\hookrightarrow$ Proof of Theorem \ref{thDKM}\hrulefill Pg \pageref{proof_thDKM}\\
\noindent $\hookrightarrow$ Proof of Theorem \ref{thSKM}\hrulefill Pg \pageref{proof_thSKM}\\
\noindent $\hookrightarrow$ Proof of Theorem \ref{thOKM}\hrulefill Pg
\pageref{proof_thOKM}\\
\noindent Proof of Theorem \ref{thmdp}\hrulefill Pg \pageref{proof_thmdp}\\
\noindent Proof of Theorem \ref{thlabf}\hrulefill Pg \pageref{proof_thlabf}\\
\noindent Proof of Theorem \ref{thdp1}\hrulefill Pg \pageref{proof_thdp1}\\
\noindent Extension to non-metric spaces\hrulefill Pg \pageref{proof_ext}\\

\noindent \textbf{Appendix on experiments} \hrulefill Pg
\pageref{exp_expes}\\
\noindent Experiments on Theorem \ref{thdp1} and the sublinear noise regime\hrulefill Pg \pageref{exp_thdp1}\\
\noindent Experiments with \protectedKMPP, $k$-means++ and $k$-means$_{\tiny{\mbox{$\|$}}}$\hrulefill Pg \pageref{exp_dkm1}\\
\noindent Experiments with $k$-\means~and GUPT\hrulefill Pg \pageref{exp_dp1}\\

\newpage

\section{Appendix on Proofs}\label{proof_proofs}

Several proofs rely on properties of the $k$-means++ algorithm that
are not exploited in the proof of \citep{avKM}. We assume here the
basic knowledge of the proof technique of \citep{avKM}.

\section*{Proof of Theorem \ref{thmain}}\label{sec-proof-thmain}

Let $A$ denote a subset of ${\mathcal{A}}$, and $\ve{c}(A) \defeq
(1/|A|)\cdot \sum_{\ve{a} \in A} \ve{a}$ 
the barycenter of $A$. It is well known that $\ve{c}(A) = \arg\min_{\ve{a}'
\in {\mathbb{R}}^d} \sum_{\ve{a} \in A}\|\ve{a}-\ve{a}'\|_2^2$, so the
potential of $A$,
\begin{eqnarray}
\phi(A) & \defeq & \sum_{\ve{a} \in A} \|\ve{a}-\ve{c}(A)\|_2^2
\end{eqnarray}
is just the optimal potential of $A$ if $A$ defines a cluster in the
optimal clustering. We also define the noisy potential of $A$ as:
\begin{eqnarray}
\phi^N(A) & \defeq & \sum_{\ve{a} \in A} \int_{\Omega_{\ve{a}}}
   \|\ve{x} - \ve{c}(A)\|_2^2\mathrm{d}p_{\ve{a}}(\ve{x}) \:\:.
\end{eqnarray}
The proof of Theorem \ref{thmain} follows the same path as the proof
of Theorem 3.1 in \citep{avKM}. Instead of reproducing the proof, we
shall assume basic knowledge of the original proof and will just provide the side
Lemmata that are sufficient for our more general result. The first
Lemma is a generalization of Lemma 3.2 in \citep{avKM}.
\begin{lemma}\label{lem21}
Let $C_{\opt}$ denotes the optimal partition of ${\mathcal{A}}$
according to eq. (\ref{costm}). Let $A$ be an arbitrary cluster in $C_{\opt}$. Let $C$ be a single-cluster clustering whose center
is chosen at random by one step of Algorithm $k$-\means~(\textit{i.e.}
for $t=1$). Then 
\begin{eqnarray}
\expect[\phi(A)] & = & \phi_{\mathrm{\opt}}(A) +
\phi^N_{\mathrm{\opt}}(A)\:\:.
\end{eqnarray}
\end{lemma}
\begin{proof}
The expected potential of cluster $A$ is 
\begin{eqnarray}
\lefteqn{\expect[\phi(A ; {\mathcal{C}} = \emptyset)]}\nonumber\\
 & = & \frac{1}{|A|} \cdot  \sum_{\ve{a}_0 \in A} \int_{\Omega_{{\ve{a}}_0}} \sum_{\ve{a} \in A} \|\ve{a} - \ve{x}\|_2^2 \mathrm{d}p_{\ve{a}_0}(\ve{x})\nonumber\\
 & = & \frac{1}{|A|}  \cdot \sum_{\ve{a}_0 \in A} \int_{\Omega_{{\ve{a}}_0}} \sum_{\ve{a} \in
   A} \|\ve{a} - \ve{c}(A) + \ve{c}(A) - \ve{x}\|_2^2 \mathrm{d}p_{\ve{a}_0}(\ve{x}) \nonumber\\
 & = & \frac{1}{|A|} \cdot \sum_{\ve{a}_0 \in A} \left( 
\begin{array}{l}
\sum_{\ve{a} \in
   A} \|\ve{a} - \ve{c}(A) \|_2^2
   + |A| \cdot \int_{\Omega_{{\ve{a}}_0}}
   \|\ve{x} - \ve{c}(A)\|_2^2\mathrm{d}p_{\ve{a}_0}(\ve{x})  \\
+ 2 \sum_{\ve{a} \in
   A} \langle \ve{a} - \ve{c}(A), \ve{c}(A) - \int_{\Omega_{{\ve{a}}_0}}
 \ve{x} \mathrm{d}p_{\ve{a}_0}(\ve{x})\rangle 
\end{array}\right) \nonumber\\
 & = & \frac{1}{|A|} \cdot  \sum_{\ve{a}_0 \in A} \left( 
\begin{array}{l}
\sum_{\ve{a} \in
   A} \|\ve{a} - \ve{c}(A) \|_2^2
   
+ |A| \cdot \int_{\Omega_{{\ve{a}}_0}}
   \|\ve{x} - \ve{c}(A)\|_2^2\mathrm{d}p_{\ve{a}_0}(\ve{x})  \\
+ 2 \langle \underbrace{\sum_{\ve{a} \in
   A} \ve{a} - |A| \ve{c}(A)}_{=0}, \ve{c}(A) - \ve{a}_0\rangle 
\end{array}
\right) \nonumber\\
 & = & \sum_{\ve{a} \in
   A} \|\ve{a} - \ve{c}(A) \|_2^2 + \sum_{\ve{a} \in A} \int_{\Omega_{{\ve{a}}_0}}
   \|\ve{x} - \ve{c}(A)\|_2^2\mathrm{d}p_{\ve{a}}(\ve{x}) \nonumber\\
 & = & \phi_{\mathrm{\opt}}(A) + \phi^N_{\mathrm{\opt}}(A)\:\:,\nonumber
\end{eqnarray}
as claimed.
\end{proof}
When $p_{\ve{a}}$ is a Dirac anchored at $\ve{a}$, we recover Lemma 3.2 in \citep{avKM}. The
following Lemma generalizes Lemma 3.3 in \citep{avKM}.
\begin{lemma}
Suppose that the optimal clustering $C_{\opt}$ is $\upeta$-probe approximable.
Let $A$ be an arbitrary cluster in $C_{\opt}$, and let $C$ be an
arbitrary clustering with centers ${\mathcal{C}}$. Suppose that the reference point $\ve{a}$ chosen
according to (\ref{choosea}) in Step 2.1 is in $A$. Then the random point $\ve{x}$
picked in Step 2.2 brings an
expected potential that satisfies 
\begin{eqnarray}
\expect[\phi(A)] & \leq & (6 + 4\upeta) \cdot \phi_{\opt}(A) + 2\cdot
  \phi^N_{\mathrm{\opt}}(A)\:\:.
\end{eqnarray}
\end{lemma}
\begin{proof}
Let us denote ${\ve{c}}^\star(\ve{u}) \defeq \arg\min_{\ve{x} \in
  {\mathcal{C}}} \|\ve{u} - \ve{x}\|_2^2$ (since $C\neq C_{\opt}$ in general, 
${\ve{c}}^\star(\ve{u})\neq {\ve{c}}_{\opt}(\ve{u})$), and $D(\ve{a}) \defeq
\|\ve{a} - {\ve{c}}^\star(\ve{a}) \|_2^2$ the contribution of $\ve{a}\in A$ to the $k$-means potential defined by
${\mathcal{C}}$. We have, using Lemma 3.3 in \citep{avKM} and Lemma \ref{lem21},
\begin{eqnarray}
\expect_{\ve{x}}[\phi(A ; {\mathcal{C}} \cup
  \{\ve{x}\})]  & = & \sum_{\ve{a}_0 \in A} \frac{D_t(\ve{a}_0)}{\sum_{\ve{a} \in A}
D_t(\ve{a})} \cdot \sum_{\ve{a}\in A} \int_{\Omega_{{\ve{a}}_0}} \min \{D(\ve{a}), \|\ve{a} -
\ve{x}\|_2^2\}\mathrm{d}p_{\ve{a}_0}(\ve{x})\label{eq1}\:\:.
\end{eqnarray}
The triangle inequality gives, for any $\ve{a}\in A$,
\begin{eqnarray}
\sqrt{D_t(\ve{a}_0)} & \defeq & \|\probe_t(\ve{a}_0)
- {\ve{c}}^\star(\probe_t(\ve{a}_0))\|_2\nonumber\\
 & \leq & \|\probe_t(\ve{a}_0)
- {\ve{c}}^\star(\probe_t(\ve{a}))\|_2\nonumber\\
 & \leq & \|\probe_t(\ve{a}_0)
- \probe_t(\ve{a})\|_2 + \|\probe_t(\ve{a})
- {\ve{c}}^\star(\probe_t(\ve{a}))\|_2 \:\:;
\end{eqnarray}
since $(a+b)^2 \leq 2a^2 + 2b^2$, then $D_t(\ve{a}_0) \leq 2\|\probe_t(\ve{a}_0)
- \probe_t(\ve{a})\|_2^2 + 2D_t(\ve{a})$, and so, after averaging over $A$,
\begin{eqnarray}
D_t(\ve{a}_0) & \leq & \frac{2}{|A|} \sum_{\ve{a}\in  A} \|\probe_t(\ve{a}_0)
- \probe_t(\ve{a})\|_2^2 + \frac{2}{|A|} \sum_{\ve{a}\in  A} D_t(\ve{a}) \:\:,
\end{eqnarray} 
and eq. (\ref{eq1}) can be
upperbounded as:
\begin{eqnarray}
\expect_{\ve{x}}[\phi(A ; {\mathcal{C}} \cup
  \{\ve{x}\})]   & \leq & \frac{2}{|A|} \sum_{\ve{a}_0 \in A} \frac{\sum_{\ve{a}\in  A} \|\probe_t(\ve{a}_0)
- \probe_t(\ve{a})\|_2^2}{\sum_{\ve{a} \in A}
D_t(\ve{a})} \cdot \sum_{\ve{a}\in A} \int_{\Omega_{{\ve{a}}_0}} \min \{D(\ve{a}), \|\ve{a} -
\ve{x}\|_2^2\}\mathrm{d}p_{\ve{a}_0}(\ve{x})  \nonumber\\
 & & + \frac{2}{|A|} \sum_{\ve{a}_0 \in A} \frac{\sum_{\ve{a}\in  A} D_t(\ve{a})}{\sum_{\ve{a} \in A}
D_t(\ve{a})} \cdot \sum_{\ve{a}\in A} \int_{\Omega_{{\ve{a}}_0}} \min \{D(\ve{a}), \|\ve{a} -
\ve{x}\|_2^2\}\mathrm{d}p_{\ve{a}_0}(\ve{x}) \nonumber\\
& \leq & \underbrace{\frac{2}{|A|} \sum_{\ve{a}_0 \in A} \frac{\sum_{\ve{a}\in A} D(\ve{a})}{\sum_{\ve{a} \in A}
D_t(\ve{a})} \cdot \sum_{\ve{a}\in  A} \|\probe_t(\ve{a}_0)
- \probe_t(\ve{a})\|_2^2}_{\defeq P_1} \nonumber\\
 & & + \underbrace{\frac{2}{|A|} \sum_{\ve{a}_0 \in A} \sum_{\ve{a}\in A} \int_{\Omega_{{\ve{a}}_0}} \|\ve{a} -
\ve{x}\|_2^2\mathrm{d}p_{\ve{a}_0}(\ve{x})}_{\defeq P_2} \:\:.\label{lasteq}
\end{eqnarray}
We bound the two potentials $P_1$ and $P_2$ separately, starting with
$P_1$. Fix any $\ve{a}_0 \in A$. If $\sum_{\ve{a}\in  A} \|
\probe_t(\ve{a}) - \probe_t(\ve{a}_0)\|_2^2 = 0$, then trivially
\begin{eqnarray}
\left(\sum_{\ve{a}\in A} D(\ve{a})\right)\cdot \left(\sum_{\ve{a}\in  A} \|\probe_t(\ve{a}_0)
- \probe_t(\ve{a})\|_2^2\right) & \leq & (1+\upeta) \cdot \left(
\sum_{\ve{a}\in A} D_t(\ve{a})\right) \cdot \left( \sum_{\ve{a}\in  A} \|\ve{a}_0
- \ve{a}\|_2^2\right)\:\:, \label{defapprox1}
\end{eqnarray}
since the right-hand side cannot be negative. If $\sum_{\ve{a}\in  A} \|
\probe_t(\ve{a}) - \probe_t(\ve{a}_0)\|_2^2 \neq 0$, then since $\probe_t$ is
$\upeta$-stretching, we have:
\begin{eqnarray}
\frac{\sum_{\ve{a}\in  A} \|\ve{a}
- \ve{c}^\star(\ve{a})\|_2^2}{\sum_{\ve{a}\in  A} \|\ve{a} - \ve{a}_0\|_2^2}  & \leq & (1+\upeta)\cdot \frac{\sum_{\ve{a}\in  A} \|\probe_t(\ve{a})
- \ve{c}^\star(\probe_t(\ve{a}))\|_2^2}{\sum_{\ve{a}\in  A} \|
\probe_t(\ve{a}) - \probe_t(\ve{a}_0)\|_2^2}\:\:,
\end{eqnarray}
which is exactly ineq. (\ref{defapprox1}) after rearranging the
terms. Ineq (\ref{defapprox1}) 
implies
\begin{eqnarray}
P_1 & \leq & 2 (1+\upeta) \cdot \frac{1}{|A|} \sum_{\ve{a}_0 \in A} \sum_{\ve{a}\in  A} \|\ve{a}_0
- \ve{a}\|_2^2\nonumber\\
 & & = 4 (1+\upeta) \cdot \phi_{\opt}(A)\:\:,
\end{eqnarray}
where the equality follows from \citep{avKM}, Lemma 3.2. Also, Lemma
\ref{lem21} brings
\begin{eqnarray}
P_2 & = & 2 \cdot \frac{1}{|A|} \sum_{\ve{a}_0 \in A} \int_{\Omega_{{\ve{a}}_0}} \sum_{\ve{a}\in A} \|\ve{a} -
\ve{x}\|_2^2\mathrm{d}p_{\ve{a}_0}(\ve{x})\nonumber\\
 & = & 2 \phi_{\mathrm{\opt}}(A) + 2 \phi^N_{\mathrm{\opt}}(A)\:\:.
\end{eqnarray}
We therefore get
\begin{eqnarray}
\expect_{\ve{x}}[\phi(A ; {\mathcal{C}} \cup
  \{\ve{x}\})]   & \leq & (6 + 4\upeta) \cdot \phi_{\opt}(A) + 2\cdot
  \phi^N_{\mathrm{\opt}}(A)\:\:, 
\end{eqnarray}
as claimed.
\end{proof}
Again, we recover Lemma 3.3 in \citep{avKM} when $p_{\ve{a}}$ is a
Dirac and the probe function $\probe = \mathrm{Id}$. The
rest of the proof of Theorem \ref{thmain} consists of the same steps
as Theorem 3.1 in \citep{avKM}, after having remarked that
$\phi^N_{\mathrm{\opt}}(A)$ can be simplified:
\begin{eqnarray}
\phi_{\opt}^N(A) & = &  \sum_{\ve{a} \in A} \int_{\Omega_{{\ve{a}}_0}}
   \|\ve{x} - \ve{c}(A)\|_2^2\mathrm{d}p_{\ve{a}}(\ve{x}) \nonumber\\
 & = & \sum_{\ve{a} \in A} 
\int_{\Omega_{{\ve{a}}_0}}
   \|\ve{x}\|_2^2\mathrm{d}p_{\ve{a}}(\ve{x})  -  2
   \langle \ve{c}(A), \ve{\mu}_{\ve{a}}\rangle 
+ \|\ve{c}(A)\|_2^2\nonumber\\
 & = & \sum_{\ve{a} \in A} 
\int_{\Omega_{{\ve{a}}_0}}
   \|\ve{x} - \ve{\mu}_{\ve{a}}\|_2^2\mathrm{d}p_{\ve{a}}(\ve{x})  +
   \|\ve{\mu}_{\ve{a}}\|_2^2 
-  2
   \langle \ve{c}(A), \ve{a}\rangle + \|\ve{c}(A)\|_2^2 \nonumber\\
 & = & \sum_{\ve{a} \in A} \left\{ \trace{\Sigma_{\ve{a}}} +  \|\ve{\mu}_{\ve{a}} - \ve{c}(A)\|_2^2\right\}\nonumber\\
 & = & \phi_{\bias}(A) + \phi_{\variance}(A) \:\:.
\end{eqnarray}

\subsection*{Proof of Lemma \ref{lembinf}}\label{proof_lembinf}
The proof is a simple application of the
Fr\'echet-Cram\'er-Rao-Darmois bound. 
Consider the simple case
$k=1$ and a spherical Gaussian noise for $p$ with a single point in ${\mathcal{A}}$. Renormalize both sides of (\ref{boundsup}) by $m
\defeq |{\mathcal{A}}|$ so that $ (1/m) \sum_{a\in {\mathcal{A}}}
\trace{\Sigma_{\ve{a}}} = \trace{\Sigma}$. One sees that the left hand
side of ineq. (\ref{boundsup}) is just an estimator of the variance of
$p_{\ve{a}}$, which, by Fr\'echet-Darmois-Cram\'er-Rao bound, has to
be at least the inverse of the Fisher information, that is in
this case, the trace of the covariance matrix, \textit{i.e.}
$\trace{\Sigma}$.  

\subsection*{Comments on Table \ref{tcomp}}\label{comments_table}

\citep{wwsDP} are concerned with approximating subspace
clustering, and so they are using a very different potential function,
which is, between two subspaces ${\mathcal{S}}$ and ${\mathcal{S}}'$,
$d({\mathcal{S}},{\mathcal{S}}') = \|\matrice{u}\matrice{u}^\top -
\matrice{u}'\matrice{u}'^\top \|_F$, where $\matrice{u}$
(resp. $\matrice{u}'$) is an \textit{orthonormal} basis for ${\mathcal{S}}$
(resp. ${\mathcal{S}}'$). To obtain an idea of the approximation on
the $k$-means clustering problem that their technique yields, we
compute $\phi$ in the projected space, using the fact that, because of
the triangle inequality and the fact that projections are linear and do not increase norms,
\begin{eqnarray}
\|\mathrm{proj}_{\matrice{u}}(\ve{a}) -
\mathrm{proj}_{\matrice{u}'}(\ve{a}')\|_2 & = &
\|(\mathrm{proj}_{\matrice{u}}(\ve{a}) - \mathrm{proj}_{\matrice{u}}(\ve{a}')) 
+ (\mathrm{proj}_{\matrice{u}}(\ve{a}') -
\mathrm{proj}_{\matrice{u}'}(\ve{a}'))\|_2\\
 & \leq & \|\mathrm{proj}_{\matrice{u}}(\ve{a}) -
 \mathrm{proj}_{\matrice{u}}(\ve{a}')\|_2 + \|\mathrm{proj}_{\matrice{u}}(\ve{a}') -
\mathrm{proj}_{\matrice{u}'}(\ve{a}'))\|_2\\
 & \leq & \|\mathrm{proj}_{\matrice{u}}(\ve{a}) -
 \mathrm{proj}_{\matrice{u}}(\ve{a}')\|_2 + 2\|\ve{a}'\|_2\:\:.
\end{eqnarray}
To account for the approximation in the inequalities, we then discard
the rightmost term, replacing therefore $\|\mathrm{proj}_{\matrice{u}}(\ve{a}) -
\mathrm{proj}_{\matrice{u}'}(\ve{a}')\|_2$ by $\|\mathrm{proj}_{\matrice{u}}(\ve{a}) -
 \mathrm{proj}_{\matrice{u}}(\ve{a}')\|_2$, which amounts, in the
 approximation bounds, to remove the dependence in the dimension. At
 this price, and using the trick to transfer the wasserstein distance
 between centers to $L_2^2$ potential between points to cluster
 centers, we obtain the approximation bound in ($\beta$) of Table
 \ref{tcomp}. While it has to be used with care, its main interest
 is in showing that the price to pay because of the noise component is
 in fact not decreasing in $m$.\\

\subsection*{Proofs of Theorems \ref{thDKM}, \ref{thSKM} and \ref{thOKM}}\label{proof_all_thDKM}

The proof of these Theorems uses a \textit{reduction} from $k$-\means~to the
corresponding algorithms, meaning that there exists particular probe
functions and densities for which the set of centers delivered by
$k$-\means~is the same as the one delivered by the corresponding
algorithms.

\begin{definition}
Let ${\mathcal{H}}$ (parameters omitted) be any hard membership $k$-clustering
algorithm. We way that $k$-\means~\textbf{reduces} to ${\mathcal{H}}$
iff there exists data, densities and probe functions
depending on the instance of ${\mathcal{H}}$ such that, in expectation
over the internal randomisation of ${\mathcal{H}}$, the set of centers
delivered by ${\mathcal{H}}$ are the same as the ones delivered by
$k$-\means. We note it
\begin{eqnarray}
\mbox{$k$-\means} & \succeq & {\mathcal{H}} \:\:.
\end{eqnarray}
\end{definition}
Hence, whenever $\mbox{$k$-\means} \succeq {\mathcal{H}}$, Theorem
\ref{thmain} immediately gives a guarantee for the
approximation of the global optimum in expectation for ${\mathcal{H}}$, but this requires
the translation of the parameters involved in $\Phi$ in
ineq. (\ref{boundsup}) to involve only parameters from ${\mathcal{H}}$. In
all our examples, this translation poses no problem at all.

\subsubsection*{Proof of Theorem \ref{thDKM}}\label{proof_thDKM}

\begin{figure}[t]
    \centering
\includegraphics[width=0.8\textwidth]{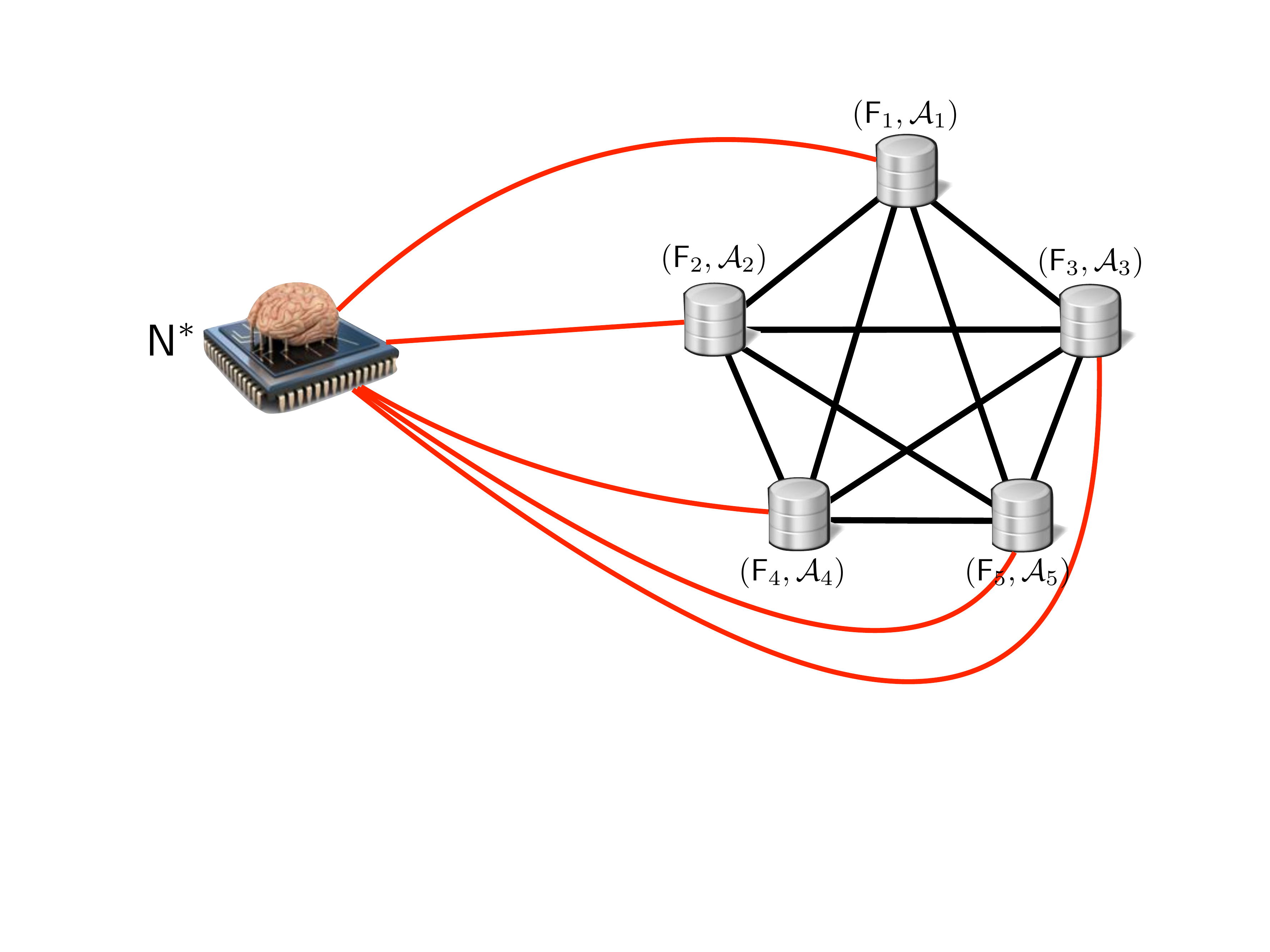}
\caption{Message passing between peers / nodes in the
  \protectedKMPP/\privateKMPP~framework. Black edges and red arcs
  denote message passing between peers / nodes. On each black edge
  circulates at most $k$ data points; on each red arcs circulates $k$ total potentials.}\label{f-archi}
\end{figure}

Figure \ref{f-archi} presents the architecture of message passing in
the \protectedKMPP/\privateKMPP~framework.
We first focus on the protected scheme, \protectedKMPP. 
We reduce $k$-\means~to Algorithm \ref{algoDKM} using identity probe
functions: $\probe_t = \mathrm{Id}, \forall t$. The trick in reduction
relies on the densities. We let $p_{\ve{\mu}_{\ve{a}},
  \ve{\theta}_{\ve{a}}}$ be uniform over the subset ${\mathcal{A}}_i$
to which $\ve{a}$ belongs. Thus, the support of densities is
discrete, and ${\mathcal{C}}$ is a subset of
${\mathcal{A}}$; furthermore, the probability $q_t(\ve{a})$ that $\ve{a}\in {\mathcal{A}}_i$ is chosen at
iteration $t$ in $k$-\means~actually simplifies to a convenient expression:
\begin{eqnarray}
q_t(\ve{a}) & =&  q^D_{ti} \cdot u_i \:\:,
\end{eqnarray}
where we recall that 
\begin{eqnarray}
q^D_{ti} & \defeq & \left\{
\begin{array}{ccl}
D_t({\mathcal{A}}_i) \cdot (\sum_{j}
  D_t({\mathcal{A}}_j))^{-1} & \mbox{ if } & t>1 \\
(1/n) & \multicolumn{2}{l}{\mbox{ otherwise}}
\end{array}
\right. \:\:. \label{defpstar}
\end{eqnarray}
Hence, picking $\ve{a}$ can be equivalently done by first
picking ${\mathcal{A}}_i$ using $q^D_{t}$, and then, given the $i$
chosen, sampling uniformly at random $\ve{a}$ in ${\mathcal{A}}_i$,
which is what Forgy nodes do. We therefore get the
equivalence between Algorithm \ref{algoDKM} and $k$-\means~as instantiated. 
\begin{lemma}
With data, densities and probes defined as before, \mbox{$k$-\means}
$\succeq$ \protectedKMPP.
\end{lemma}
To get
the approximability ratio of \protectedKMPP, we translate the
parameters of $\Phi$ in ineq. (\ref{boundsup}). First, since $(a + b)^2 \leq 2a^2 + 2b^2$,
\begin{eqnarray}
\phi_{\bias} & \defeq & \sum_{\ve{a} \in {\mathcal{A}}} \|\ve{\mu}_{\ve{a}}
- \ve{c}_{\opt}(\ve{a})\|_2^2\nonumber\\
 & = & \sum_{i \in [n]}\sum_{\ve{a} \in
  {\mathcal{A}}_i} \|\ve{c}({\mathcal{A}}_i) -
\ve{c}_{\opt}(\ve{a})\|_2^2\label{spreadex}\\
 & = & \sum_{i \in [n]}\sum_{\ve{a} \in
  {\mathcal{A}}_i} \|\ve{c}({\mathcal{A}}_i) - \ve{a} + \ve{a} -
\ve{c}_{\opt}(\ve{a})\|_2^2\nonumber\\
 & \leq & 2\sum_{i \in [n]}\sum_{\ve{a} \in
  {\mathcal{A}}_i} \|\ve{c}({\mathcal{A}}_i) - \ve{a} \|_2^2 + 2\sum_{\ve{a} \in
  {\mathcal{A}}} \|\ve{a} -
\ve{c}_{\opt}(\ve{a})\|_2^2\nonumber\\
 &  & = 2\phi^F_s + 2\phi_{\opt} \:\:.\label{eq11}
\end{eqnarray}
Furthermore,
\begin{eqnarray}
\phi_{\variance} & \defeq & \sum_{\ve{a} \in {\mathcal{A}}}
\trace{\Sigma_{\ve{a}}}\nonumber\\
 & = & \sum_{\ve{a} \in {\mathcal{A}}} \int_{\Omega_{\ve{a}}}
   \|\ve{x} -
   \ve{\mu}_{\ve{a}}\|_2^2\mathrm{d}p_{\ve{a}}(\ve{x}) \nonumber\\
 & = & \sum_{i \in [n]}\sum_{\ve{a} \in
  {\mathcal{A}}_i} \sum_{\ve{a}' \in
  {\mathcal{A}}_i} \frac{1}{m_i} \cdot \|\ve{a}' -
   \ve{c}({\mathcal{A}}_i)\|_2^2 \nonumber\\
 & = & \sum_{i \in [n]} \sum_{\ve{a} \in
  {\mathcal{A}}_i} \|\ve{a} -
   \ve{c}({\mathcal{A}}_i)\|_2^2 = \phi^F_s\:\:.\label{eq22}
\end{eqnarray}
There remains to plug ineq. (\ref{eq11}) and eq. (\ref{eq22}) in
Theorem \ref{thmain}, along with $\upeta = 0$ (since $\probe =
\mathrm{Id}$), to get $\expect[\phi({\mathcal{A}} ; {\mathcal{C}})] \leq (2 + \log
k)\cdot \left(10\phi_{\opt} + 6\phi_{s}\right)$, as in Theorem
\ref{thDKM}.\\

The private version, \privateKMPP, follows immediately by leaving
$\phi_{\variance}$ in $\Phi$ instead of carrying
eq. (\ref{eq22}). This ends the proof of Theorem \ref{thDKM}.

\subsubsection*{Proof of Theorem \ref{thSKM}}\label{proof_thSKM}

\begin{figure}[t]
\centering
\includegraphics[width=.5\linewidth]{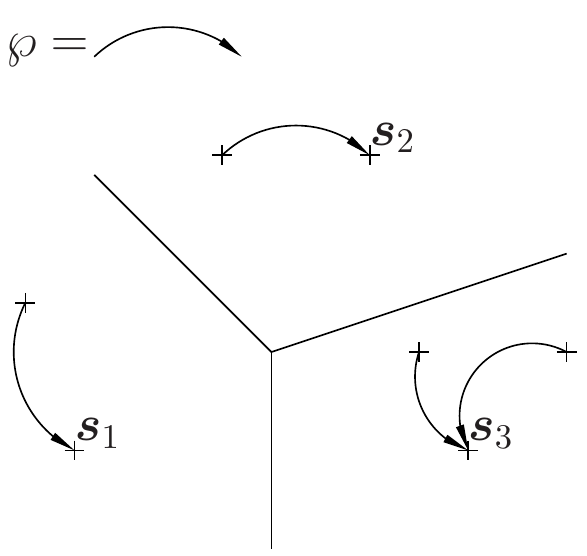}
\caption{Computation of the probe function $\probe$ for the reduction
  from $k$-\means~to \SKM. Segments display parts of the Voronoi
  diagram of ${\mathcal{S}}$.}
\label{f-model}
\end{figure}

The proof proceeds in the same way as for Theorem \ref{thDKM}. The
probe function (the same for every iteration, $\probe_t = \probe,
\forall t$) is already defined in the statement of Theorem
\ref{thSKM}, from the definition of synopses. The distributions $p_{\ve{\mu}_{\ve{a}},
  \ve{\theta}_{\ve{a}}}$ are Diracs anchored at the \textit{probe}
(synopses) locations. The centers chosen in $k$-\means~are thus
synopses, and it is not hard to check that the probability to pick a
synopsis $\ve{s}_j$ at iteration $t$ factors in the same way as in
the definition of $q^S_t$ in eq. (\ref{qw}). We therefore get the
equivalence between Algorithm \ref{algoSKM} and $k$-\means~as instantiated. 
\begin{lemma}
With data, densities and probes defined as before, \mbox{$k$-\means}
$\succeq$ \SKM.
\end{lemma}
The proof of the approximation property of \SKM~then follows from the
fact that $\phi_{\variance} = 0$ (Diracs) and
\begin{eqnarray}
\phi_{\bias} & \defeq & \sum_{\ve{a} \in {\mathcal{A}}}
\|
   \ve{\mu}_{\ve{a}} - \ve{c}_{\opt}(\ve{a})\|_2^2\nonumber\\
 & = & \sum_{\ve{a} \in {\mathcal{A}}}
\|
   \probe({\ve{a}}) - \ve{c}_{\opt}(\ve{a})\|_2^2 \nonumber\\
& = & \sum_{\ve{a} \in {\mathcal{A}}}
\|
   \probe({\ve{a}}) - {\ve{a}} + {\ve{a}} - \ve{c}_{\opt}(\ve{a})\|_2^2 \nonumber\\
 & \leq & 2\sum_{\ve{a} \in {\mathcal{A}}}
\|
   \probe({\ve{a}}) - {\ve{a}}\|_2^2 +2\sum_{\ve{a} \in {\mathcal{A}}}
\|
   {\ve{a}} - \ve{c}_{\opt}(\ve{a})\|_2^2 \nonumber\\
 & & = 2\sum_{\ve{a} \in \stream}
\|
   \probe({\ve{a}}) - {\ve{a}}\|_2^2 +2\sum_{\ve{a} \in \stream}
\|
   {\ve{a}} - \ve{c}_{\opt}(\ve{a})\|_2^2 = 2\phi^{\probe}_s +
   2\phi_{\opt} \label{eq22p}
\end{eqnarray}
(using again $(a+b)^2 \leq 2a^2 + 2b^2$). 
Using Theorem \ref{thmain}, this brings the statement of the Theorem.

Figure \ref{f-model} shows that the "quality" of the probe function
(spread $\phi^\probe_s$, stretching factor $\upeta$) stem from the
quality of the Voronoi diagram induced by the synopses in ${\mathcal{S}}$.

\subsubsection*{Proof of Theorem \ref{thOKM}}\label{proof_thOKM}

\begin{table}[t]
\centering
{\footnotesize
\begin{tabular}{cc|ll}\\ \hline\hline
Setting & Algorithm & Probe functions $\probe_t$ & Densities $p_{(\ve{\mu}_.,
  \ve{\theta}_.)}$ \\ \hline
Batch & $k$-means++ \citep{avKM} & Identity & Diracs\\ \hline\hline
Distributed & \protectedKMPP & Identity & Uniform on data
subsets  \\ \hline
Distributed & \privateKMPP & Identity & Non uniform, compact support  \\ \hline
Streaming & \SKM & synopses & Diracs  \\ \hline
On-line & \OKM & point (batch not hit) & Diracs  \\ 
 &  & / closest center (batch hit) &  \\ \hline\hline
\end{tabular}
}
\caption{Synthesis of the parameters for the reductions from
  $k$-\means. We indicate $k$-means++ as the batch clustering solution
  \citep{avKM}.}
\label{tabT1}
\end{table}

The proof proceeds in the same way as for Theorem \ref{thDKM}. The the reduction from $k$-\means~to \OKM~relies on two
things: first, the uniform choice of the first center in $k$-means++
can be replaced by picking the center uniformly in \textit{any} subset of the
data: it does not change the expected approximation properties of the
algorithm (this comes from Lemma 3.4 in \citep{avKM}); therefore,
the choice $q_1 \defeq u_m$ in $k$-\means~can be replaced with $q_1
\defeq u_1$ (uniform with support ${\mathcal{A}}_1$). Second, 
a particular probe function needs to be devised,
sketched in Figure \ref{f-modelOL}. Basically, all probe functions of a
minibatch are the same: each point in the minibatch is probed to
itself, while points occurring outside the minibatch are probed to
their closest center. The reduction proceeds in the following steps: we first let ${\mathcal{A}}$
be the complete set of points in the stream $\stream$. Then, we let
${\mathcal{A}}_j$ denote the set of points of minibatch
$\stream_j$. Remark that minibatch ${\mathcal{A}}_j$ occurs in the
stream before
${\mathcal{A}}_{j'}$ for $j<j'$, and minibatches induce a partition of ${\mathcal{A}}$. Let $j(t)$ denote the batch related
to iteration $t$ in $k$-\means. We define the
following probe function $\probe_t(\ve{a})$ in $k$-\means,
letting
${\mathcal{A}}_j$ the minibatch to which $\ve{a}$
belongs (we do not necessarily have $j=j(t)$):
\begin{itemize}
\item if $j=j(t)$, then $\probe_t(\ve{a}) \defeq \ve{a}$;
\item else $\probe_t(\ve{a}) \defeq \arg\min_{\ve{c} \in
    {\mathcal{C}}} \|\ve{a}-\ve{c}\|_2^2$ (remark that
  $|{\mathcal{C}}|\geq 1$ in this case).
\end{itemize}
Finally, densities $p_{(\ve{\mu}_., \ve{\theta}_.)}$ are Diracs
anchored at selected points, like in $k$-means++.
We get the
equivalence between Algorithm \ref{algoOKM} and $k$-\means~as instantiated. 
\begin{lemma}
With data, densities and probes defined as before, \mbox{$k$-\means}
$\succeq$ \OKM.
\end{lemma}
The proof is immediate, since each minibatch is hit by a center exactly once in
\OKM, and when one subset ${\mathcal{A}}_j$ is hit by a center, then
the probe function makes that \textit{no} other center can be sampled again
from ${\mathcal{A}}_j$ (all contributions to the density $q_t$ are then zero
in ${\mathcal{A}}_j$). We now finish the proof of Theorem
\ref{thOKM} by showing the same approximability ratio for
$k$-\means~as reduced. 
\begin{figure}[t]
\centering
\includegraphics[width=.9\linewidth]{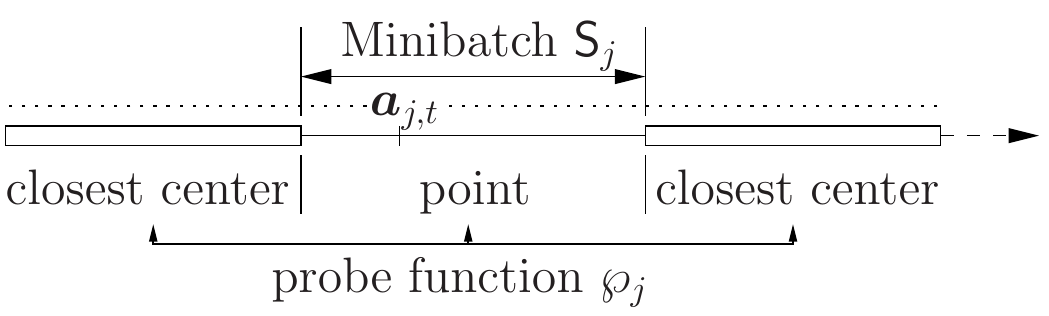}
\caption{Computation of the probe function $\probe_t$ for the reduction
  from \OKM~to $k$-\means, depending on each minibatch stream $\stream_j$.}
\label{f-modelOL}
\end{figure}
Because \textit{optimal} clusters are $\varsigma$-wide with respect
to stream $\stream$, we have
\begin{eqnarray}
\frac{1}{|A|}\cdot \sum_{\ve{a}, \ve{a}' \in A} 
\|\ve{a}-\ve{a}'\|_2^2 & \geq & \varsigma \cdot R\:\:.\nonumber
\end{eqnarray}
Recall that $\ve{c}(A) \defeq (1/|A|)\cdot \sum_{\ve{a}\in A}
\ve{a}$. For any $\ve{a}_0 \in A$, it holds that:
\begin{eqnarray}
\frac{1}{|A|-1}\cdot \sum_{\ve{a}\in A} 
\|\ve{a}-\ve{a}_0\|_2^2 & \geq & \frac{1}{|A|-1}\cdot \sum_{\ve{a}\in A} 
\|\ve{a}-\ve{c}(A)\|_2^2\label{inf11}\\
 & & = \frac{1}{|A|-1}\cdot \left(\frac{1}{2|A|}\cdot \sum_{\ve{a}, \ve{a}' \in A} 
\|\ve{a}-\ve{a}'\|_2^2\right)\label{inf12}\\
 & = & \frac{1}{4}\cdot \frac{2}{|A|(|A|-1)} \cdot \sum_{\ve{a}, \ve{a}' \in A} 
\|\ve{a}-\ve{a}'\|_2^2\nonumber\\
 & \geq & \frac{\varsigma}{4}\cdot R\:\:.\label{inf13}
\end{eqnarray}
Ineq. (\ref{inf11}) holds because $\ve{c}(A)$ is the population
minimizer for optimal cluster $A$ (see \textit{e.g.}, \citep{avKM},
Lemma 2.1). Since probes are points of ${\mathcal{A}}$, 
\begin{eqnarray}
\phi(\probe_j(A); \{\probe_j(\ve{a}_0)\}) & \leq & |A| \cdot
 R\nonumber\\
 & \leq & \frac{4|A|}{\varsigma(|A|-1)}\cdot \sum_{\ve{a}\in A} 
\|\ve{a}-\ve{a}_0\|_2^2\:\:.\label{eqq11}
\end{eqnarray}
On the other hand, we have:
\begin{eqnarray}
\phi(\probe_t(A) ;
  {\mathcal{C}}) & = & \sum_{\ve{a}\in
  A\cap \stream_j} {\|\ve{a} - \ve{c}(\ve{a})\|_2^2}\:\:,
\end{eqnarray}
but since minibatches are $\varsigma$ accurate, $\sum_{\ve{a}\in
  A\cap \stream_j} {\|\ve{a} - \ve{c}(\ve{a})\|_2^2} \geq \varsigma
\cdot \sum_{\ve{a}\in
  A} {\|\ve{a} - \ve{c}(\ve{a})\|_2^2}$. Therefore, for any $\ve{a}_0
\in A$,
\begin{eqnarray}
 \frac{\phi(\probe_t(A) ;
  {\mathcal{C}})}{\phi(\probe_t(A); \{\probe_t(\ve{a}_0)\})} & \geq &
\left(\frac{\varsigma^2(|A|-1)}{4|A|}\right)\cdot \frac{\sum_{\ve{a}\in
  A} {\|\ve{a} - \ve{c}(\ve{a})\|_2^2}}{\sum_{\ve{a}\in A} 
\|\ve{a}-\ve{a}_0\|_2^2}\nonumber\\
 & & = \left(\frac{\varsigma^2(|A|-1)}{4|A|}\right)\cdot \frac{\phi(A ; {\mathcal{C}})}{\phi(A; \{\ve{a}_0\})}\:\:.
\end{eqnarray}
In other words, probe functions are $\upeta$-stretching, for any
$\upeta$ satisfying:
\begin{eqnarray}
\upeta & \geq & \frac{4|A|}{\varsigma^2(|A|-1)}-1\:\:,
\end{eqnarray}
and they are therefore $\upeta$-stretching for $\upeta =
8/\varsigma^2 - 1$. There remains to check that, because of the
densities chosen,
\begin{eqnarray}
\phi_{\bias} & = & \phi_{\opt}\:\:,\\
\phi_{\variance} & = & 0\:\:.
\end{eqnarray}
This ends the proof of Theorem \ref{thOKM}.

\subsection*{Proof of Theorem  \ref{thmdp}}\label{proof_thmdp}

To simplify notations in the proof, we let $p_{\ve{a}}(\ve{x})$ denote the value of
density $p_{(\ve{\mu}_{\ve{a}}, \ve{\theta}_{\ve{a}})}$ on some $\ve{x}\in \Omega$.
Let us denote $Seq(n : k)$ the number of sequences of
integers in set $\{1, 2, ..., n\}$ having exactly $k$ elements, whose cardinal is
$|Seq(n : k)| = n!/(n-k)!$. For any sequence $I \in Seq(n : k)$, we let $I_i$
denote its $i^{th}$ element. For any set ${\mathcal{C}} \defeq \{\ve{c}_1, \ve{c}_2,
..., \ve{c}_k\}$ returned by Algorithm $k$-\means with input 
instance set ${\mathcal{A}} \defeq \{\ve{a}_1, \ve{a}_2, ..., \ve{a}_{n}\} \subset
\Omega$, the density of ${\mathcal{C}}$ given ${\mathcal{A}}$ is:
\begin{eqnarray}
\pr[{\mathcal{C}} | {\mathcal{A}}] & = & \sum_{\ve{\sigma} \in S_k}
\sum_{I \in Seq(n : k)}
p(\ve{\sigma}, I, {\mathcal{C}} | {\mathcal{A}}) \:\:,
\end{eqnarray}
where $S_k$ denotes the symmetric group on $k$ elements, and the
following shorthand is used:
\begin{eqnarray}
  p(\ve{\sigma}, I, {\mathcal{C}} | {\mathcal{A}}) & \defeq & \prod_{i=1}^{k} q_i(\ve{a}_{I_i}) p_{{\ve{a}}_{I_i}}({\ve{c}_{\sigma(i)}})\:\:,\label{defpppp}
\end{eqnarray}
where 
$q_i$ is computed using eq. (\ref{choosea}) and taking into
account the modification due to the choice of each $I_j$ for $j<i$ in
the sequence $I$. 

In the following, we let ${\mathcal{A}}$ and ${\mathcal{A}}'$ denote
two sets of points that differ from one $a$
(they have the same size), say $\ve{a}_n \in {\mathcal{A}}$ and $\ve{a}'_n \in
{\mathcal{A}}'$, $\ve{a}_n \neq \ve{a}'_n$. We
analyze:
\begin{eqnarray}
\frac{\pr[{\mathcal{C}} | {\mathcal{A}}']}{\pr[{\mathcal{C}} | {\mathcal{A}}]} & = & \frac{\sum_{\ve{\sigma} \in S_k}
\sum_{I \in Seq(n : k)}
p(\ve{\sigma}, I, {\mathcal{C}} | {\mathcal{A}}')}{\sum_{\ve{\sigma} \in S_k}
\sum_{I \in Seq(n : k)}
p(\ve{\sigma}, I, {\mathcal{C}} | {\mathcal{A}})}\:\:.\label{ratio}
\end{eqnarray}
Using the definition of $q(.)$, we refine $p(\ve{\sigma}, I, {\mathcal{C}} |
{\mathcal{A}})$ as
\begin{eqnarray}
p(\ve{\sigma}, I, {\mathcal{C}} | {\mathcal{A}}) & = &
\frac{N(I)}{\prod_{i=1}^{k} M(I^i | {\mathcal{A}})}\cdot
\prod_{i=1}^k p_{{\ve{a}}_{I_i}}({\ve{c}_{\sigma(i)}})\:\:, \label{defpsig}
\end{eqnarray}
where
\begin{eqnarray}
N(I) & \defeq & \prod_{i=2}^{j} \|\ve{a}_{I_i} -
\nn_{I^i}(\ve{a}_{I_i})\|_2^2\:\:,\\
M(I^i | {\mathcal{A}}) & \defeq & \left\{
\begin{array}{ccl}
n & \mbox{ if } & i = 1\\
\sum_{j=1}^{n}{\|\ve{a}_j - \nn_{I^i}(\ve{a}_j)\|^2_2} &
\multicolumn{2}{c}{\mbox{ otherwise}}
\end{array}
\right.\:\:,
\end{eqnarray}
and $I^i$ is the prefix sequence $I_1, I_2, ..., I_{i-1}$, and
$\nn_{I^i}(a) \defeq \arg\min_{j\leq i-1} \|a -
\ve{a}_{I_j}\|_2$ is the nearest neighbor of $a$ in the prefix
sequence. Notice that there is a factor $1/m$ for $q(.)$ at the first
iteration that we omit in $N(I)$ since it disappears in the ratio in
eq. (\ref{ratio}).

We analyze separately each element in (\ref{defpsig}), starting with
$N(I)$. We define the \textit{swapping} operation $s_\ell(I)$ that
returns the sequence in which $\ve{a}_{I_\ell}$ and $\ve{a}_{I_{\ell+1}}$ are
permuted, for $1\leq \ell \leq k-1$. This incurs non-trivial modifications in $N(s_\ell(I))$
compared to $N(I)$, since the nearest neighbors of $\ve{a}_{I_\ell}$ and
$\ve{a}_{I_{\ell+1}}$ may change in the permutation:
\begin{eqnarray}
N(s_\ell(I)) & = & 
\prod_{i=2}^{\ell-1} \|\ve{a}_{I_i} -
\nn_{I^i}(\ve{a}_{I_i})\|_2^2 \nonumber\\
 & & \cdot \underbrace{\|\ve{a}_{I_{\ell+1}} -
\nn_{I^\ell}(\ve{a}_{I_{\ell+1}})\|_2^2 \cdot \|\ve{a}_{I_{\ell}} -
\nn_{I^\ell \cup \{I_{\ell + 1}\}}(\ve{a}_{I_{\ell}})\|_2^2}_{\neq \|\ve{a}_{I_{\ell}} -
\nn_{I^\ell }(\ve{a}_{I_{\ell}})\|_2^2 \cdot \|\ve{a}_{I_{\ell+1}} -
\nn_{I^{\ell+1}}(\ve{a}_{I_{\ell+1}})\|_2^2} \nonumber\\
 & & \cdot \prod_{i=\ell+2}^{k} \|\ve{a}_{I_i} -
\nn_{I^i}(\ve{a}_{I_i})\|_2^2 
\end{eqnarray}
($I\cup \{j\}$ indicates that element $j$ is
put at the end of the sequence). We want to quantify the maximal
increase in $N(s_\ell(I))$ compared to $N(I)$. The following Lemma
shows that the maximal increase ratio is actually a constant, and thus
does not depend on the data.
\begin{lemma}\label{lemeta1}
The following holds true:
\begin{eqnarray}
N(s_1(I)) & = & N(I)\:\:,\\
N(s_\ell(I))
 & \leq & (1+\eta)^2  N(I) \:\:, \forall 2\leq \ell\leq k-1\:\:.
\end{eqnarray}
Here, $0\leq \eta\leq 3$ is a constant.
\end{lemma}
The proof stems directly from the following Lemma.
\begin{lemma}\label{lslide}
  For any non-empty ${\mathcal{N}} \subseteq {\mathcal{A}}$ and $x \in
 \Omega$, let $\nn_{\mathcal{N}}(x)$ denote the nearest
  neighbor of $x$ in ${\mathcal{N}}$. There exists a constant $0\leq \eta \leq 3$ such that for any $\ve{a}_i, \ve{a}_j
\in {\mathcal{A}}$ and any nonempty subset ${\mathcal{N}}
\subseteq {\mathcal{A}} \backslash \{\ve{a}_i, \ve{a}_j\}$,
\begin{eqnarray}
\frac{\|\ve{a}_i - \nn_{\mathcal{N}}(\ve{a}_i)\|_2}{\|\ve{a}_i - \nn_{\mathcal{N}
    \cup\{\ve{a}_j\}}(\ve{a}_i)\|_2} & \leq & (1+\eta) \cdot \frac{\|\ve{a}_j -
  \nn_{\mathcal{N}}(\ve{a}_j)\|_2}{\|\ve{a}_j - \nn_{\mathcal{N} \cup\{\ve{a}_i\}}(\ve{a}_j)\|_2}\:\:.\label{eq1111}
\end{eqnarray}
\end{lemma}
\begin{proof}
Since $\|\ve{a}_j - \nn_{\mathcal{N} \cup\{\ve{a}_i\}}(\ve{a}_j)\|_2\leq \|\ve{a}_j -
  \nn_{\mathcal{N}}(\ve{a}_j)\|_2$, the proof is true for $\eta = 0$ when
  $\nn_{\mathcal{N}}(\ve{a}_i) = \nn_{\mathcal{N}
    \cup\{\ve{a}_j\}}(\ve{a}_i)$. So suppose that $\nn_{\mathcal{N}}(\ve{a}_i) \neq \nn_{\mathcal{N}
    \cup\{\ve{a}_j\}}(\ve{a}_i)$, implying $\nn_{\mathcal{N}
    \cup\{\ve{a}_j\}}(\ve{a}_i) = \ve{a}_j$. We distinguish two cases. 

\noindent \textbf{Case 1/2}, if $\nn_{\mathcal{N}
    \cup\{\ve{a}_i\}}(\ve{a}_j) = \ve{a}_i$, then we are reduced to showing that $\|\ve{a}_i - \nn_{\mathcal{N}}(\ve{a}_i)\|_2\leq (1+\eta)
  \|\ve{a}_j - \nn_{\mathcal{N}}(\ve{a}_j)\|_2$ under the conditions (C) that
  ${\mathcal{N}} \cap B(\ve{a}_i, \|\ve{a}_i - \ve{a}_j\|_2) = \emptyset$ and
  ${\mathcal{N}} \cap B(\ve{a}_j, \|\ve{a}_i - \ve{a}_j\|_2) =
  \emptyset$. Here, $B(\ve{a},r)$ denotes the open ball of center
  $\ve{a}$ and radius $R$. The
  triangle inequality and conditions (C) bring
\begin{eqnarray}
\|\ve{a}_i - \nn_{\mathcal{N}}(\ve{a}_i)\|_2 & \leq & \|\ve{a}_i - \ve{a}_j\|_2 + \|\ve{a}_j -
\nn_{\mathcal{N}}(\ve{a}_i)\|_2\nonumber\\
 & \leq & \|\ve{a}_j -
\nn_{\mathcal{N}}(\ve{a}_j)\|_2+\|\ve{a}_j -
\nn_{\mathcal{N}}(\ve{a}_i)\|_2\:\:.\label{adm1}
\end{eqnarray}
If $\nn_{\mathcal{N}}(\ve{a}_i)=\nn_{\mathcal{N}}(\ve{a}_j)$ then the inequality
holds for $\eta = 1$. Otherwise, suppose that $\|\ve{a}_j -
\nn_{\mathcal{N}}(\ve{a}_i)\|_2 >  3 \|\ve{a}_j -
\nn_{\mathcal{N}}(\ve{a}_j)\|_2$. The triangle inequality yields again $\|\ve{a}_j -
\nn_{\mathcal{N}}(\ve{a}_i)\|_2 \leq \|\ve{a}_j -
\ve{a}_i\|_2 + \|\ve{a}_i -
\nn_{\mathcal{N}}(\ve{a}_i)\|_2$, and so we have the inequality:
\begin{eqnarray}
3 \|\ve{a}_j -
\nn_{\mathcal{N}}(\ve{a}_j)\|_2 & <& \|\ve{a}_j -
\ve{a}_i\|_2 + \|\ve{a}_i -
\nn_{\mathcal{N}}(\ve{a}_i)\|_2\:\:,
\end{eqnarray}
and since (C) holds, $\|\ve{a}_j -
\nn_{\mathcal{N}}(\ve{a}_j)\|_2 \geq \|\ve{a}_j -
\ve{a}_i\|_2$ which implies 
\begin{eqnarray}
\|\ve{a}_j -
\nn_{\mathcal{N}}(\ve{a}_j)\|_2 < \frac{1}{2} \cdot \|\ve{a}_i -
\nn_{\mathcal{N}}(\ve{a}_i)\|_2 \:\:.\label{ad0}
\end{eqnarray}
On the other hand, the triangle inequality brings again
\begin{eqnarray}
\|\ve{a}_i -
\nn_{\mathcal{N}}(\ve{a}_j)\|_2 & \leq & \|\ve{a}_i - \ve{a}_j\|_2 + \|\ve{a}_j -
\nn_{\mathcal{N}}(\ve{a}_j)\|_2\nonumber\\
 & \leq & 2 \cdot \|\ve{a}_j -
\nn_{\mathcal{N}}(\ve{a}_j)\|_2\label{ad1}\\
 & < & 2 \cdot \frac{1}{2} \cdot \|\ve{a}_i -
\nn_{\mathcal{N}}(\ve{a}_i)\|_2 = \|\ve{a}_i -
\nn_{\mathcal{N}}(\ve{a}_i)\|_2 \:\:, \label{ad2}
\end{eqnarray}
a contradiction since $\|\ve{a}_i -
\nn_{\mathcal{N}}(\ve{a}_i)\|_2 \leq \|\ve{a}_i -
\ve{a}_l\|_2, \forall \ve{a}_l \in {\mathcal{N}}$ by definition. Ineq. (\ref{ad1}) uses (C) and ineq. (\ref{ad2}) uses
ineq. (\ref{ad0}). Hence, if $\nn_{\mathcal{N}}(\ve{a}_i)\neq
\nn_{\mathcal{N}}(\ve{a}_j)$ then since $\|\ve{a}_j -
\nn_{\mathcal{N}}(\ve{a}_i)\|_2 \leq  3 \|\ve{a}_j -
\nn_{\mathcal{N}}(\ve{a}_j)\|_2$, ineq. (\ref{adm1}) brings  $\|\ve{a}_i - \nn_{\mathcal{N}}(\ve{a}_i)\|_2\leq 4\cdot
  \|\ve{a}_j - \nn_{\mathcal{N}}(\ve{a}_j)\|_2$, and the inequality
holds for $\eta = 3$.

\noindent \textbf{Case 2/2}, if $\nn_{\mathcal{N}
    \cup\{\ve{a}_i\}}(\ve{a}_j) \neq \ve{a}_i$, then it implies $\nn_{\mathcal{N}
    \cup\{\ve{a}_i\}}(\ve{a}_j) = \nn_{\mathcal{N}}(\ve{a}_j)$ and so
\begin{eqnarray}
\exists \ve{a}_* \in {\mathcal{N}} : \|\ve{a}_j - \ve{a}_*\|_2 & \leq \|\ve{a}_j - \ve{a}_i\|_2\:\:.
\end{eqnarray}
Ineq. (\ref{eq1111}) reduces to
  proving
\begin{eqnarray}
  \|\ve{a}_i - \nn_{\mathcal{N}}(\ve{a}_i)\|_2 & \leq & (1+\eta) \cdot \|\ve{a}_i - \ve{a}_j\|_2 \:\:,\label{lastineq}
\end{eqnarray}
but $\|\ve{a}_i - \ve{a}_*\|_2 \leq \|\ve{a}_i - \ve{a}_j\|_2 + \|\ve{a}_j - \ve{a}_*\|_2 \leq 2
\|\ve{a}_i - \ve{a}_j\|_2$, and since $\ve{a}_* \in {\mathcal{N}}$, $\|\ve{a}_i -
\nn_{\mathcal{N}}(\ve{a}_i)\|_2 \leq \|\ve{a}_i - \ve{a}_*\|_2 \leq 2 \|\ve{a}_i -
\ve{a}_j\|_2$, and (\ref{lastineq}) is proved for $\eta = 1$. This achieves
the proof of Lemma \ref{lslide}.
\end{proof}
Let $I$ be any sequence not containing the index of $\ve{a}'_n$, and let
$I(i)$ denote the sequence in which we replace $\ve{a}_{I_i}$ by the index
of $\ve{a}'_n$. The sequence of swaps 
\begin{eqnarray}
I(k) & = & (s_{k-1}  \circ ... \circ
s_{i+1}  \circ s_{i})  (I(i)) 
\end{eqnarray} 
produces a sequence $I(k)$ in which all elements different from $\ve{a}'_n$ are in
the same relative order as they are in $I$ with respect to each other,
and $\ve{a}'_n$ is pushed to the end of the sequence in $k^{th}$ rank. We
also have 
\begin{eqnarray}
N(I(i)) & \leq & (1+\eta)^{2(k-i)}N(I(k))\:\:.
\end{eqnarray}

All the properties we need on $N(.)$ are now established. We turn to the analysis of $M(I^i| {\mathcal{A}})$. 
\begin{lemma}\label{ll0}
For any $\updeltas>0$ such that ${\mathcal{A}}$ is
$\updeltas$-monotonic, the following holds. For any
  ${\mathcal{N}} \subseteq {\mathcal{A}}$ with $| {\mathcal{N}}| \in 
\{1, 2, ..., k-1\}$, $\forall \ve{x}, \ve{x}' \in\Omega$, we have:
\begin{eqnarray}
\sum_{\ve{a} \in {\mathcal{A}}} \|\ve{a} - \nn_{\mathcal{N}\cup \{\ve{x}\}}(\ve{a})\|_2^2
& \leq & (1+\updeltas) \cdot \sum_{\ve{a} \in {\mathcal{A}}} \|\ve{a} - \nn_{\mathcal{N}\cup \{\ve{x}'\}}(\ve{a})\|_2^2\:\:.
\end{eqnarray} 
\end{lemma}
\begin{proof}
Since adding a point to ${\mathcal{N}}$ cannot increase the potential
$\sum_{\ve{a} \in {\mathcal{A}}} \|\ve{a} - \nn_{\mathcal{N}\cup
  \{\ve{x}\}}(\ve{a})\|_2^2$, it comes 
\begin{eqnarray}
\sum_{\ve{a} \in {\mathcal{A}}} \|\ve{a} -
\nn_{\mathcal{N}\cup \{\ve{x}\}}(\ve{a})\|_2^2 & \leq & \sum_{\ve{a} \in {\mathcal{A}}}
\|\ve{a} - \nn_{\mathcal{N}}(\ve{a})\|_2^2 \:\:, \forall \ve{x}\in\Omega\:\:.\label{eqq1}
\end{eqnarray} 
Consider any $\ve{x}'\in\Omega$ such that $\sum_{\ve{a} \in {\mathcal{A}}} \|\ve{a} -
\nn_{\mathcal{N}\cup \{\ve{x}'\}}(\ve{a})\|_2^2 = \sum_{\ve{a} \in {\mathcal{A}}}
\|\ve{a} - \nn_{\mathcal{N}}(\ve{a})\|_2^2$, \textit{i.e.}, all points
of ${\mathcal{A}}$ are closer to a point in $\mathcal{N}$ than they
are from $\ve{x}'$. In this case, we obtain from ineq. (\ref{eqq1}),
\begin{eqnarray}
\sum_{\ve{a} \in {\mathcal{A}}} \|\ve{a} -
\nn_{\mathcal{N}\cup \{\ve{x}\}}(\ve{a})\|_2^2 & \leq & \sum_{\ve{a} \in {\mathcal{A}}} \|\ve{a} -
\nn_{\mathcal{N}\cup \{\ve{x}'\}}(\ve{a})\|_2^2\:\:,
\end{eqnarray}
and since $\updeltas>0$, the statement of the Lemma holds. 

More
interesting is the case where $\ve{x}'\in\Omega$ is such that $\sum_{\ve{a} \in {\mathcal{A}}} \|\ve{a} -
\nn_{\mathcal{N}\cup \{\ve{x}'\}}(\ve{a})\|_2^2 < \sum_{\ve{a} \in {\mathcal{A}}}
\|\ve{a} - \nn_{\mathcal{N}}(\ve{a})\|_2^2$, implying $\ve{x}' \not
\in {\mathcal{N}}$. In this case, let $A \defeq \{\ve{a}
\in {\mathcal{A}} : \nn_{{\mathcal{N}} \cup \{\ve{x}'\}}(\ve{a}) = \ve{x}'\}$, which is
then non-empty. Let us denote for short $\ve{c}(A) \defeq
(1/|A|) \cdot \sum_{\ve{a}\in A}
\ve{a}$. Since $\ve{x}' \not
\in {\mathcal{N}}$, $A\cap {\mathcal{N}} = \emptyset$, and since
${\mathcal{A}}$ is $\updeltas$-monotonic, then it comes from
ineq. (\ref{eqq1})
\begin{eqnarray}
\sum_{\ve{a} \in {\mathcal{A}}} \|\ve{a} -
\nn_{\mathcal{N}\cup \{\ve{x}\}}(\ve{a})\|_2^2 & \leq & (1+\updeltas) \cdot \sum_{\ve{a} \in {\mathcal{A}}} \|\ve{a} -
\nn_{\mathcal{N}\cup \{\ve{c}(A)\}}(\ve{a})\|_2^2\:\:.\label{eqqq3}
\end{eqnarray}
We have:
\begin{eqnarray}
\sum_{\ve{a} \in {\mathcal{A}}} \|\ve{a} -
\nn_{\mathcal{N}\cup \{\ve{c}(A)\}}(\ve{a})\|_2^2 & = & \sum_{\ve{a} \in
  {\mathcal{A}} \backslash A} \|\ve{a} -
\nn_{\mathcal{N}\cup \{\ve{c}(A)\}}(\ve{a})\|_2^2 + \sum_{\ve{a} \in
  A} \|\ve{a} -
\nn_{\mathcal{N}\cup \{\ve{c}(A)\}}(\ve{a})\|_2^2\nonumber\\
 & \leq & \sum_{\ve{a} \in
  {\mathcal{A}} \backslash A} \|\ve{a} -
\nn_{\mathcal{N}\cup \{\ve{c}(A)\}}(\ve{a})\|_2^2 + \sum_{\ve{a} \in
  A} \|\ve{a} -
\ve{c}(A)\|_2^2\nonumber\\
 & \leq & \sum_{\ve{a} \in
  {\mathcal{A}} \backslash A} \|\ve{a} -
\nn_{\mathcal{N}\cup \{\ve{c}(A)\}}(\ve{a})\|_2^2 + \sum_{\ve{a} \in
  A} \|\ve{a} -\ve{x}'\|_2^2\label{lasteqq}\:\:.
\end{eqnarray}
Eq. (\ref{lasteqq}) holds because the arithmetic average is the
population minimizer of $L_2^2$. Because of the definition of $A$, 
\begin{eqnarray}
\sum_{\ve{a} \in
  {\mathcal{A}} \backslash A} \|\ve{a} -
\nn_{\mathcal{N}\cup \{\ve{c}(A)\}}(\ve{a})\|_2^2 & \leq & \sum_{\ve{a} \in
  {\mathcal{A}} \backslash A} \|\ve{a} -
\nn_{\mathcal{N}}(\ve{a})\|_2^2\nonumber\\
 & & = \sum_{\ve{a} \in
  {\mathcal{A}} \backslash A} \|\ve{a} -
\nn_{\mathcal{N}\cup \{\ve{x}'\}}(\ve{a})\|_2^2\:\:,\label{eqqq1}
\end{eqnarray}
and, still because of the definition of $A$,
\begin{eqnarray}
\sum_{\ve{a} \in
  A} \|\ve{a} -\ve{x}'\|_2^2 & = & \sum_{\ve{a} \in
  A} \|\ve{a} -
\nn_{\mathcal{N}\cup \{\ve{x}'\}}(\ve{a})\|_2^2\:\:,\label{eqqq2}
\end{eqnarray}
so we get from (\ref{eqqq1}) and (\ref{eqqq2}) $\sum_{\ve{a} \in
  {\mathcal{A}} \backslash A} \|\ve{a} -
\nn_{\mathcal{N}\cup \{\ve{c}(A)\}}(\ve{a})\|_2^2 + \sum_{\ve{a} \in
  A} \|\ve{a} -\ve{x}'\|_2^2\leq \sum_{\ve{a} \in
  {\mathcal{A}}} \|\ve{a} -
\nn_{\mathcal{N}\cup \{\ve{x}'\}}(\ve{a})\|_2^2$, and finally from
ineq. (\ref{lasteqq}),
\begin{eqnarray}
\sum_{\ve{a} \in {\mathcal{A}}} \|\ve{a} -
\nn_{\mathcal{N}\cup \{\ve{c}(A)\}}(\ve{a})\|_2^2 & \leq & \sum_{\ve{a} \in
  {\mathcal{A}}} \|\ve{a} -
\nn_{\mathcal{N}\cup \{\ve{x}'\}}(\ve{a})\|_2^2 \label{lleqq}\:\:,
\end{eqnarray}
which, using ineq. (\ref{eqqq3}), completes the proof of Lemma \ref{ll0}.
\end{proof}

\begin{lemma}\label{lemdev2}
The following holds true, for any $i>1$, any ${\mathcal{A}}'
\approx {\mathcal{A}}$, any $\updeltaw, \updeltas>0$: 
\begin{eqnarray}
\mbox{${\mathcal{A}}$ is $\updeltaw$-spread} & \Rightarrow & (n\not\in I^i
\Rightarrow M(I^i| {\mathcal{A}})  \leq
(1+\updeltaw) \cdot M(I^i| {\mathcal{A}}'))\label{pp1}\:\:,\\
\mbox{${\mathcal{A}}$ is $\updeltas$-monotonic} & \Rightarrow & (n\in I^i
\Rightarrow M(I^i| {\mathcal{A}})  \leq
(1+\updeltas) \cdot M(I^i| {\mathcal{A}}'))\label{pp12}\:\:.
\end{eqnarray}
\end{lemma}
\begin{proof}
Suppose first that $n \not\in I^i$. In this case, since ${\mathcal{A}}$ is $\updeltaw$-spread,
\begin{eqnarray}
M(I^i| {\mathcal{A}}) & = & \sum_{j=1}^{n}{\|\ve{a}_j -
  \nn_{I^i}(\ve{a}_j)\|^2_2}\nonumber\\
 & = & \sum_{j=1}^{n-1}{\|\ve{a}_j -
  \nn_{I^i}(\ve{a}_j)\|^2_2} + \|\ve{a}_n - \nn_{I^i}(\ve{a}_j)\|_2^2 \nonumber\\
 & \leq & \sum_{j=1}^{n-1}{\|\ve{a}_j -
  \nn_{I^i}(\ve{a}_j)\|^2_2} + R^2 \nonumber\\
 & \leq & (1+\updeltaw) \cdot \sum_{j=1}^{n-1}{\|\ve{a}_j -
  \nn_{I^i}(\ve{a}_j)\|^2_2} \label{useassum}\\
 &  \leq & (1+\updeltaw) \cdot \left(\sum_{j=1}^{n-1}{\|\ve{a}_j -
  \nn_{I^i}(\ve{a}_j)\|^2_2} + \|\ve{a}'_n -
\nn_{I^i}(\ve{a}'_n)\|^2_2\right)\nonumber\\
 & & = (1+\updeltaw) \cdot M(I^i| {\mathcal{A}}')\:\:,
\end{eqnarray}
as indeed computing the nearest neighbors do not involve the $n^{th}$
element of the sets, \textit{i.e.} $\ve{a}_n$ or $\ve{a}'_n$. We have
used in ineq. (\ref{useassum}) the fact that ${\mathcal{A}}$ is $\updeltaw$-spread.

When $n \in I^i$, eq. (\ref{pp12}) is an immediate consequence of
Lemma \ref{ll0} in which the
distinct elements of ${\mathcal{A}}$ and ${\mathcal{A}}'$ play the
role of $\ve{x}$ and $\ve{x}'$.
\end{proof}

\begin{lemma}\label{lemvar2}
For any $\updeltaw > 0$, if ${\mathcal{A}}$ is $\updeltaw$-spread,
then for any
  ${\mathcal{N}} \subseteq {\mathcal{A}}$ with $| {\mathcal{N}}| =
  k-1$, $\forall \ve{x} \in\Omega$,  it holds that $\|\ve{x} -
  \nn_{\mathcal{N}}(\ve{x})\|_2^2\leq \updeltaw \sum_{\ve{a} \in {\mathcal{A}}} {\|\ve{a} - \nn_{\mathcal{N}}(\ve{a})\|_2^2}$.
\end{lemma}
\begin{proof}
Follows directly from the fact that $\|\ve{x} -
  \nn_{\mathcal{N}}(\ve{x})\|_2^2\leq R^2$ by assumption.
\end{proof}
Letting $I(k)$ denote a sequence containing element $n$ pushed to the
end of the sequence, we get:
\begin{eqnarray}
\lefteqn{\sum_{\ve{\sigma} \in S_k}
\sum_{I \in Seq^+(n : k)}
p(\ve{\sigma}, I, {\mathcal{C}} | {\mathcal{A}}')}\nonumber\\
 & = & \sum_{\ve{\sigma} \in S_k}
\sum_{I \in Seq^+(n : k)}
\frac{N(I)}{\prod_{i=1}^{k} M(I^i | {\mathcal{A}}')}\cdot
p_{{\ve{a}'}_{n}}({\ve{c}_{\sigma(i)}}) \cdot
\prod_{i=1 : I_i \neq n}^k p_{{\ve{a}}_{I_i}}({\ve{c}_{\sigma(i)}}) 
\nonumber\\
& \leq & (1+\eta)^{2(k-2)} \nonumber\\
 & & \cdot \sum_{\ve{\sigma} \in S_k}
\sum_{I \in Seq^+(n : k)}
\frac{N(I(k))}{\prod_{i=1}^{k} M(I^i | {\mathcal{A}}')}\cdot p_{{\ve{a}'}_{n}}({\ve{c}_{\sigma(i)}}) \cdot
\prod_{i=1 : I_i \neq n}^k p_{{\ve{a}}_{I_i}}({\ve{c}_{\sigma(i)}}) \label{eq2}\:\:.
\end{eqnarray}
Now, take any element $I \in Seq_+(n : k)$ with $\ve{a}'_n$ in position
$k$, and change $\ve{a}'_n$ by some $\ve{a} \in {\mathcal{A}}$. Any of these changes generates a different element $I'
\in Seq^-(n:k)$, and so using Lemma \ref{lemvar2} and the following
two facts:
\begin{itemize}
\item the fact that
\begin{eqnarray}
p_{{\ve{a}'}_{n}}({\ve{c}_{\sigma(i)}}) & \leq & \varrho(R)
\cdot p_{{\ve{a}}}({\ve{c}_{\sigma(i)}})  \:\:,\label{pbsup} 
\end{eqnarray}
for any $\ve{a} \in {\mathcal{A}}$,
\item the fact that, if ${\mathcal{A}}$ is $\updeltas$-monotonic,
\begin{eqnarray}
M(I_{{\ve{a}}}^i | {\mathcal{A}}) & \leq & (1+\updeltas)\cdot M(I^i |
{\mathcal{A}})\:\:,
\end{eqnarray} 
for any ${\ve{a}}\in {\mathcal{A}}$ not already in the sequence, where $I_{\ve{a}}$ denotes the
sequence $I$ in which $\ve{a}'_n$ has been replaced by ${\ve{a}}$, 
\end{itemize}
we
get from ineq. (\ref{eq2}),
\begin{eqnarray}
\lefteqn{\sum_{\ve{\sigma} \in S_k}
\sum_{I \in Seq^+(n : k)}
p(\ve{\sigma}, I, {\mathcal{C}} | {\mathcal{A}}')}\nonumber\\
 & \leq & (1+\eta)^{2(k-2)} \cdot (1+\updeltas)^{k-1} \cdot \updeltaw
 \nonumber\\
 & &\cdot
 \varrho(R) \cdot \sum_{\ve{\sigma} \in S_k}
\sum_{I \in Seq^-(n : k)}
\frac{N(I)}{\prod_{i=1}^{k} M(I^i | {\mathcal{A}})}\cdot 
\prod_{i=1}^k p_{{\ve{a}}_{I_i}}({\ve{c}_{\sigma(i)}}) \label{eq34}\:\:.
\end{eqnarray}

\begin{lemma}\label{lemlast}
For any $\updeltaw, \updeltas > 0$ such that ${\mathcal{A}}$ is
$\updeltaw$-spread and $\updeltas$-monotonic, for any ${\mathcal{A}}'
\approx {\mathcal{A}}$, we have:
\begin{eqnarray}
\frac{\pr[{\mathcal{C}} | {\mathcal{A}}']}{\pr[{\mathcal{C}} | {\mathcal{A}}]}  & \leq &
(1+\updeltaw)^{k-1} \cdot \left(1 + \updeltaw  \cdot \left(
    \frac{1+\updeltas}{1+\updeltaw}\right)^{k-1} \cdot  (1+\eta)^{2(k-2)}  \cdot
 \varrho(R)\right)  \:\:.
\end{eqnarray}
\end{lemma}
\begin{proof}
We get from the fact that ${\mathcal{A}}$ is $\updeltaw$-spread,
\begin{eqnarray}
\sum_{\ve{\sigma} \in S_k}
\sum_{I \in Seq_-(n : k)}
p(\ve{\sigma}, I, {\mathcal{C}} | {\mathcal{A}}')  & \leq & (1+\updeltaw)^{k-1} \cdot \sum_{\ve{\sigma} \in S_k}
\sum_{I \in Seq_-(n : k)}
p(\ve{\sigma}, I, {\mathcal{C}} | {\mathcal{A}})\:\:,\label{eq234}
\end{eqnarray}
and furthermore ineq. (\ref{eq34}) yields:
\begin{eqnarray}
\frac{\pr[{\mathcal{C}} | {\mathcal{A}}']}{\pr[{\mathcal{C}} | {\mathcal{A}}]} & = & \frac{\sum_{\ve{\sigma} \in S_k}
\sum_{I \in Seq(n : k)}
p(\ve{\sigma}, I, {\mathcal{C}} | {\mathcal{A}}')}{\sum_{\ve{\sigma} \in S_k}
\sum_{I \in Seq(n : k)}
p(\ve{\sigma}, I, {\mathcal{C}} | {\mathcal{A}})} 
\nonumber\\
 & \leq &
\frac{
\left(
\begin{array}{c}
(1+\updeltaw)^{k-1} \cdot \sum_{\ve{\sigma} \in S_k}
\sum_{I \in Seq_-(n : k)}
p(\ve{\sigma}, I, {\mathcal{C}} | {\mathcal{A}}) \\
+ \\
\sum_{\ve{\sigma} \in S_k}
\sum_{I \in Seq_+(n : k)}
p(\ve{\sigma}, I, {\mathcal{C}} | {\mathcal{A}}')
\end{array}
\right)}{\sum_{\ve{\sigma} \in S_k}
\sum_{I \in Seq(n : k)}
p(\ve{\sigma}, I, {\mathcal{C}} | {\mathcal{A}})} \nonumber\\
 & \leq &(1+\updeltaw)^{k-1} \nonumber\\
 & & \cdot 
\frac{
\left(
\begin{array}{c}
\sum_{\ve{\sigma} \in S_k}
\sum_{I \in Seq_-(n : k)}
p(\ve{\sigma}, I, {\mathcal{C}} | {\mathcal{A}}) \\
+\\
\updeltaw  \cdot \left(
    \frac{1+\updeltas}{1+\updeltaw}\right)^{k-1} \cdot  (1+\eta)^{2(k-2)} \cdot
  \varrho(R)\cdot \sum_{\ve{\sigma} \in S_k}
\sum_{I \in Seq_-(n : k)}
p(\ve{\sigma}, I, {\mathcal{C}} | {\mathcal{A}}')
\end{array}
\right)}{\sum_{\ve{\sigma} \in S_k}
\sum_{I \in Seq(n : k)}
p(\ve{\sigma}, I, {\mathcal{C}} | {\mathcal{A}})} \nonumber\\
 &  & =(1+\updeltaw)^{k-1} \cdot \left(1 + \updeltaw  \cdot \left(
    \frac{1+\updeltas}{1+\updeltaw}\right)^{k-1} \cdot  (1+\eta)^{2(k-2)} \cdot
  \varrho(R)\right) \nonumber\\
 & & \qquad\vspace{1cm}\cdot 
\underbrace{\frac{\sum_{\ve{\sigma} \in S_k}
\sum_{I \in Seq_-(n : k)}
p(\ve{\sigma}, I, {\mathcal{C}} | {\mathcal{A}})}{\sum_{\ve{\sigma} \in S_k}
\sum_{I \in Seq(n : k)}
p(\ve{\sigma}, I, {\mathcal{C}} | {\mathcal{A}})}}_{\leq 1} \:\:.\nonumber
\end{eqnarray}
This ends the proof of Lemma \ref{lemlast}. 
\end{proof}
Since
\begin{eqnarray}
\lefteqn{(1+\updeltaw)^{k-1} \cdot \left(1 + \updeltaw  \cdot \left(
    \frac{1+\updeltas}{1+\updeltaw}\right)^{k-1} \cdot  (1+\eta)^{2(k-2)}  \cdot
 \varrho(R)\right)}\nonumber\\
 & = & (1+\updeltaw)^{k-1} + (1+\eta)^{2(k-2)} \cdot  \updeltaw  \cdot
 (1+\updeltas)^{k-1} \cdot
 \varrho(R)\nonumber\:\:,
\end{eqnarray}
and $\upeta \leq 3$ from Lemma \ref{lemeta1},
we get Theorem \ref{thmdp} with 
\begin{eqnarray}
f(k) & \defeq & 4^{2k-4}\:\:. \label{defff}
\end{eqnarray}

\subsection*{Proof of Theorem  \ref{thlabf}}\label{proof_thlabf}

\begin{figure}[t]
\centering
\includegraphics[width=.8\linewidth]{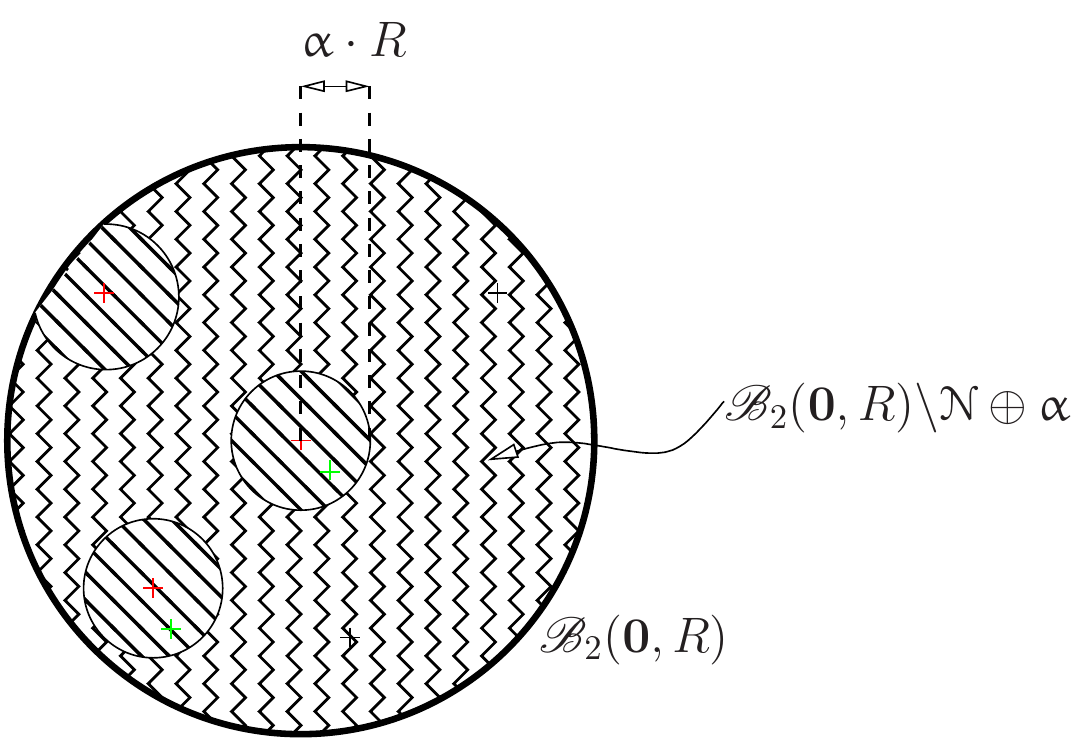}
\caption{$q_*$ in eq. (\ref{minsup}) measures the probability the a
  point drawn in ${\mathscr{B}}_2(\ve{0}, R)$ escapes the neighborhoods
  of ${\mathcal{N}} \oplus \upalpha$. In this example, two points in
  black escape the neighborhoods (defined by three points in 
  {\color{red}red}), while two in 
  {\color{green}green} do not.}
\label{f-model3}
\end{figure}

Assume that density ${\mathcal{D}}$ contains a $L_2$ ball
${\mathscr{B}_2}(\ve{0}, R)$ of radius $R$, centered without loss of
generality in $\ve{0}$. Fix $0<\upkappa <m-1$.
For any $\upalpha \in (0,1)$ and ${\mathcal{N}} \subseteq
{\mathcal{A}}$ with $|{\mathcal{N}}| \in \{1, 2, ..., \upkappa\}
\defeq [ \upkappa]_*$, let
${\mathcal{N}} \oplus \upalpha \defeq \cup_{\ve{x}\in {\mathcal{N}}}
{\mathscr{B}_2}(\ve{x}, \upalpha\cdot R)$ be the union of all small
balls centered around each element of ${\mathcal{N}}$, each of radius
$\upalpha\cdot R$. An important quantity is
\begin{eqnarray}
q_* & \defeq & \min_{{\mathcal{N}} \subseteq {\mathcal{A},
    |{\mathcal{N}}| \in [ \upkappa]_*}} \frac{\mu({\mathscr{B}_2}(\ve{0}, R) \backslash {\mathcal{N}}
\oplus \upalpha)}{\mu({\mathscr{B}_2}(\ve{0},
R))}\label{minsup}
\end{eqnarray}
the minimal mass of ${\mathscr{B}_2}(\ve{0}, R) \backslash {\mathcal{N}}
\oplus \upalpha$ relatively to ${\mathscr{B}_2}(\ve{0},
R)$ as measured using ${\mathcal{D}}$. As depicted in Figure
\ref{f-model3}, $q_*$ is a minimal value of the probability to escape the neighborhoods
  of ${\mathcal{N}} \oplus \upalpha$ when sampling points according to
  ${\mathcal{D}}$ in ball ${\mathscr{B}_2}(\ve{0}, R)$. If, for some $\upalpha$ that shall depend upon the
  dimension $d$ and $\upkappa$, $q_*$ is large enough, then the spread
  of points drawn shall guarantee "small" values for $\updeltaw$ and $\updeltas$.

This is formalized in the following Theorem, which assumes $\upepsilon_m = \upepsilon_M = 1$, \textit{i.e.}
the ball has uniform density. Theorem  \ref{thlabf} is a direct
consequence of this Theorem.

\begin{theorem}\label{thb1}
Suppose ${\mathcal{A}} \subset {\mathscr{B}_2}(\ve{0},
R)$. For any $\updelta \in
(0,1)$, if 
\begin{eqnarray}
m & \geq & 3 \left(\frac{\upkappa}{q_* \updelta^2}\right)^2\:\:,\label{binfm}
\end{eqnarray}
then there is probability $\geq 1 - \updelta$ over its sampling that ${\mathcal{A}}$ is
$\updeltaw$-spread and $\updeltas$-monotonic for the following values of
$\updeltaw,\updeltas$:
\begin{eqnarray}
\updeltaw & = & \frac{1}{q_* (1 -
\updelta) (m - \upkappa - 1) \upalpha^2}\:\:,\label{propw}\\
\updeltas & = & \frac{m}{m-\upkappa} \cdot \left(\frac{2}{\min\left\{\frac{1}{4}, q_* (1 -
\updelta)\right\} \cdot \upalpha}\right)^2 - 1\:\:.\label{props}
\end{eqnarray}
\end{theorem}
\begin{proof}
We first prove the following Lemma.
\begin{lemma}\label{lemD1}
Suppose ${\mathcal{A}} \subset {\mathscr{B}_2}(\ve{0},
R)$. Let $q_*$ be defined as in eq. (\ref{minsup}). Then for any $\updelta \in
(0,1)$, if $m$ meets ineq. (\ref{binfm}), then there is probability $\geq 1 - \updelta$ that 
\begin{eqnarray}
\left| ({\mathscr{B}_2}(\ve{0}, R) \backslash {\mathcal{N}}
\oplus \upalpha) \cap ({\mathcal{A}}\backslash {\mathcal{N}})
\right| & \geq & q_* (1 -
\updelta) (m - \upkappa)\:\:, \forall {\mathcal{N}} \subseteq {\mathcal{A}},
    |{\mathcal{N}}| \in [ \upkappa]_*\:\:.
\end{eqnarray}
\end{lemma}
\begin{proof}
Since we assume ${\mathcal{A}}\subset {\mathscr{B}_2}(\ve{0},
R)$, Chernoff bounds imply that for any \textit{fixed}  ${\mathcal{N}} \subseteq
{\mathcal{A}}$ with $|{\mathcal{N}}|\in [ \upkappa]_*$,
\begin{eqnarray}
\pr_{\mathcal{D}} \left[\frac{\left| ({\mathscr{B}_2}(\ve{0}, R) \backslash {\mathcal{N}}
\oplus \upalpha) \cap ({\mathcal{A}}\backslash {\mathcal{N}})
\right|}{\left|{\mathcal{A}}\backslash {\mathcal{N}}\right|} \leq q_* (1 - \updelta) \right] & \leq &
\exp\left( - \updelta^2 q_* \left|{\mathcal{A}}\backslash {\mathcal{N}}\right|/ 2 \right)\:\:.
\end{eqnarray}
Now, remark that
\begin{eqnarray}
\sum_{j=1}^{\upkappa} {m\choose j} & \leq & m^\upkappa\:\:, \forall
m, \upkappa\geq 1\:\:.
\end{eqnarray}
This can be proven by induction, $m$ being fixed: it trivially holds
for $\upkappa = 1$ and $\upkappa = 2$, and furthermore 
\begin{eqnarray}
\sum_{j=1}^{\upkappa} {m\choose j} & = & \sum_{j=1}^{\upkappa-1}
{m\choose j} + {m\choose \upkappa}\nonumber\\
 & \leq & m^{\upkappa-1} + \frac{m!}{(m-\upkappa)!\upkappa!}\label{lll1}\:\:,
\end{eqnarray}
by induction at rank $\upkappa -1$. To prove that the right-hand side of (\ref{lll1}) is no more than
$m^\upkappa$, we just have to remark that
\begin{eqnarray}
\frac{m!}{(m-\upkappa)!\upkappa! m^{\upkappa-1}} & < &
\frac{m}{\upkappa!} \nonumber\\
 & \leq & m - 1\:\:,
\end{eqnarray}
as long as $\upkappa > 1$ and $m > 1$. So, the property at rank
$\upkappa - 1$ for $\upkappa > 1$ implies property at rank $\upkappa$,
which concludes the induction.

So, we have at most $m^\upkappa$ choices for ${\mathcal{N}}$, so relaxing the choice of ${\mathcal{N}}$, we get
\begin{eqnarray}
\lefteqn{\pr_{\mathcal{D}} \left[\exists {\mathcal{N}} \subseteq
{\mathcal{A}}, |{\mathcal{N}}| = \upkappa : \frac{\left| ({\mathscr{B}_2}(\ve{0}, R) \backslash {\mathcal{N}}
\oplus \upalpha) \cap {\mathcal{A}}_{\mathcal{N}}
\right|}{\left|{\mathcal{A}}_{\mathcal{N}}\right|} \leq q_* (1 -
\updelta) \right]}\nonumber\\
 & \leq &
m^\upkappa \exp\left( - \frac{\updelta^2 q_* (m-\upkappa)}{2} \right)\:\:.
\end{eqnarray}
We want to compute the minimal $m$ such that the right-hand side is no
more than $\updelta$, this being equivalent to 
\begin{eqnarray}
\updelta^2 q_* m 
& \geq & 2\log\left(\frac{m^{\upkappa}}{\updelta}\right) + \upkappa\updelta^2 q_* \nonumber\:\:,
\end{eqnarray}
which, since $\updelta \in (0,1)$, is ensured if
\begin{eqnarray}
\updelta^2 q_* m 
& \geq & 2\upkappa \log\left(\frac{m}{\updelta}\right) + \upkappa\updelta^2 q_* \:\:.\label{sufeq}
\end{eqnarray}
Suppose 
\begin{eqnarray}
m & = & 3 \left(\frac{\upkappa}{q_* \updelta^2}\right)^2\:\:.\nonumber
\end{eqnarray}
Since we trivially have $\upkappa^2 / (q_* \updelta^2)^2 \geq
\upkappa\updelta^2 q_* $ ($\upkappa \geq 1, q_* \in (0,1), \updelta
\in (0,1)$), it is sufficient to prove:
\begin{eqnarray}
\frac{2\upkappa}{q_*\updelta^2} & \geq & 2\log 3 + 2 \log\left( 
  \frac{\upkappa^2}{q^2_* \updelta^5} \right) \:\:,
\end{eqnarray}
which, again observing that $\updelta \in (0,1)$, holds if we can prove
\begin{eqnarray}
\frac{\upkappa}{q_*\updelta^2} & \geq & \log 2 + \frac{3}{2}\cdot \log\left( 
  \frac{\upkappa}{q_* \updelta^2} \right) \:\:,
\end{eqnarray}
which is equivalent to showing $x \geq (3/2)\log x + \log2$ for $x\geq
1$, which indeed holds (end of the proof of Lemma \ref{lemD1}).
\end{proof}
The consequence of Lemma \ref{lemD1} is the following:  if ${\mathcal{A}} \subset {\mathscr{B}_2}(\ve{0},
R)$ and $m$
satisfies (\ref{binfm}), then for any
  ${\mathcal{N}} \subseteq {\mathcal{A}}$ with $| {\mathcal{N}}| = k-1$, and any ${\mathcal{B}} \subseteq
  {\mathcal{A}}$ with $| {\mathcal{B}}| = |{\mathcal{A}}| - 1$,
\begin{eqnarray}
\sum_{\ve{a} \in {\mathcal{B}}} \|\ve{a} - \nn_{\mathcal{N}}(\ve{a})\|_2^2 & \geq &
q_* (1 -
\updelta) (m - \upkappa - 1) \upalpha^2 \cdot R^2\:\:,\label{condf1}
\end{eqnarray}
and so from Definition \ref{defspread} ${\mathcal{A}}$ is $\updeltas$-spread for:
\begin{eqnarray}
\updeltaw & = & \frac{1}{q_* (1 -
\updelta) (m - \upkappa - 1) \upalpha^2}\:\:. \label{assum1P1}
\end{eqnarray}
Now, suppose we add a single point $\ve{x}_*$ in $\mathcal{N}$. If,
for some fixed $\upalpha_* \in (0,\upalpha/2]$,
\begin{eqnarray}
\ve{x}_* & \not\in & \ve{a} \oplus \upalpha_*\label{cond1}\:\:,\forall
\ve{a} \in {\mathcal{A}}\:\:,
\end{eqnarray}
then because of (\ref{condf1}), 
\begin{eqnarray}
\sum_{\ve{a} \in {\mathcal{A}}} \|\ve{a} -
\nn_{{\mathcal{N}}\cup \{\ve{x}_*\} }(\ve{a})\|_2^2 & \geq & (m - \upkappa)\cdot \min\left\{\upalpha_*^2, q_* (1 -
\updelta)\upalpha^2\right\}\cdot R^2\:\:.
\end{eqnarray}
Otherwise, consider one $\ve{a}_*$ for which $\ve{x}_*\in \ve{a}_*
\oplus \upalpha_*$. If we replace $\ve{a}_*$ by $\ve{x}_*$ in all
${\mathcal{N}}$ in which $\ve{a}_*$ belongs to in Lemma \ref{lemD1},
then because $\ve{x}_*\oplus \upalpha_* \subset
\ve{a}_*\oplus\upalpha$, it comes from Lemma \ref{lemD1}:
\begin{eqnarray}
\sum_{\ve{a} \in {\mathcal{A}}} \|\ve{a} -
\nn_{{\mathcal{N}}\cup \{\ve{x}_*\} }(\ve{a})\|_2^2 & \geq & \frac{1}{4}\cdot (m - \upkappa)\cdot q_* (1 -
\updelta)\upalpha^2\cdot R^2\:\:.
\end{eqnarray}
We thus get in all cases
\begin{eqnarray}
\sum_{\ve{a} \in {\mathcal{A}}} \|\ve{a} -
\nn_{\mathcal{N}\cup \{\ve{c}(A)\}}(\ve{a})\|_2^2 & \geq &
\min\left\{\frac{\upalpha^2}{4}, \upalpha_*^2, q_* (1 -
\updelta)\upalpha^2\right\} (m - \upkappa)\cdot q_* (1 -
\updelta) \cdot R^2\:\:,
\end{eqnarray}
where $\ve{c}(A)$ is the arithmetic average computed according to the
definition of $\updeltas$-monotonicity, of any $A\subseteq
{\mathcal{A}}\backslash {\mathcal{N}}$. Since $ {\mathcal{N}}\subseteq
{\mathcal{A}} \subset {\mathscr{B}_2}(\ve{0},
R)$, we have  $\sum_{\ve{a} \in {\mathcal{A}}} \|\ve{a} -
\nn_{{\mathcal{N}}}(\ve{a})\|_2^2 \leq 4mR^2$, and so
\begin{eqnarray}
\sum_{\ve{a} \in {\mathcal{A}}} \|\ve{a} -
\nn_{{\mathcal{N}}}(\ve{a})\|_2^2 & \leq & \frac{4m}{\min\left\{\frac{\upalpha^2}{4}, \upalpha_*^2, q_* (1 -
\updelta)\upalpha^2\right\} (m - \upkappa)\cdot q_* (1 -
\updelta)} \cdot \sum_{\ve{a} \in {\mathcal{A}}} \|\ve{a} -
\nn_{\mathcal{N}\cup \{\ve{c}(A)\}}(\ve{a})\|_2^2\:\:,
\end{eqnarray}
implying from Definition \ref{defmon} that 
$\updeltas$-monotonicity holds with:
\begin{eqnarray}
\updeltas & = & \frac{m}{m-\upkappa} \cdot \frac{4}{\min\left\{\frac{\upalpha^2}{4}, \upalpha_*^2, q_* (1 -
\updelta)\upalpha^2\right\} \cdot q_* (1 -
\updelta)} - 1\:\:.
\end{eqnarray}
The statement of the Theorem follows with $\upalpha_* = \upalpha/2$
(end of the proof of Theorem \ref{thb1}).
\end{proof}
We finish the proof of Theorem  \ref{thlabf}. We have 
\begin{eqnarray}
q_* & \geq & 1 - \upkappa \upalpha^d\:\:,
\end{eqnarray}
where the lowerbound corresponds to the case where all neighborhoods
in ${\mathcal{N}}\oplus \upalpha$
are distinct and included in ${\mathscr{B}_2}(\ve{0},
R)$. So we have, for any fixed choice of $\upalpha\in (0,1)$,
\begin{eqnarray}
\updeltaw & \leq & \frac{1}{\upalpha^2 \cdot (1 - \upkappa \upalpha^d)
 (1 -
\updelta) (m - \upkappa - 1)}\:\:.
\end{eqnarray}
To minimize this upperbound, we pick $\upalpha$ to maximize
$\upalpha^2 \cdot (1 - \upkappa \upalpha^d)$ with $\upalpha \in
(0,1)$, which is easily achieved picking
\begin{eqnarray}
\upalpha & = & \left(\frac{1}{\upkappa(d+1)}\right)^{\frac{1}{d}}\:\:,
\end{eqnarray}
and yields
\begin{eqnarray}
\updeltaw & \leq & \left(1 + \frac{1}{d}\right) \cdot \frac{1}{(\upkappa(d+1))^\frac{2}{d}
 (1 -
\updelta) (m - \upkappa - 1)}\nonumber\\
 & \leq & \left(1 + \frac{1}{d}\right) \cdot \frac{1}{\upkappa^\frac{2}{d}
 (1 -
\updelta) (m - \upkappa - 1)}\:\:.\label{eq1deltaw}
\end{eqnarray}
But we have for this choice, $1 - \upkappa \upalpha^d = d/(d+1) \geq 1/2$, so as long
as 
\begin{eqnarray}
\updelta < 1/2\:\:,\label{conddelta}
\end{eqnarray} 
we shall have $q_* (1-\updelta) > 1/4$ and so we shall have
\begin{eqnarray}
\updeltas + 1 & = & 64\cdot \frac{m}{m-\upkappa}
\cdot\frac{1}{\upalpha^2} \nonumber\\
 & \leq & 64\cdot \frac{m}{m-\upkappa} \cdot \frac{1}{\upkappa^\frac{2}{d}}  \:\:. \label{eq1deltas}
\end{eqnarray}
We now go back to ineq. (\ref{defppp}), which reads:
\begin{eqnarray}
\frac{\pr[{\mathcal{C}} | {\mathcal{A}}']}{\pr[{\mathcal{C}} |
  {\mathcal{A}}]}  \leq \varrho_1 + \varrho_2\:\:,
\end{eqnarray}
with
\begin{eqnarray}
\varrho_1 & \defeq & (1+\updeltaw)^{k-1}\:\:,\label{defAA}\\
\varrho_2 & \defeq & f(k)\cdot \updeltaw  \cdot \left(1+\updeltas\right)^{k-1}  \cdot
 \varrho(R) \:\:.\label{defBB}
\end{eqnarray}
We upperbound separately both terms.
\begin{lemma}\label{lemrho1}
Suppose ineqs (\ref{conddelta}) and (\ref{condk}) are
met. Then 
\begin{eqnarray}
\varrho_1 & \leq & 1 + \frac{4}{m^{\frac{1}{4}+\frac{1}{d+1}}}\:\:.\label{condrho1}
\end{eqnarray}
\end{lemma}
\begin{proof}
Since $d\geq 1$ and $\updelta<1/2$, we get from
ineq. (\ref{eq1deltaw}) (using $\upkappa=k$)
\begin{eqnarray}
 (1+\updeltaw)^{k-1} & \leq & \left( 1 + \left(1 + \frac{1}{d}\right) \cdot \frac{1}{k^\frac{2}{d}
 (1 -
\updelta) (m - k - 1)}\right)^{k-1}\nonumber\\
 & \leq & \left( 1 + \frac{2}{k^\frac{2}{d}
 (1 -
\updelta) (m - k - 1)}\right)^{k-1}\nonumber\\
& \leq & \left( 1 + \frac{4}{k^\frac{2}{d} (m - k -
    1)}\right)^{k-1}\label{geq11}\:\:.
\end{eqnarray}
Let $h(k)$ be the right-hand side of ineq. (\ref{geq11}). $h(1)$
trivially meets ineq. (\ref{condrho1}). When $k\geq 2$, $h$ decreases
until $k=2(m-1)/(d+2)$ and then increases. We thus just need to check
ineq. (\ref{condrho1}) for $k=2$ and $k=\sqrt{m}$ from ineq. (\ref{condk}). We get $h(2) = 1 + 4/(4^{1/d}(m-3))$. For
ineq. (\ref{condrho1}) to be satisfied, we need to have
$4^{1/d}(m-3)\geq m^{\frac{1}{4}+\frac{1}{d+1}}$, which holds if $m
\geq 3 + m^{3/4}$ ($d\geq 1$), that is, $m\geq 8$. But since ineqs (\ref{conddelta}) and (\ref{condk}) are
satisfied, we have $m\geq 16 k^2 / \updelta^2 \geq 64 k^2 \geq
64$, and so $h(2)$ satisfies ineq. (\ref{condrho1}).

There remains to check ineq. (\ref{condrho1}) for $k=\sqrt{m}$. We
have
\begin{eqnarray}
h(\sqrt{m}) & = &  \left( 1 + \frac{4}{m^\frac{1}{d} (m - \sqrt{m} -
    1)}\right)^{\sqrt{m}-1}\nonumber\\
 & \leq & \left( 1 + \frac{4}{m^\frac{1}{d} (m - \sqrt{m} )}\right)^{\sqrt{m}}\nonumber\\
 & \leq & \left(1+\frac{2}{\sqrt{m}\cdot
     m^{\frac{1}{4}+\frac{1}{d}}}\right)^{\sqrt{m}}\:\:,\label{hh1}
\end{eqnarray}
since any $m\geq 64$, we have $m-\sqrt{m}\geq 2m^{3/4}$. To conclude,
ineq (\ref{hh1}) yields
\begin{eqnarray}
h(\sqrt{m})  & \leq & \left(1+\frac{2}{\sqrt{m}\cdot
     m^{\frac{1}{4}+\frac{1}{d}}}\right)^{\sqrt{m}}\nonumber\\
 & \leq & \exp\left( \frac{2}{m^{\frac{1}{4}+\frac{1}{d}}} \right) \nonumber\\
 & \leq & 1+ \frac{4}{m^{\frac{1}{4}+\frac{1}{d}}}\:\:.
\end{eqnarray}
The penultimate ineq. comes from $1+x\leq \exp x$, and the last one
comes from the fact that $\exp(2x) \leq 1+ 4x$ for $x\leq 1$. Since
$m^{\frac{1}{4}+\frac{1}{d}} \geq m^{\frac{1}{4}+\frac{1}{d+1}}$, we
obtain the statement of the Lemma for $h(\sqrt{m})$. This concludes
the proof of Lemma \ref{lemrho1}.
\end{proof}

\begin{lemma}\label{lemrho2}
Suppose ineqs (\ref{conddelta}) and (\ref{condk}) are
met. Then
\begin{eqnarray}
\varrho_2 & \leq & \left(\frac{64}{
  k^\frac{2}{d}}\right)^k 
\cdot \frac{\varrho(2R)}{m}\:\:.
\end{eqnarray}
\end{lemma}
\begin{proof}
We fix $\upkappa = k$, use $f(k) = 4^{2k-4}$ (eq. \ref{defff}), so we get
\begin{eqnarray}
\varrho_2  & = &
4^{2k-2}\cdot \left(1 + \frac{1}{d}\right) \cdot \frac{1}{k^\frac{2}{d}
 (1 -
\updelta) (m - k - 1)} \cdot \left(64\cdot \frac{m}{m-k} \cdot \frac{1}{k^\frac{2}{d}}\right)^{k-1}  \cdot
 \varrho(2R) \nonumber\\
 & \leq &
2 \cdot 64^{k-1}\cdot \left(1 + \frac{1}{d}\right) \cdot \frac{1}{k^\frac{2k}{d}
(m - k - 1)} \cdot \left(1+ \frac{k}{m-k}\right)^{k-1}  \cdot
 \varrho(2R)  \label{defppp3}\\
 & \leq & \underbrace{4
\cdot \frac{1}{
(m - k - 1)} \cdot \left(1+ \frac{k}{m-k}\right)^{k-1}}_{\defeq \varrho_3}  \cdot 64^{k-1}\cdot \frac{1}{k^\frac{2k}{d}}\cdot
 \varrho(2R)  \:\:,
\end{eqnarray}
using the fact
that $\updelta < 1/2$ and $d\geq 1$. 
Now, we also have
\begin{eqnarray}
\left(1+ \frac{k}{m-k}\right)^{k-1}  & \leq &
\exp\left(\frac{k^2}{m-k}\right)\noindent\\
 & \leq & e\:\:,
\end{eqnarray}
as long as $k\leq (1/16)\cdot\sqrt{m}$, and furthermore, since $m\geq
64$ (see the proof of Lemma \ref{lemrho1}), we also have
$1/(m-k-1)\leq 5/m$. We thus obtain
\begin{eqnarray}
\varrho_3 & \leq & \frac{20e}{m} \nonumber\\
 & \leq & \frac{64}{m}\:\:,
\end{eqnarray}
which yields
\begin{eqnarray}
\varrho_2 & \leq &  \left(\frac{64}{
  k^\frac{2}{d}}\right)^k 
\cdot \frac{\varrho(2R)}{m}\:\:,
\end{eqnarray}
as claimed.
\end{proof}
Putting altogether Lemmata \ref{lemrho1} and \ref{lemrho2}, we get:
\begin{eqnarray}
\frac{\pr[{\mathcal{C}} | {\mathcal{A}}']}{\pr[{\mathcal{C}} |
  {\mathcal{A}}]}  & \leq & 1 + \frac{4}{m^{\frac{1}{4}+\frac{1}{d+1}}} + \left(\frac{64}{
  k^\frac{2}{d}}\right)^k 
\cdot \frac{\varrho(2R)}{m} \:\:,\label{eqass1}
\end{eqnarray}
as claimed. 
There remains to check that, with our choice of
$\upalpha$, the constraint on $m$ in (\ref{binfm}) is satisfied if
\begin{eqnarray}
m & \geq & \frac{12  k^2}{\updelta^4}
\end{eqnarray}
since $q_* \geq d/(d+1)$. We obtain the sufficient constraint on $k$:
\begin{eqnarray}
k & \leq & \frac{\updelta^2}{4} \cdot \sqrt{m}\:\:,
\end{eqnarray}
which proves Theorem \ref{thlabf} when $\upepsilon_m = \upepsilon_M = 1$.\\

When the density do not satisfy $\upepsilon_m =
\upepsilon_M = 1$ we just have to remark that the lowerbound on $q_*$
is now
\begin{eqnarray}
q_* & \leq & \frac{\varepsilon_m}{\varepsilon_M} \cdot (1 - \upkappa
\upalpha^d)\:\:.
\end{eqnarray}
Ineq. (\ref{eq1deltaw}) becomes
\begin{eqnarray}
\updeltaw & \leq & \frac{\varepsilon_M}{\varepsilon_m}\cdot \left(1 + \frac{1}{d}\right) \cdot \frac{1}{\upkappa^\frac{2}{d}
 (1 -
\updelta) (m - \upkappa - 1)}\:\:,\label{eq1deltaw2}
\end{eqnarray}
ineq. (\ref{eq1deltas}) becomes
\begin{eqnarray}
\updeltas + 1 & \leq & \frac{\varepsilon_M}{\varepsilon_m}\cdot  64\cdot \frac{m}{m-\upkappa} \cdot \frac{1}{\upkappa^\frac{2}{d}}  \:\:. \label{eq1deltas2}
\end{eqnarray}
So, the only difference with the $\upepsilon_m =
\upepsilon_M = 1$ is the ratio $\varepsilon_M / \varepsilon_m$ ($\geq
1$) which multiplies all quantities of interest, and yields, in lieu
of ineq. (\ref{eqass1}),
\begin{eqnarray}
\frac{\pr[{\mathcal{C}} | {\mathcal{A}}']}{\pr[{\mathcal{C}} |
  {\mathcal{A}}]}  & \leq & 1 + \left( \frac{\varepsilon_M}{\varepsilon_m}
\right)^k \cdot \left(\frac{4}{m^{\frac{1}{4}+\frac{1}{d+1}}} + \left(\frac{64}{
  k^\frac{2}{d}}\right)^k 
\cdot \frac{\varrho(2R)}{m}\right) \:\:,\label{eqass12}
\end{eqnarray}
which is the statement of Theorem \ref{thlabf}. 
\section*{Proof of Theorem \ref{thdp1}}\label{proof_thdp1}

When $p_{(\ve{\mu}_{\ve{a}}, \ve{\theta}_{\ve{a}})}$ is a product of
Laplace distributions $Lap(b)$ ($b$ being the \textit{scale} parameter
of the distribution \citep{drTA}), condition in ineq. (\ref{pwc}) becomes:
\begin{eqnarray}
\frac{p_{(\ve{\mu}_{\ve{a}'}, \ve{\theta}_{\ve{a}'})}( \ve{x})}{p_{(\ve{\mu}_{\ve{a}}, \ve{\theta}_{\ve{a}})}( \ve{x})} & \leq &
\exp\left(\frac{\|\ve{a}-\ve{a}'\|_1}{b}\right) \nonumber\\
 & & = \exp\left(\frac{\sqrt{2}\|\ve{a}-\ve{a}'\|_1}{\sigma_1}\right) \nonumber\\
 & \leq & \exp\left(\frac{2 \sqrt{2} R}{\sigma_1}\right)\:\:,\forall {\ve{a}}, {\ve{a}'} \in {\mathcal{A}}, \forall
\ve{x}\in \Omega\:\:,\label{pwc2}
\end{eqnarray}
assuming ${\mathcal{A}}\subset {\mathscr{B}}_1(\ve{0}, R)$. Let us fix
$\varrho(R) \defeq \exp\left(2\sqrt{2} R / \sigma_1\right)$. Since
${\mathscr{B}}_1(\ve{0}, R)\subset {\mathscr{B}}_2(\ve{0}, R)$
(the $L_2$ ball), we now want
$(1+\updeltaw)^{k-1} + f(k)\cdot \updeltaw  \cdot \left(1+\updeltas\right)^{k-1}  \cdot
 \varrho(R) = \exp(\epsilon)$. Solving for $\sigma_1$ yields:
\begin{eqnarray}
\sigma_1 & = & \frac{2 \sqrt{2} R}{ \log\left(\frac{\exp(\epsilon) - (1+\updeltaw)^{k-1}}{f(k)\cdot \updeltaw  \cdot \left(1+\updeltas\right)^{k-1}}\right)}\:\:,
\end{eqnarray}
as claimed. The proof that $k$-\means~meets
ineq. (\ref{boundsup}) with
\begin{eqnarray}
\Phi  & = & \Phi_1 \defeq 8 \cdot\left(\phi_{\opt} + \frac{mR^2}{\tilde{\epsilon}^2}\right)\label{boundsupDP1SI}
\end{eqnarray}
comes from a direct application of Theorem \ref{thmain},
with
\begin{eqnarray}
\upeta & = & 0 \:\:,\nonumber\\
\phi_{\bias} & = & \phi_{\opt}\:\:,\nonumber\\
\phi_{\variance} & = & m \cdot \left(\frac{2\sqrt{2}R}{\tilde{\epsilon}}\right)^2\:\:.\nonumber
\end{eqnarray}
The statements for $\sigma_2$ and $\Phi_2$ are direct applications of
the Laplace mechanism properties \citep{drTA,dmnsCN}.

\subsection*{Extension to non-metric spaces}\label{proof_ext}

Since its inception, the $k$-means++ seeding technique has been
successfully adapted to various distortion measures $D(\cdot\|\cdot)$ to
handle non-Euclidean features \citep{jsbAA,nlkMB,nnaOCD}.
Similarly, our extended seeding technique can be adapted to these
scenarii: this boils down to putting the distortion as a free
parameter of the algorithm, replacing $D_t(\ve{a})$ (eq. (\ref{choosea})) by
$D_t(\ve{a}) \defeq \min_{\ve{a}' \in {\mathcal{P}}} D(\ve{a}\|
\ve{a}')$. For example, by noticing that the squared Euclidean distance is merely an
example of Bregman divergences (the well-known canonical divergences in
information geometry of dually flat spaces),
$k$-\means~can be been extended to that family of
dissimilarities~\citep{nlkMB}. But more interesting examples now
appear, that build on constraints that distortions have to satisfy for
certain problems, like the invariance to rotations of the coordinate
space. 
This is all the more challenging in practice  for clustering since sometimes \textbf{no-closed form} solution are available for some of these
divergences. Because it bypasses the construction of the population
minimisers, $k$-\means~offers an elegant solution to the problem. Such
hard distortions include the skew Jeffreys
$\alpha$-centroids \citep{nnaOCD}. This also include the recent class of
total Bregman/Jensen divergences that are examples of conformal
divergences~\citep{nnTJ,nnaOCD}. We give an example of the extension of
$k$-\means to the total Jensen divergence, to show that $k$-\means~can
approximate the optimal clustering even without closed form solutions
for the population minimisers \citep{nnTJ}. For any convex function
$\varphi : {\mathbb{R}}^d \rightarrow {\mathbb{R}}$ and $\alpha \in (0,1)$, the skew Jensen
divergence is 
\begin{eqnarray}
J_\alpha(\ve{a}, \ve{a}') & \defeq & \alpha \varphi(\ve{a}) +
(1-\alpha) \varphi(\ve{a}') - \varphi(\alpha \ve{a} +
(1-\alpha) \ve{a}')\:\:,
\end{eqnarray} 
and the total Jensen divergence is
\begin{eqnarray}
tJ_\alpha(\ve{a}, \ve{a}') & \defeq & \frac{1}{\sqrt{1+U^2}} \cdot
J_\alpha(\ve{a}, \ve{a}')\:\:,
\end{eqnarray} 
where $U \defeq
(\varphi(\ve{a})-\varphi(\ve{a}')) / \|\ve{a} - \ve{a}'\|_2$. There is
no closed form solution for the population minimiser of $tJ_\alpha$,
yet we can prove the following Theorem, which builds upon Theorem 3 in
\citep{nnTJ}.
\begin{theorem}
In $k$-\means, replace $D_t(\ve{a})$ (eq. (\ref{choosea})) by
$D_t(\ve{a}) \defeq \min_{\ve{a}' \in {\mathcal{P}}} tJ_\alpha(\ve{a},
\ve{a}')$ ans suppose for simplicity that probe functions are
identity: $\probe_t = \mathrm{Id}, \forall t$. Denote
$\phi_{\opt}$ the optimal noise-free potential of the clustering
problem using $tJ_\alpha$ as distortion measure. Then there exists a
\textbf{constant} $\upomega > 0$ such that
for \textbf{any} choice of densities $p_{\ve{\mu}_., \ve{\theta}_.}$, the expected $tJ_\alpha$-potential $\phi$ of $k$-\means~satisfies:
\begin{eqnarray}
\expect[\phi({\mathcal{A}} ; {\mathcal{C}})] & \leq & \upomega \cdot \log
k\cdot \left(6\phi_{\opt} + 2\phi_{\bias} +
2\phi_{\variance}\right)\:\:,\label{boundsuptJ}
\end{eqnarray}
where $\phi_{\variance}$ is defined in Theorem \ref{thmain} and $\phi_{\bias}$
is defined in eq. (\ref{defopt2}).
\end{theorem}

\newpage
\section{Appendix on Experiments}\label{exp_expes}

\subsection*{Experiments on Theorem \ref{thdp1} and the sublinear noise regime}\label{exp_thdp1}

\begin{figure}[t]
    \centering
\includegraphics[width=0.5\textwidth]{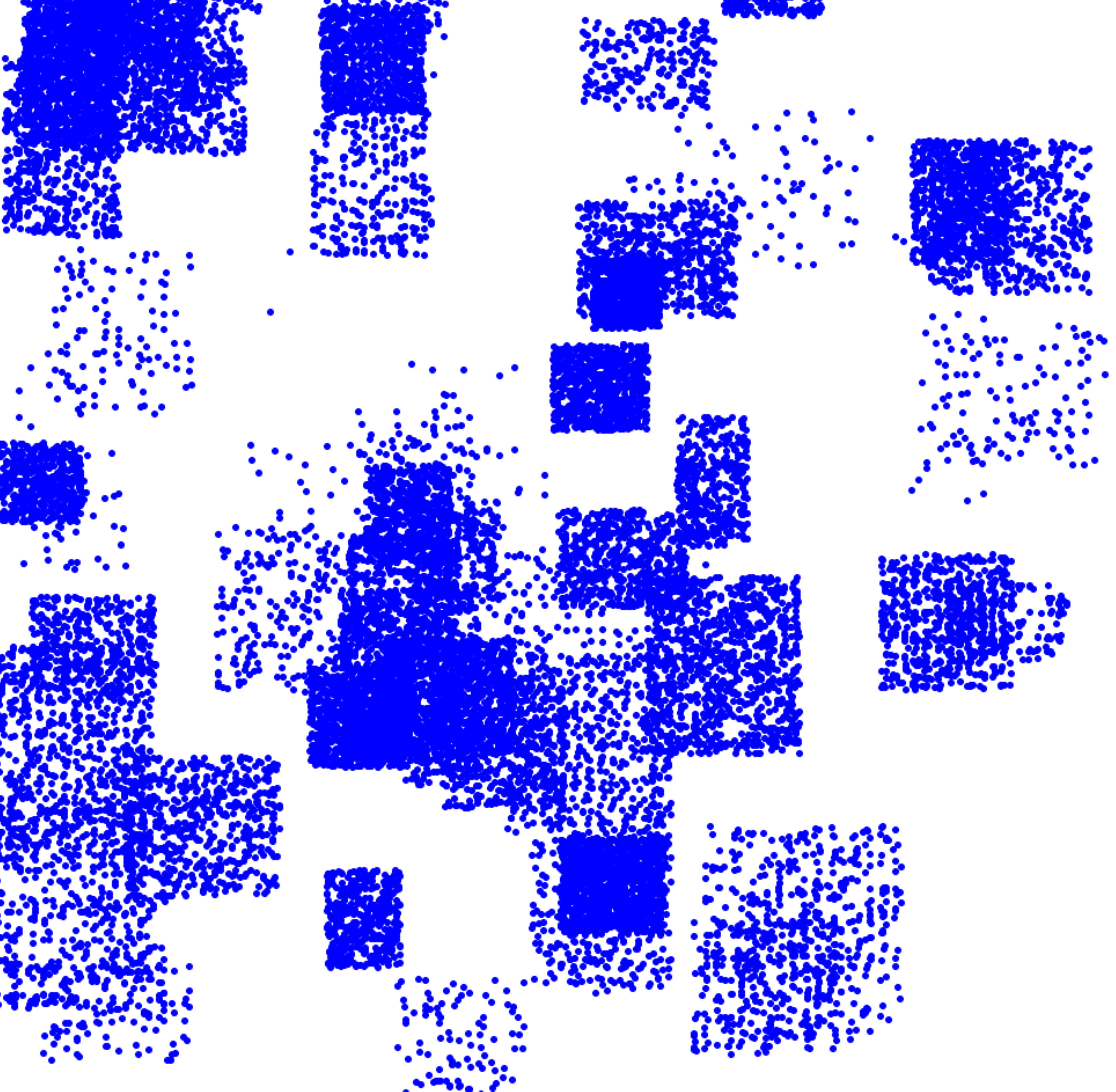}
\caption{Final dataset for the experiments in Table \ref{t-exp1} (plot of the two first coordinates ($d=10$)).}\label{f-exp1}
\end{figure}

\begin{table}[t]
    \centering
\begin{center}
\begin{tabular}{ccc}
\hline 
\includegraphics[width=0.3\textwidth]{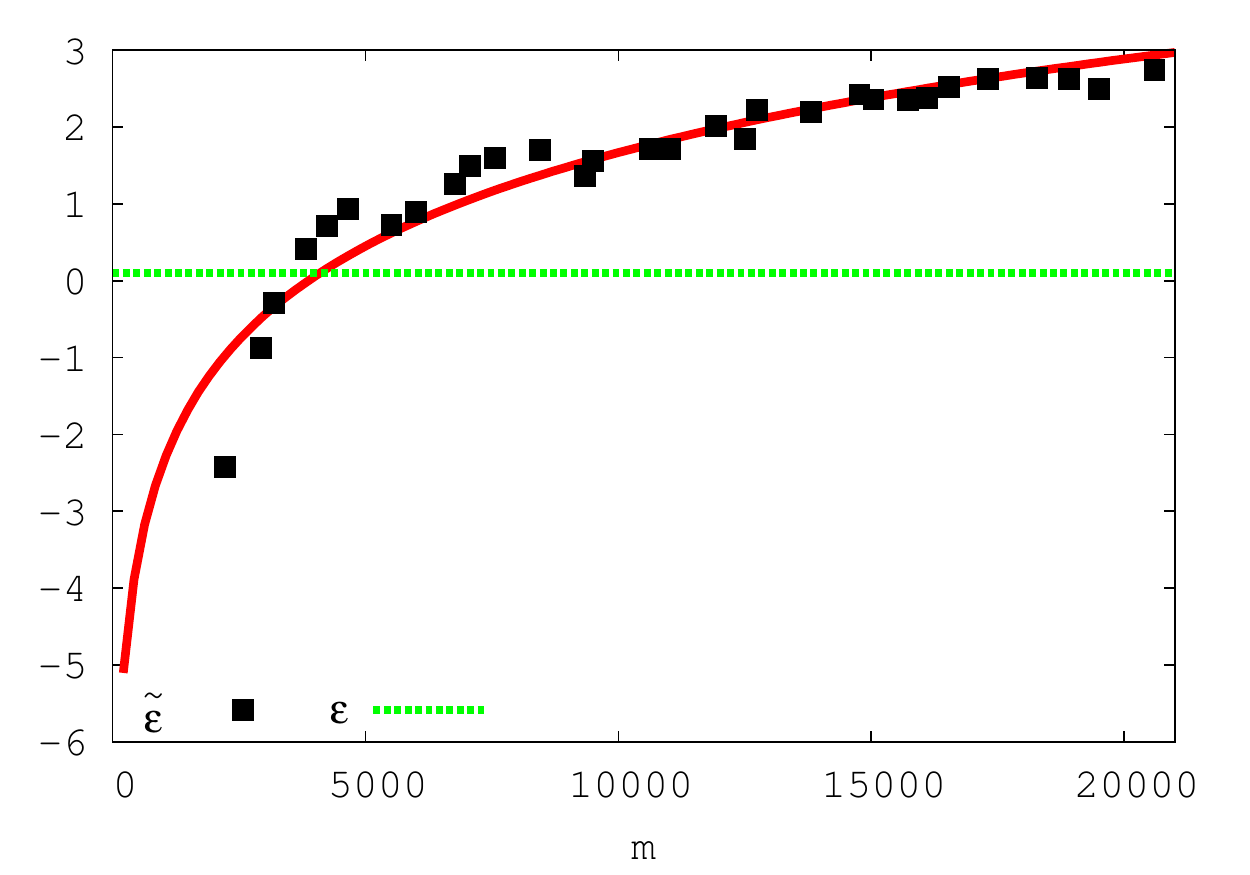}
& \includegraphics[width=0.3\textwidth]{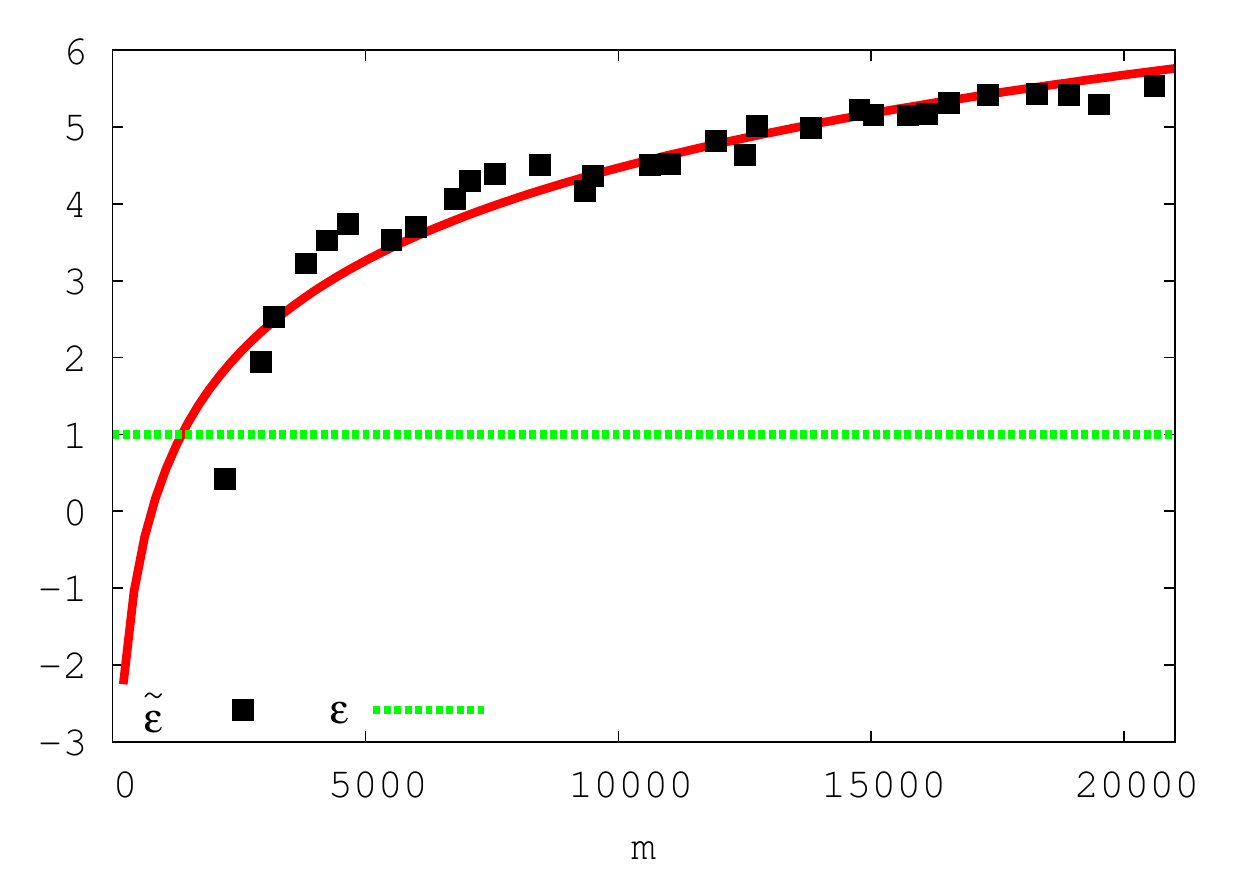}
& \includegraphics[width=0.3\textwidth]{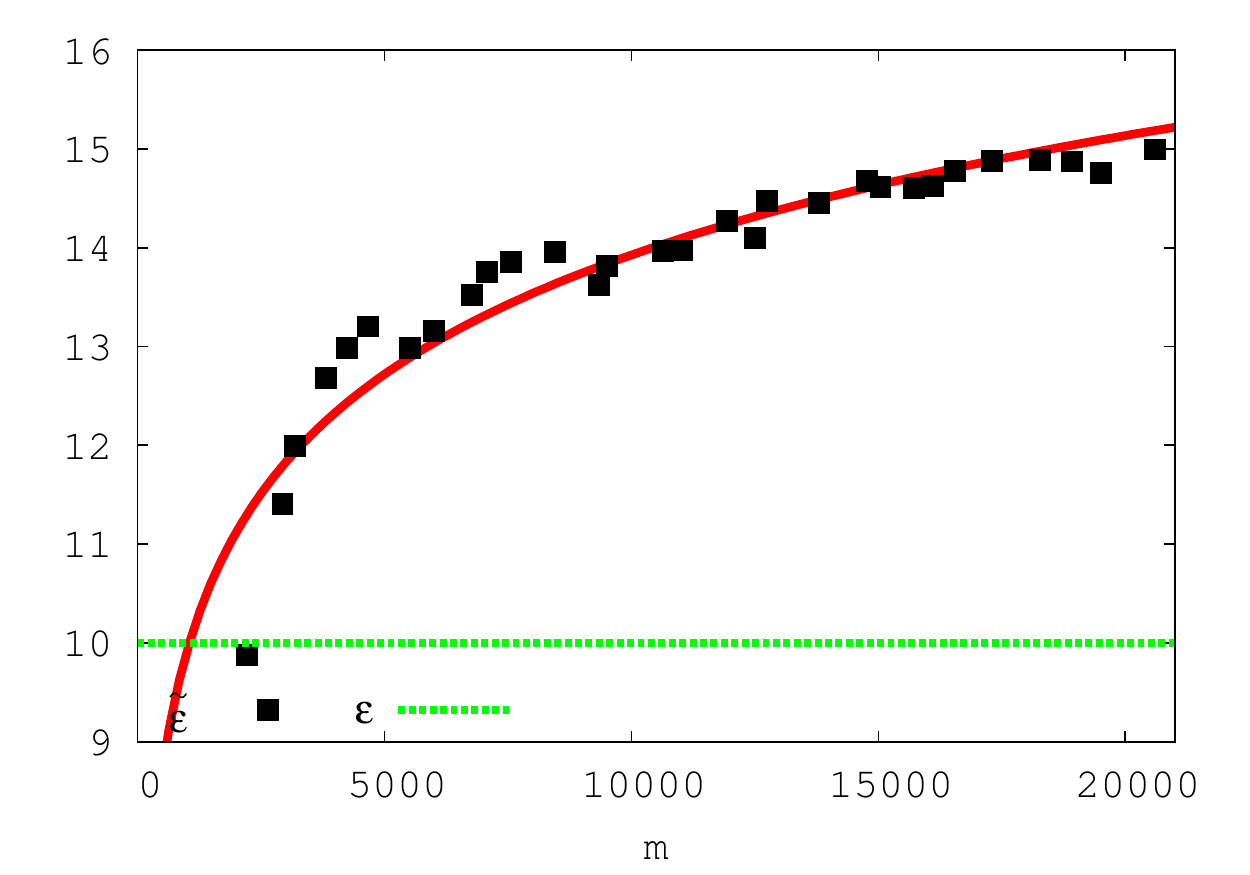}\\ 
$\epsilon = 0.1$ & $\epsilon = 1$ & $\epsilon = 10$ \\
\hline 
\end{tabular}
\caption{Case $d=10, k=3$ --- Plot of $\tilde{\epsilon}$ as in Theorem \ref{thdp1} (see also eq. (\ref{defepsilont-si}) below) and best fit for model $\tilde{\epsilon} = a + b\log
  m$. Figure \ref{f-exp1} displays the final dataset obtained (see text).}\label{t-exp1}
\end{center}
\end{table}

\begin{table}[t]
    \centering
\begin{center}
\begin{tabular}{ccc}
\hline 
\includegraphics[width=0.3\textwidth]{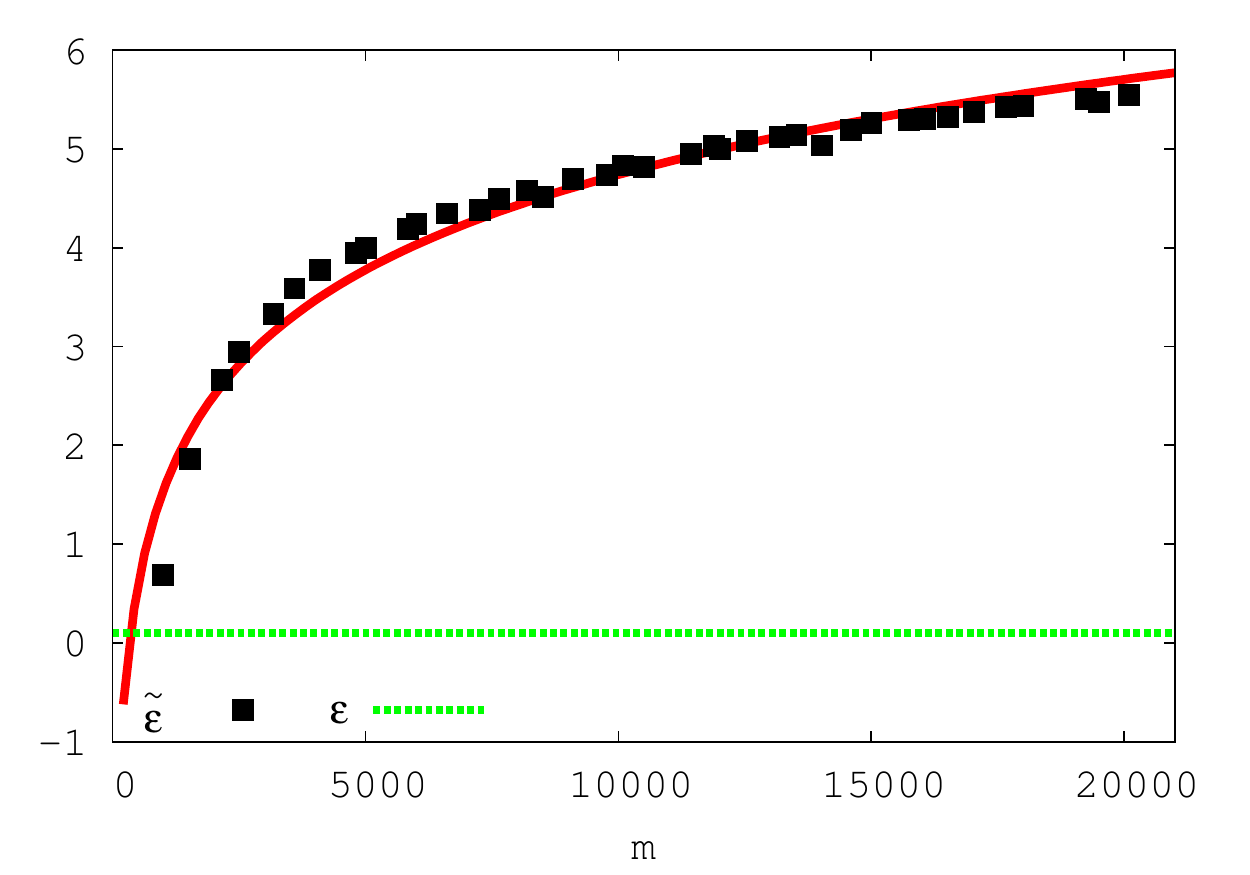}
& \includegraphics[width=0.3\textwidth]{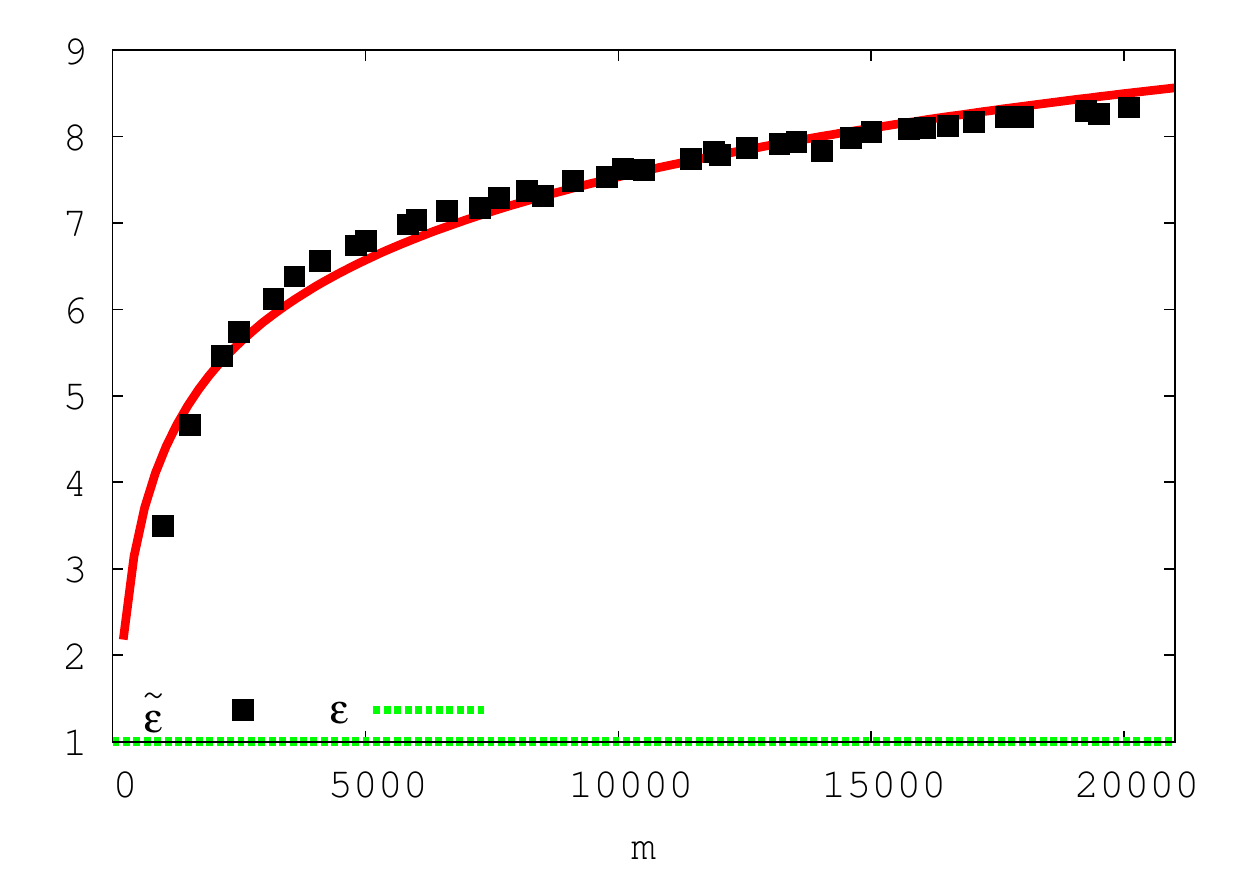}
& \includegraphics[width=0.3\textwidth]{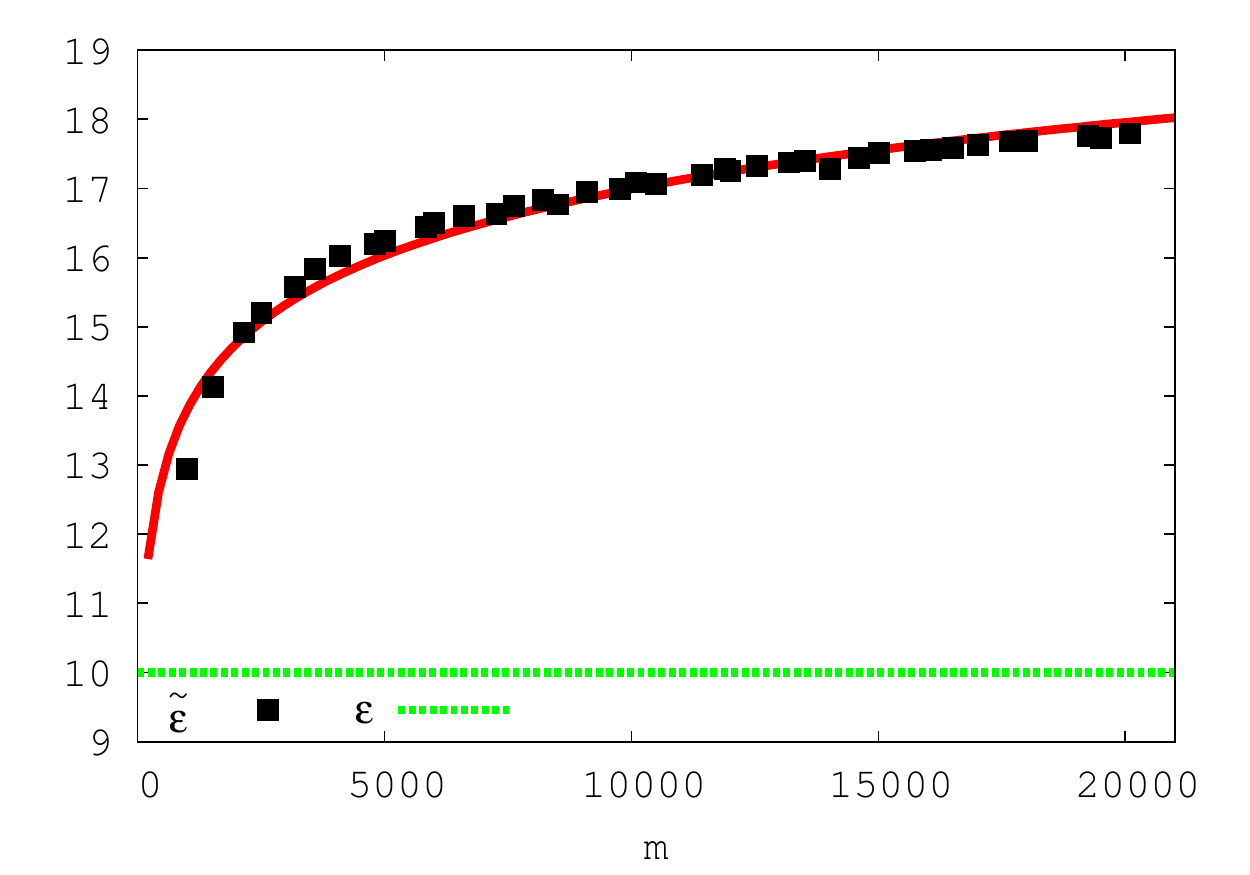}\\ 
$\epsilon = 0.1$ & $\epsilon = 1$ & $\epsilon = 10$ \\
\hline 
\end{tabular}
\caption{Case $d=50, k=3$ --- Plot of $\tilde{\epsilon}$ as in Theorem \ref{thdp1} (see also eq. (\ref{defepsilont-si}) below) and best fit for model $\tilde{\epsilon} = a + b\log
  m$. All other parameters are the same as for Table \ref{t-exp1}.}\label{t-exp2}
\end{center}
\end{table}

\begin{table}[t]
    \centering
\begin{center}
\begin{tabular}{ccc}
\hline 
\includegraphics[width=0.3\textwidth]{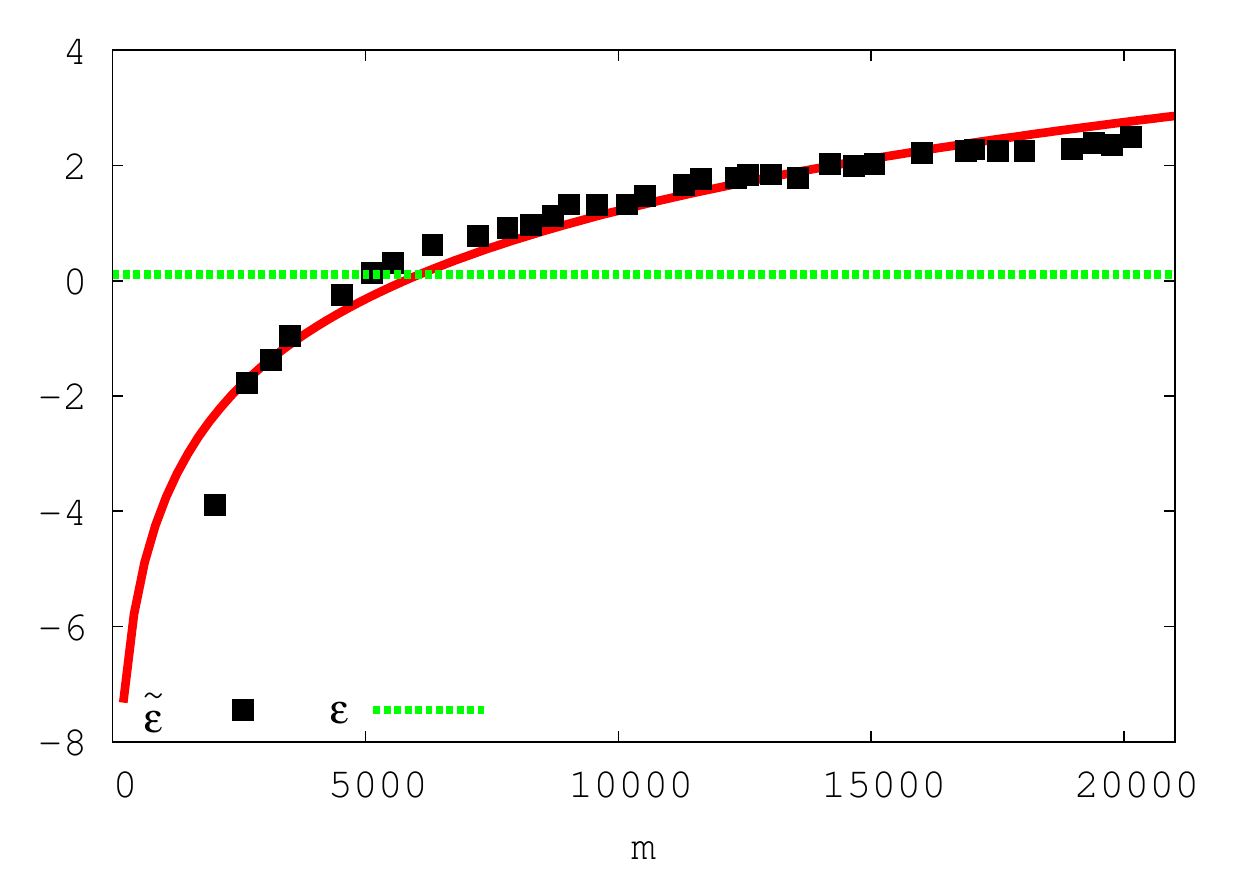}
& \includegraphics[width=0.3\textwidth]{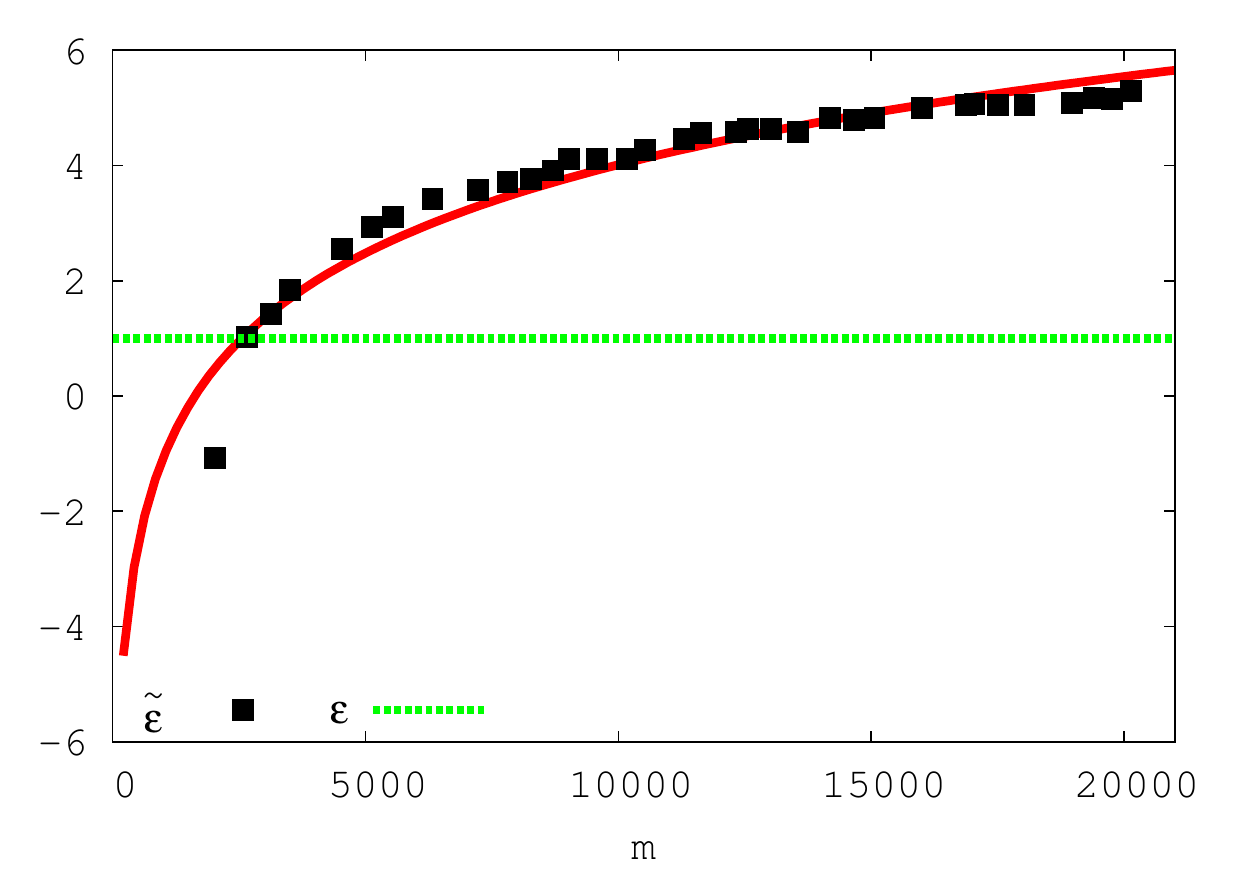}
& \includegraphics[width=0.3\textwidth]{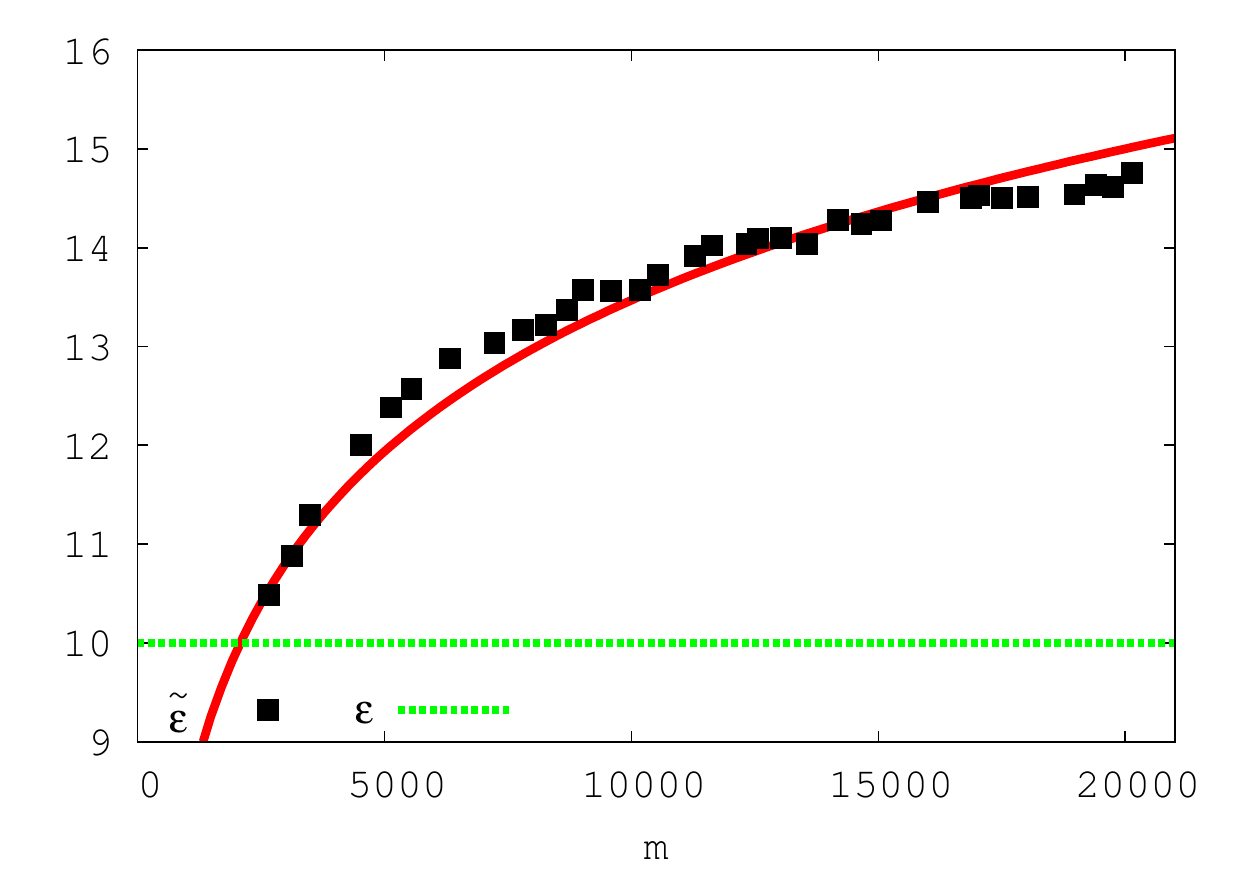}\\ 
$\epsilon = 0.1$ & $\epsilon = 1$ & $\epsilon = 10$ \\
\hline 
\end{tabular}
\caption{Case $d=50, k=4$ --- Plot of $\tilde{\epsilon}$ as in Theorem \ref{thdp1} (see also eq. (\ref{defepsilont-si}) below) and best fit for model $\tilde{\epsilon} = a + b\log
  m$. All other parameters are the same as for Table \ref{t-exp1}.}\label{t-exp3}
\end{center}
\end{table}

\paragraph{$\hookrightarrow$ comments on $\tilde{\epsilon}$} An important parameter of Theorem \ref{thdp1} is $\tilde{\epsilon}$,
which replaces $\epsilon$ in the computation of the noise standard
deviation in $\sigma_1$: the larger it is compared to $\epsilon$, the
less noise we can put while still ensuring $\pr[{\mathcal{C}} | {\mathcal{A}}'] /
\pr[{\mathcal{C}} | {\mathcal{A}}] \leq \exp
\epsilon$ in Definition \ref{defEDPF}. Recall its formula:
\begin{eqnarray}
\tilde{\epsilon} & \defeq & \log\left(\frac{\exp(\epsilon) - (1+\updeltaw)^{k-1}}{f(k)\cdot \updeltaw  \cdot \left(1+\updeltas\right)^{k-1}}\right)\:\:.\label{defepsilont-si}
\end{eqnarray}
The experimental setting is the following one: we repeatedly sample
clusters that are
uniform in a subset of the domain (with limited, random size), taken to be a $d$-dimensional
hyperrectangle of randomly chosen edge lengths. Each cluster contains a randomly picked number
of points between 1 and 1000. After each cluster is picked, we updated
an \textit{estimation} of $\updeltaw$ and $\updeltas$:
\begin{itemize}
\item we compute $\updeltaw$ by randomly picking ${\mathcal{B}}$ and
  ${\mathcal{N}}$ for a total number of $n_{\est}$ iterations, with
  $n_{\est} = 5000$;
\item we compute $\updeltas$ by randomly picking ${\mathcal{N}}$ for a
  total number of $n_{\est}$ iterations. Instead of computing $A$ then
  $\ve{x}$, we opt for the fast proxy which consists in replacing
  $\ve{c}(A)$ by a random data point, thus \textit{without} making the
  ${\mathcal{N}}$-packed test. This should reasonably overestimate
  $\updeltas$ and thus slightly loosen our approximation bounds.
\end{itemize}
Figure \ref{f-exp1} shows the dataset obtained for $d=10$ at the end
of the process. Predictably, the distribution on the whole space looks
like a
highly non-uniform cover by locally uniform clusters. Tables
\ref{t-exp1}, \ref{t-exp2} and \ref{t-exp3} display results obtained for three different values of
$\epsilon$ and three different values for the couple $(d,k)$. To test the large sample regime intuition and the fact
that the the noise dependence grows sublinearly in $m$, we have
regressed in each plot $\tilde{\epsilon}$ as a function of $m$ for 
\begin{eqnarray}
\tilde{\epsilon}(m) & = & a + b\log
  m\:\:.
\end{eqnarray}
The plots obtained confirm a good approximation of this intuition,
but they also display some more good news. The smaller
$\epsilon$, the larger can be $\tilde{\epsilon}$ relatively to
$\epsilon$, by order of magnitudes if $\epsilon$ is small. Hence,
despite the fact that we evetually overestimate $\updeltas$, we still
get large $\tilde{\epsilon}$. Furthermore, if $k$ is small, this
"large sample" regime in which $\tilde{\epsilon} > \epsilon$
actually happens for quite small values of $m$.

Also, one may remark that the curves all look like an approximate translation of
the same curve. This is not surprising, since we can reformulate
\begin{eqnarray}
\tilde{\epsilon} & = & \epsilon + \log\left(1 -
  \frac{U}{\epsilon}\right) + g(m)\:\:,
\end{eqnarray}
whene $U\defeq (1+\updeltaw)^{k-1}$ and $g$ do not depend on
$\epsilon$. It happens that $\updeltaw$ quickly decreases to very
small values (bringing also a separate empirical validation of its
behavior as computed in ineq. (\ref{eq1deltaw2}) in the proof of
Theorem \ref{thlabf}). Hence, we rapidly get for small $m$ some
$\tilde{\epsilon}$ that looks like
\begin{eqnarray}
\tilde{\epsilon} & \approx & \epsilon + \log\left(1 -
  \frac{1+o(1)}{\epsilon}\right) + g(m)\nonumber\\
 & \approx &  h(\epsilon) + g(m)\:\:,
\end{eqnarray}
which may explain what is observed experimentally. 

We can sumarise the global picture for $\tilde{\epsilon}$ vs
$\epsilon$ by saying that it becomes more and more in favor of
$\tilde{\epsilon}$ as data size ($d$ or $m$) increase, but become less in favor of
$\tilde{\epsilon}$ as the number of clusters $k$ increases (predictably). 

\begin{table}[t]
    \centering
\begin{center}
\begin{tabular}{c|ccc}
\hline 
 & $(d,k)= (10,3)$ & $(d,k)= (50,3)$ &$(d,k)= (50,4)$ \\ \hline
$\updeltaw$ & \includegraphics[width=0.3\textwidth]{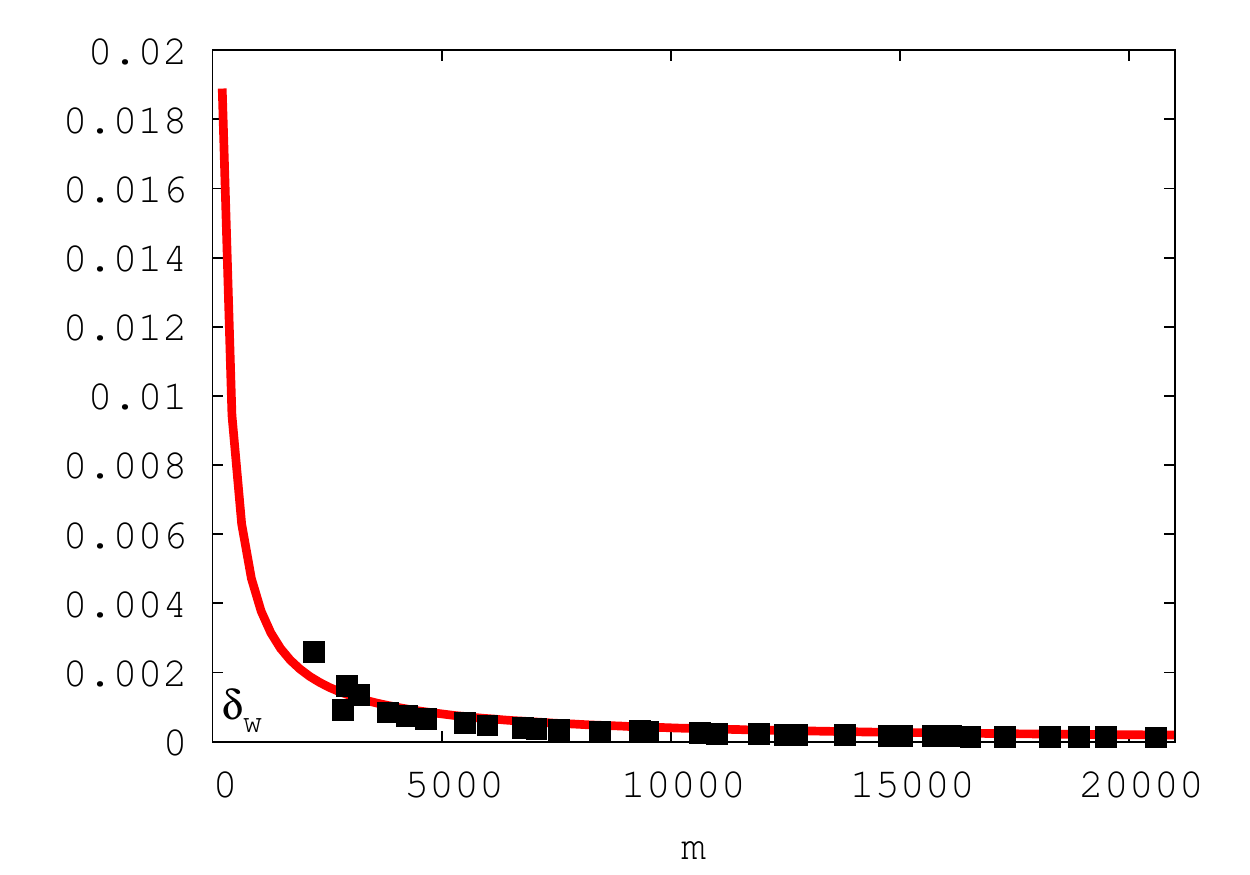}
& \includegraphics[width=0.3\textwidth]{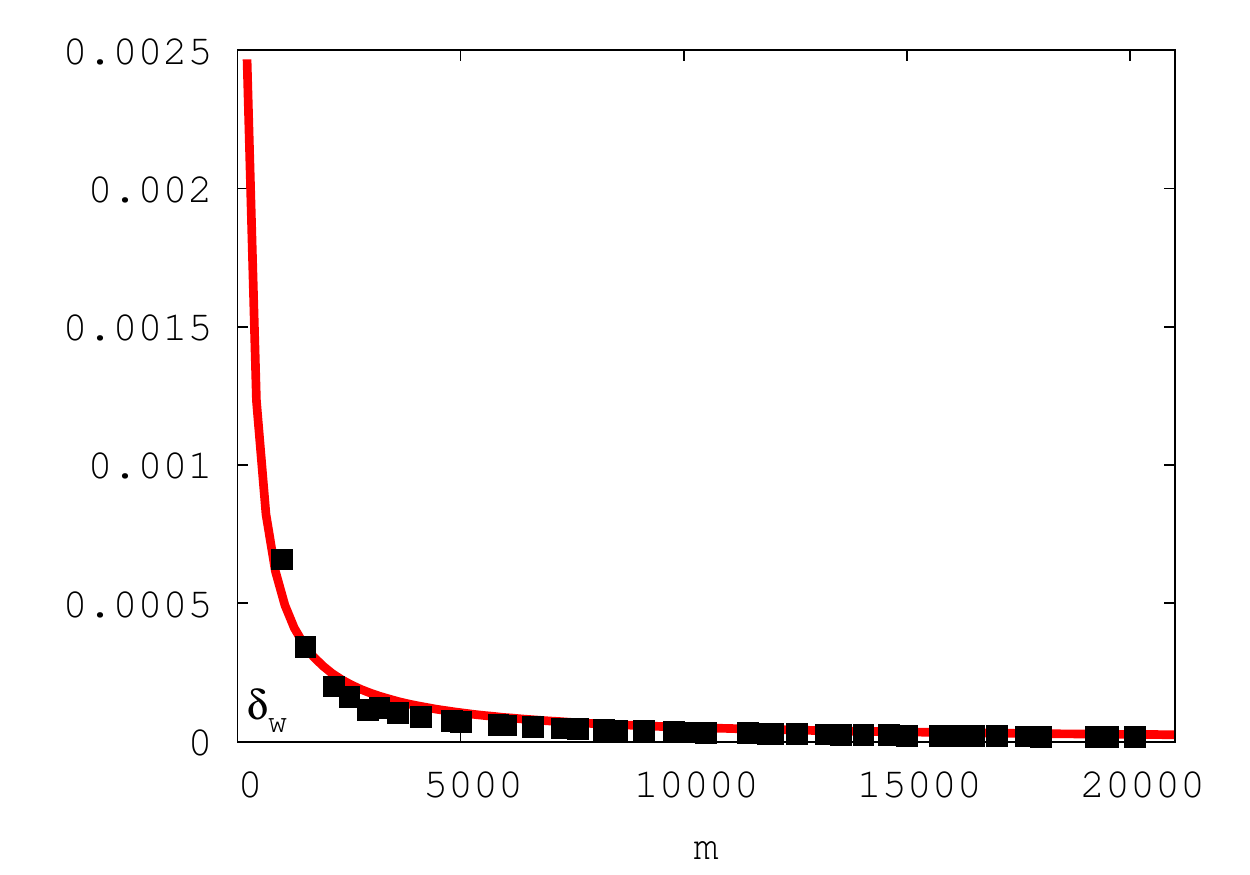}
& \includegraphics[width=0.3\textwidth]{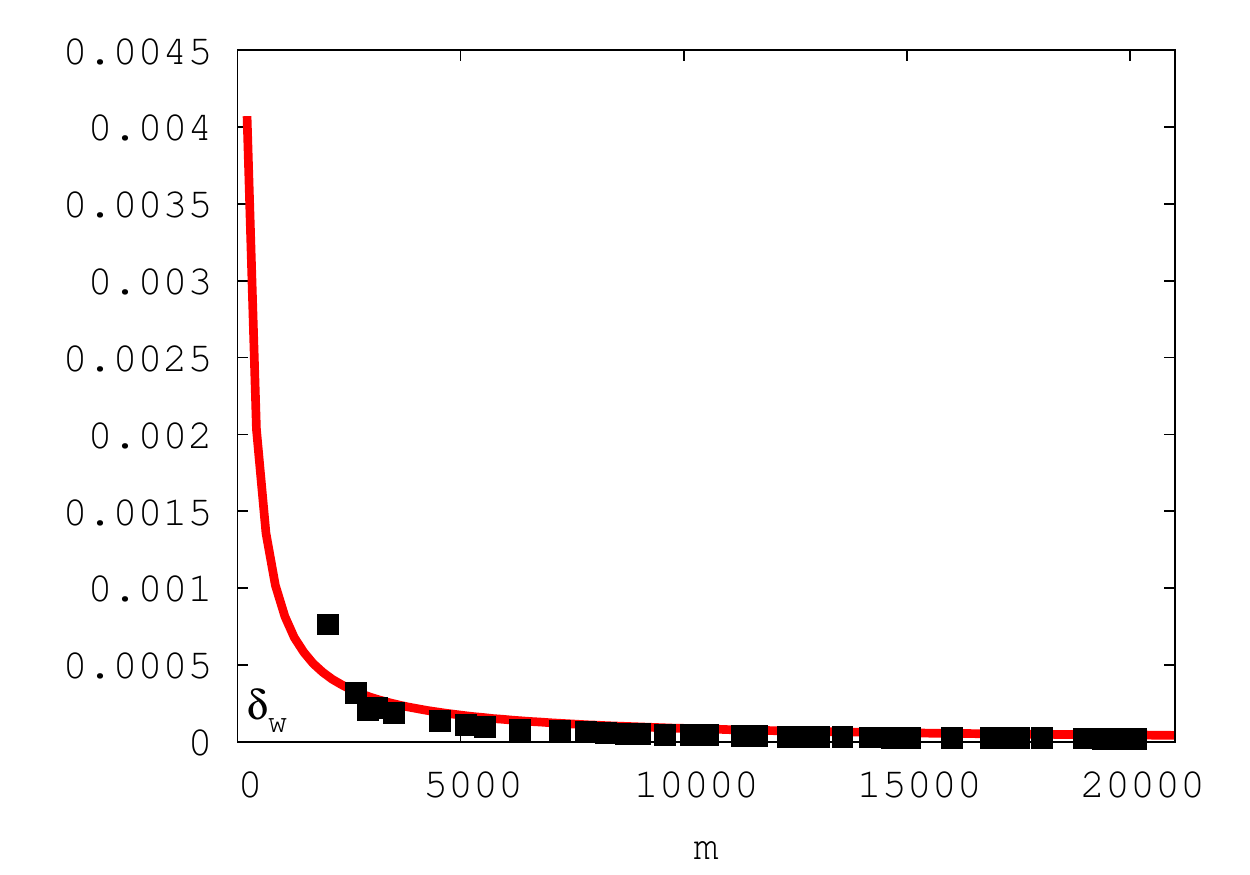}\\ \hline
 $\updeltas$ & \includegraphics[width=0.3\textwidth]{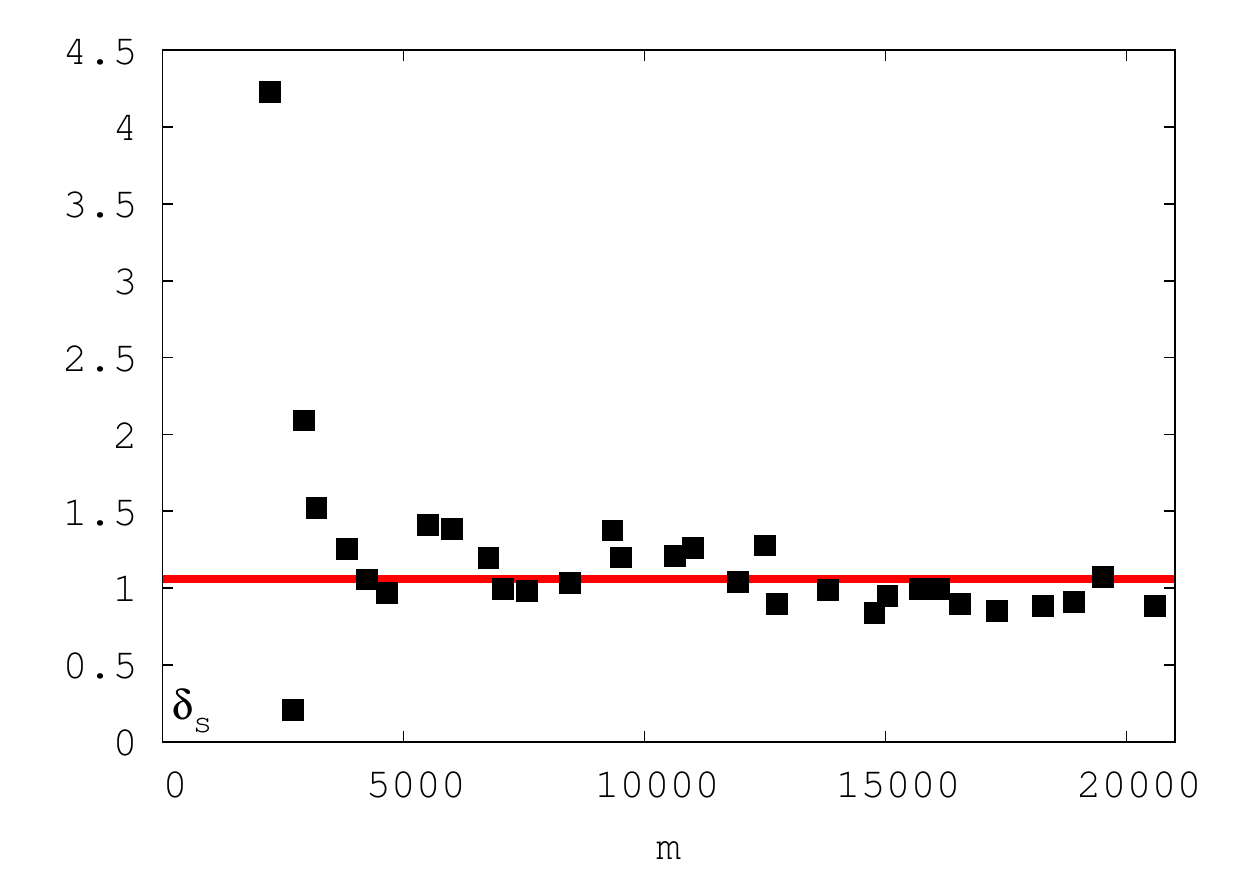}
& \includegraphics[width=0.3\textwidth]{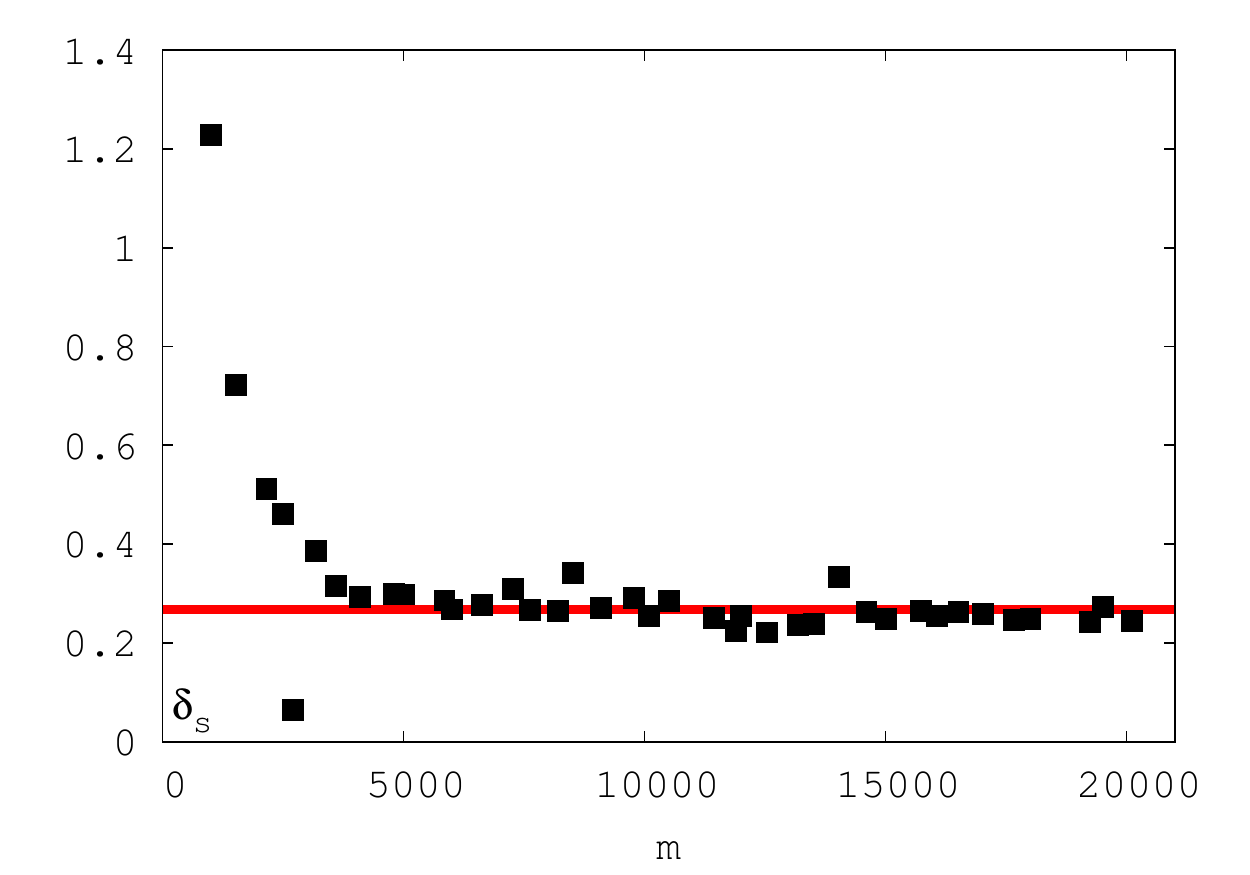}
& \includegraphics[width=0.3\textwidth]{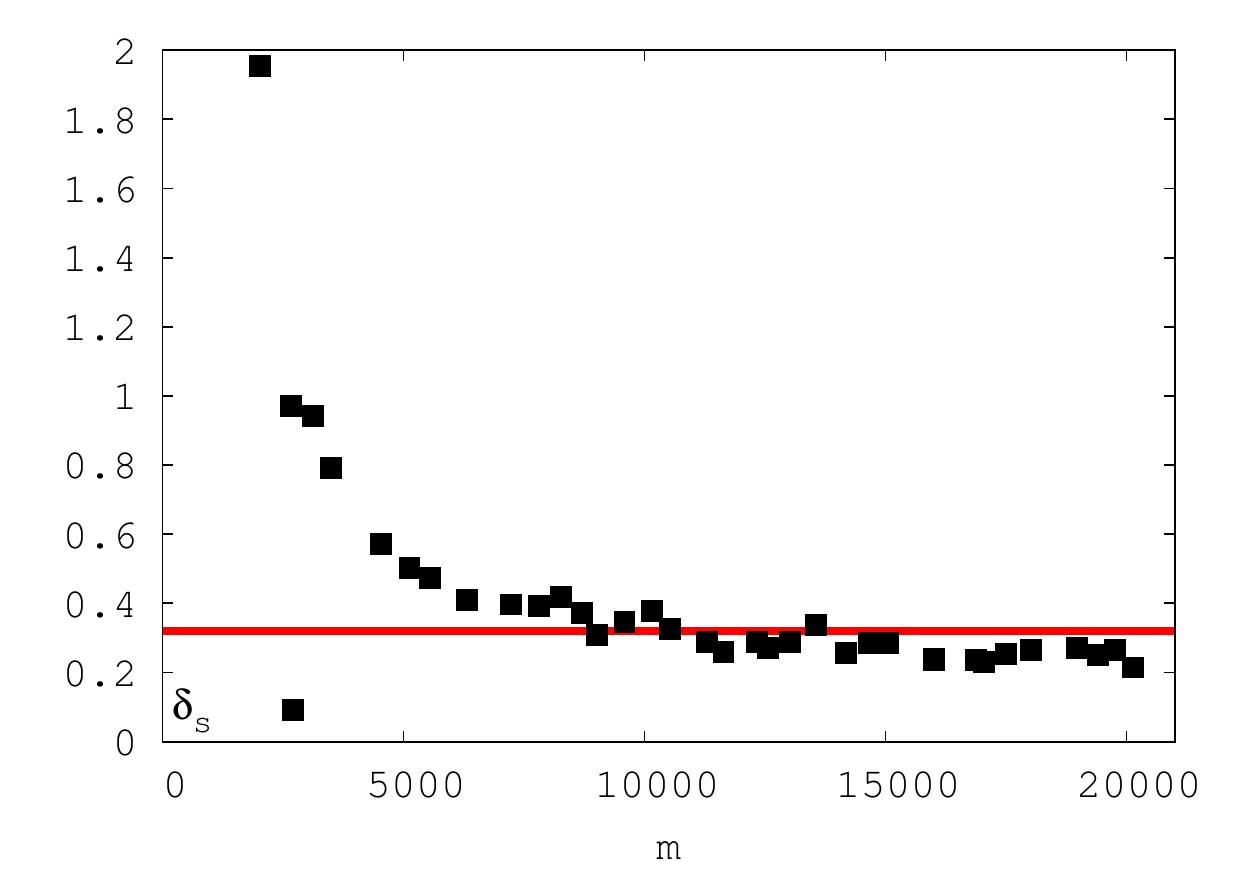}\\ \hline
\hline 
\end{tabular}
\caption{Estimations of $\updeltaw$ (top row) and $\updeltas$ (bottom
  row) as a function of $m$, for three values of $(d,k)$. We also
  indicate the best fit for $\updeltaw(m) = a/m$ (top row) and $\updeltas(m)
  = b$ (for $m\geq 4000$, bottom row).}\label{t-expDeltaWS}
\end{center}
\end{table}

\paragraph{$\hookrightarrow$ comments on $\updeltaw$ and $\updeltas$}
Table \ref{t-expDeltaWS} presents the estimated values of $\updeltaw$
and $\updeltas$ for the settings of Tables \ref{t-exp1}, \ref{t-exp2}
and \ref{t-exp3}. We wanted to test the intuition as to whether, for
$m$ sufficiently large, it would hold that $\updeltaw = O (1/m)$ while
$\updeltas = O (1)$. The essential part is on $\updeltaw$, since such
a behaviour would be sufficient for the sublinear growth of the noise
dependence. One can check that such behaviours are indeed observed,
and more: $\updeltaw$ converges very rapidly to zero, at least for all
settings in which we have tested data generation. Another quite good
news, is that $\updeltas$ seems indeed to be $\theta(1)$, but for an
actual value which is also not large, so the denominator of
eq. (\ref{defepsilont-si}) is actually driven by $f(k)$, even when, as
we already said, we may have a tendency to overestimate $\updeltas$
with our randomized procedure.

\subsection*{Experiments with \protectedKMPP, $k$-means++ and $k$-means$_{\tiny{\mbox{$\|$}}}$}\label{exp_dkm1}

\begin{figure}[t]
    \centering
\includegraphics[width=0.5\textwidth]{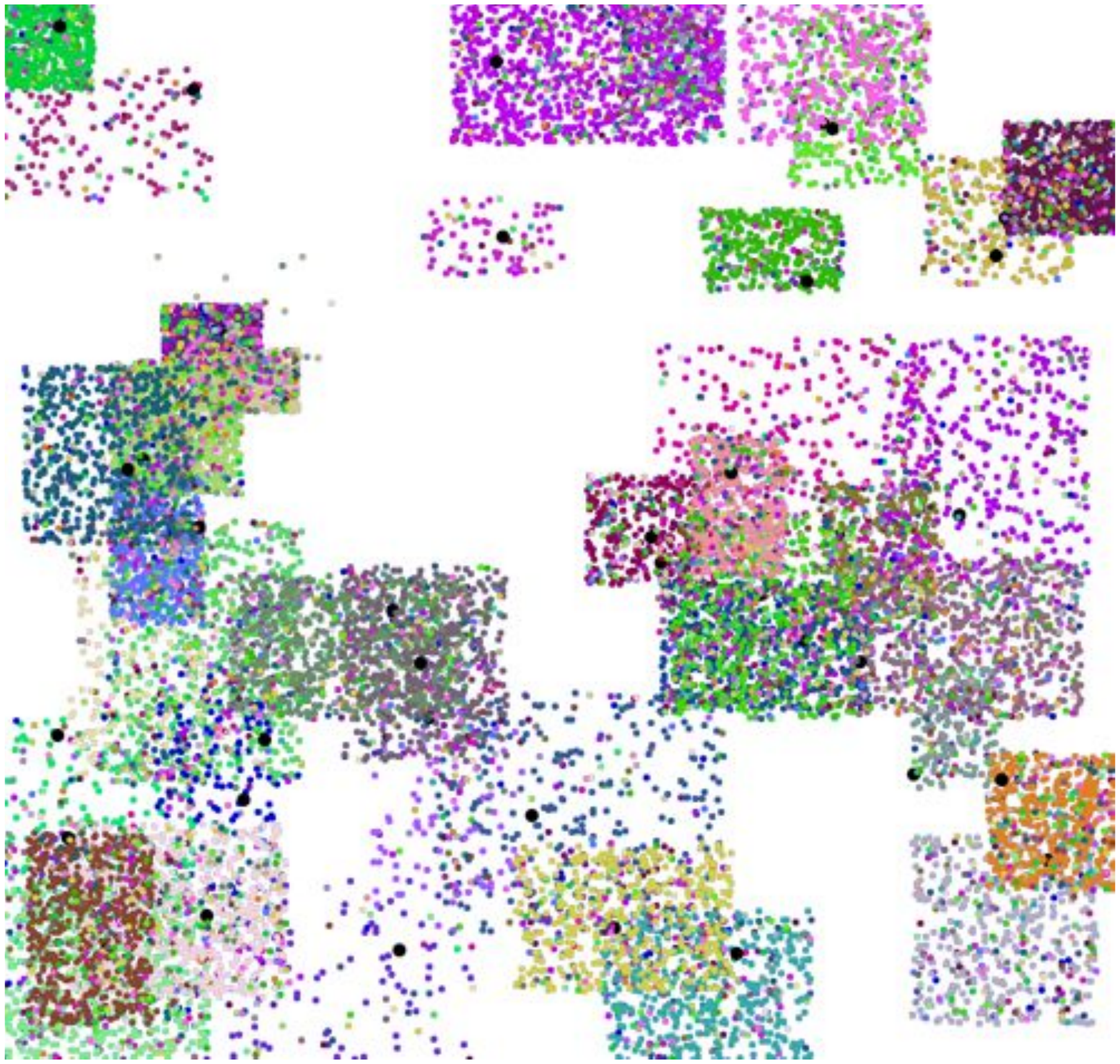}
\caption{Example dataset obtained for $p=50\%$ ($d=50$). Each color represents
  the points held by a peer (Forgy node) after the process of moving
  each point from a true cluster to another cluster with probability
  $p=0.5$. Big black dots are the datapoints
  that are the closet to the true cluster centers.}\label{f-expProtected}
\end{figure}

\begin{table}[t]
    \centering
\begin{center}
\begin{tabular}{ccccccc}
\hline 
 \includegraphics[width=0.15\textwidth]{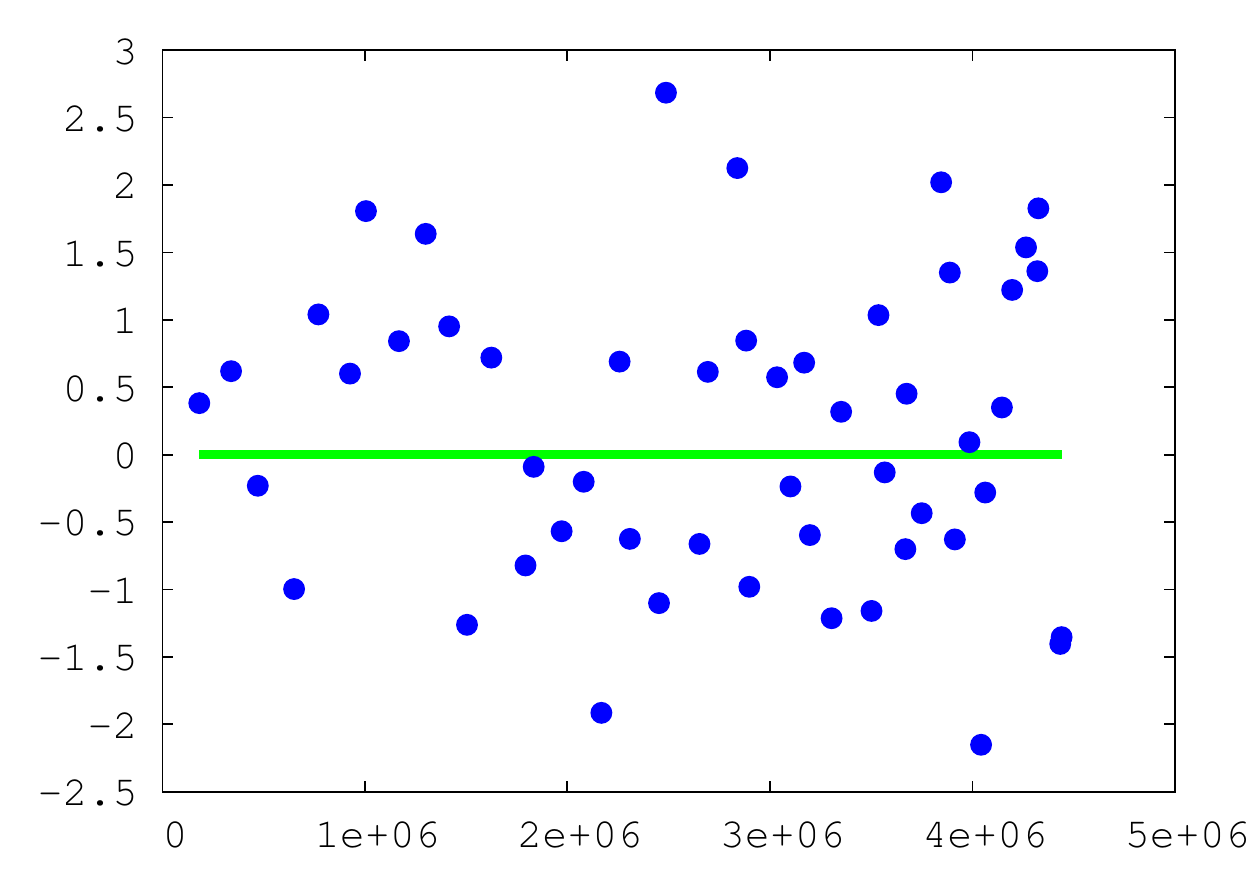}
 \hspace{-0.4cm} & \hspace{-0.4cm} \includegraphics[width=0.15\textwidth]{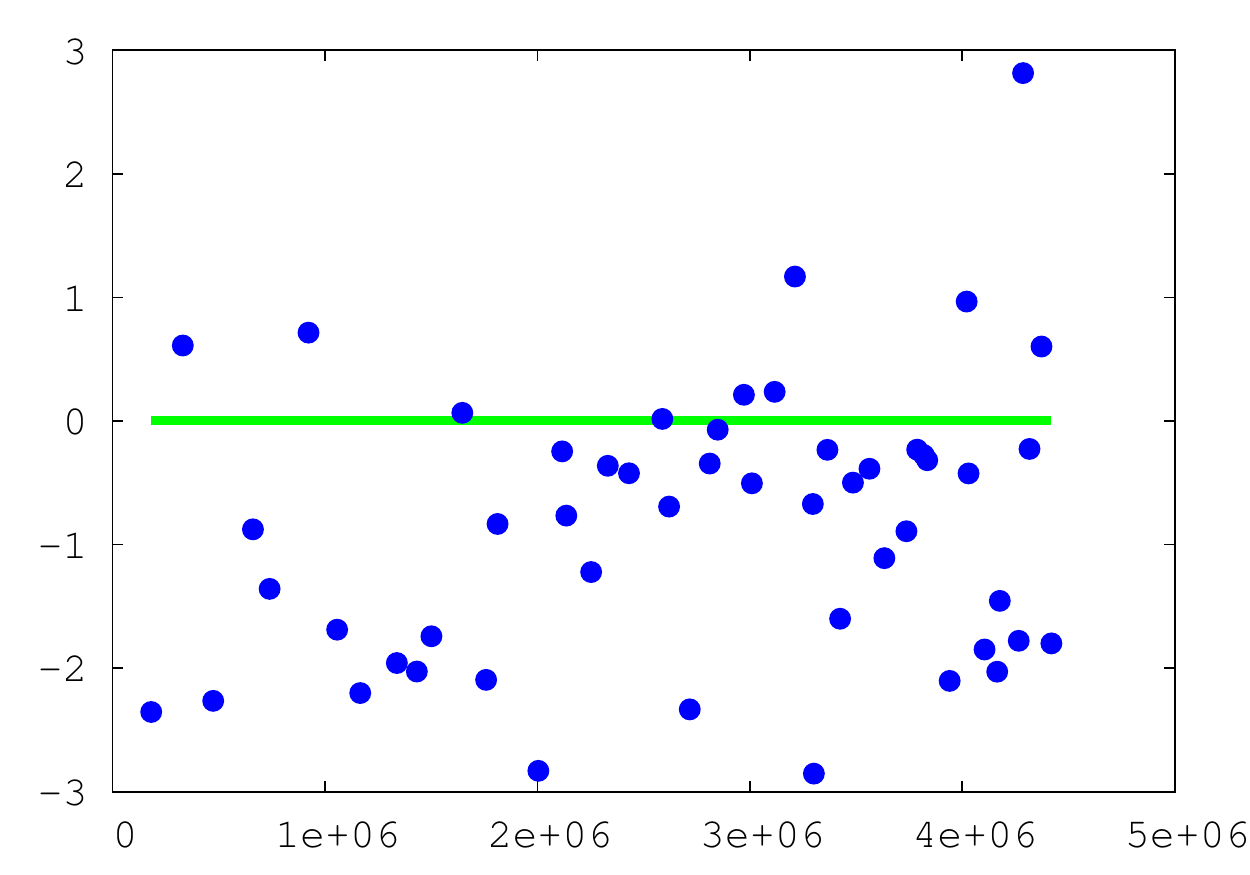}
 \hspace{-0.4cm} & \hspace{-0.4cm} \includegraphics[width=0.15\textwidth]{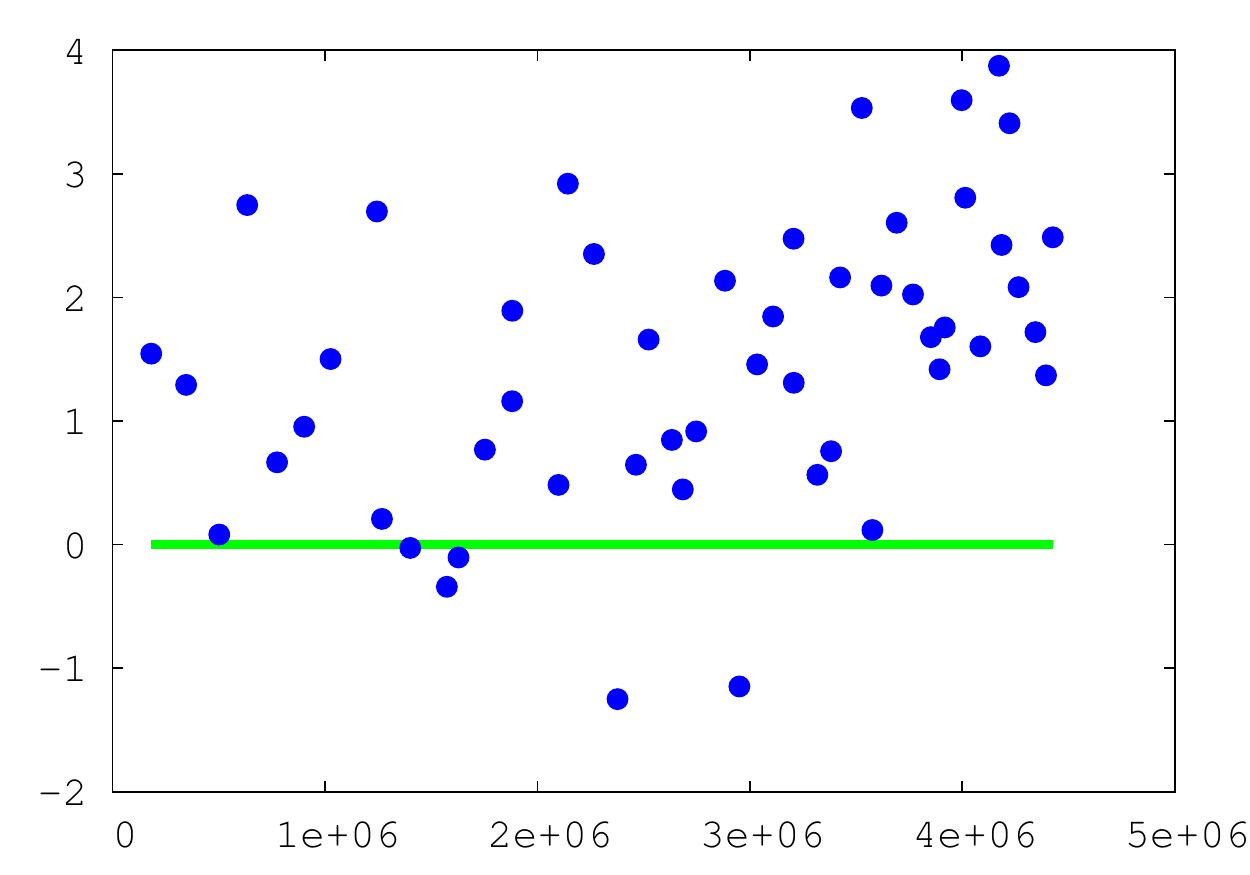}
 \hspace{-0.4cm} & \hspace{-0.4cm} \includegraphics[width=0.15\textwidth]{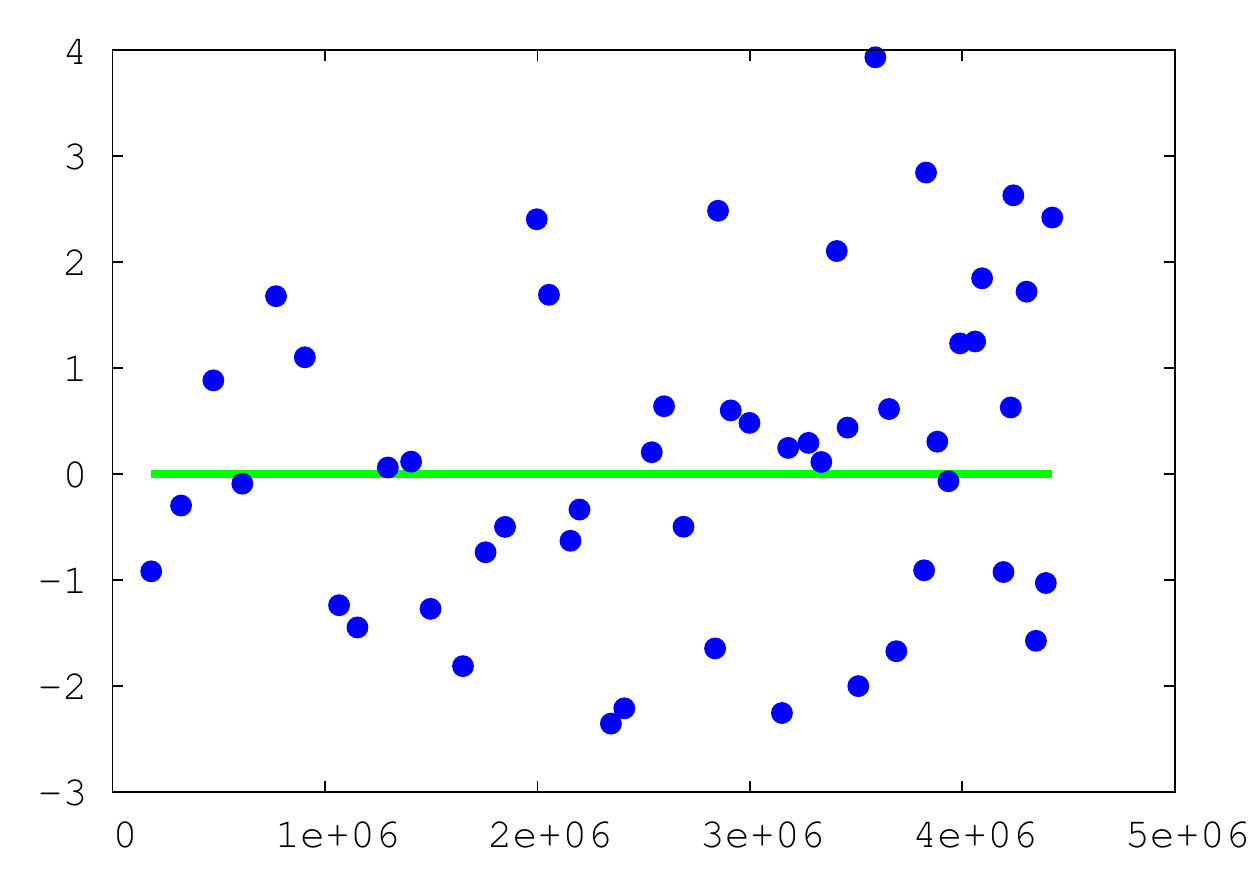}
 \hspace{-0.4cm} & \hspace{-0.4cm} \includegraphics[width=0.15\textwidth]{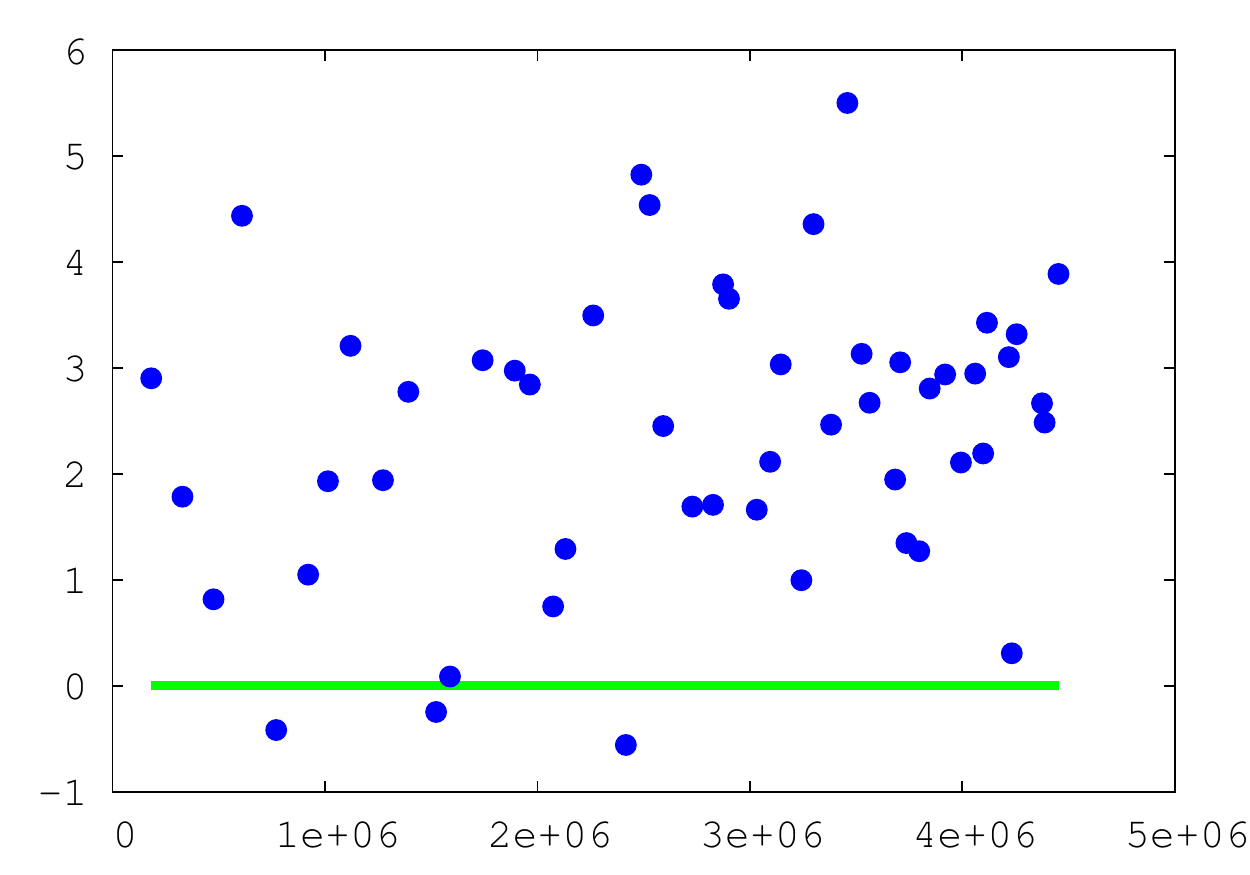}
 \hspace{-0.4cm} & \hspace{-0.4cm} \includegraphics[width=0.15\textwidth]{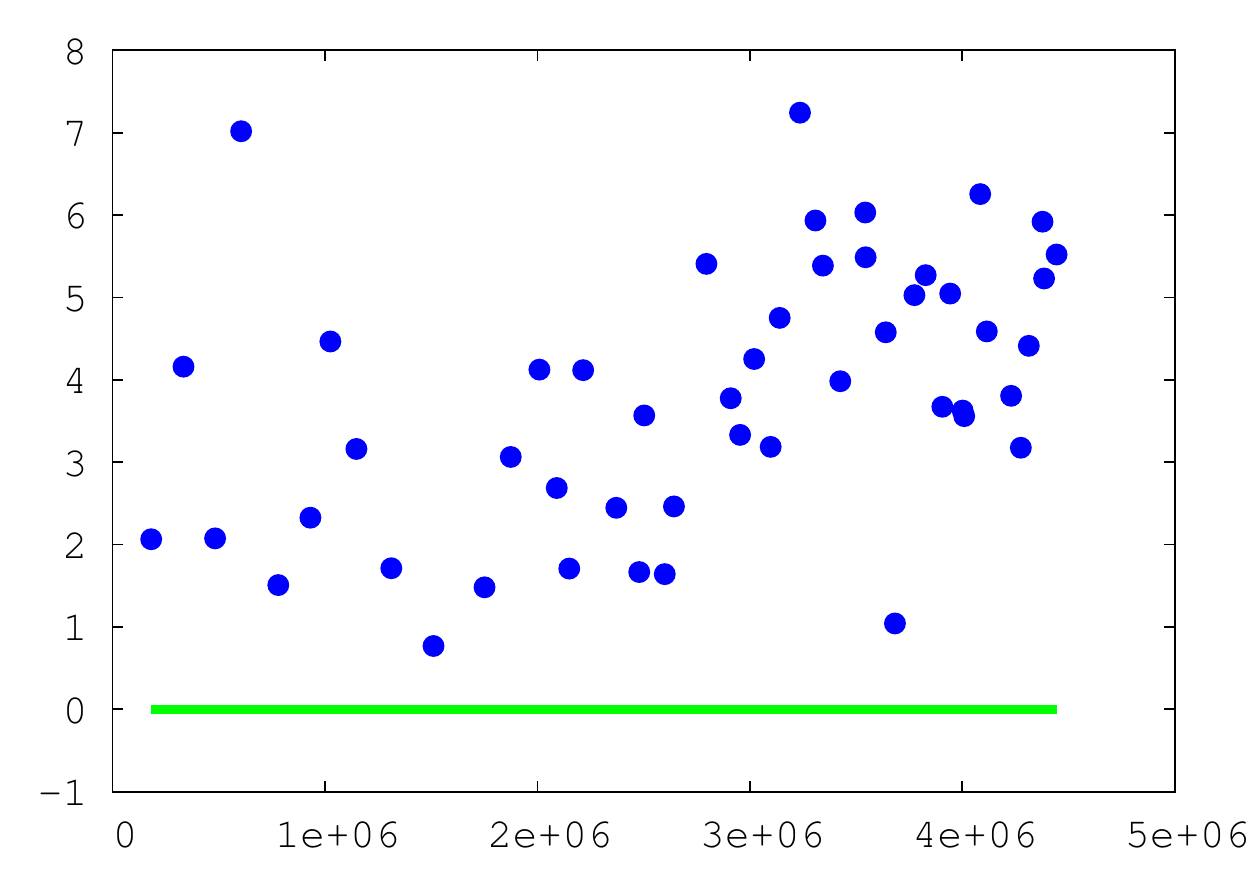}
 \hspace{-0.4cm} & \hspace{-0.4cm} \includegraphics[width=0.15\textwidth]{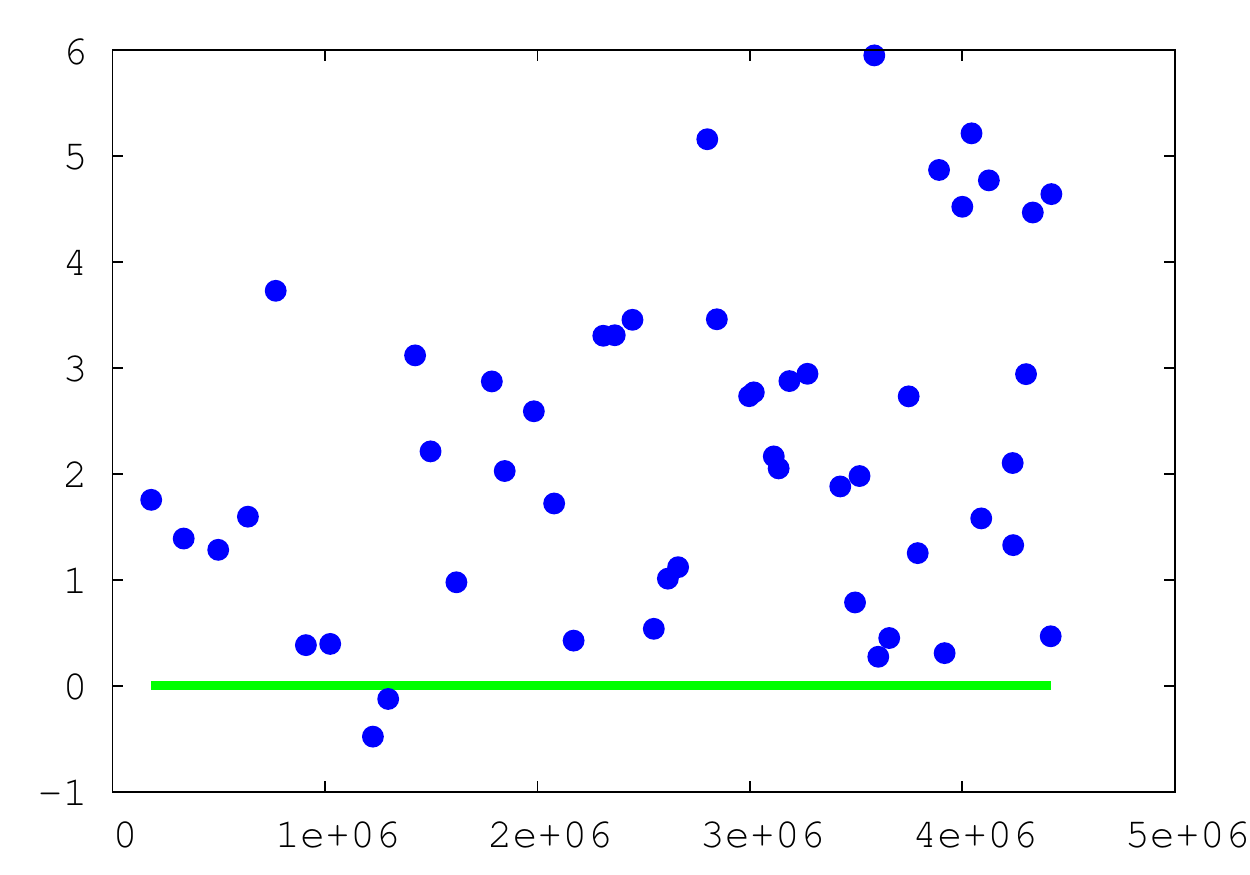}
 \\ 
$k=4$ \hspace{-0.4cm} & \hspace{-0.4cm} $k=5$ \hspace{-0.4cm} & \hspace{-0.4cm}$k=6$ 
\hspace{-0.4cm} & \hspace{-0.4cm} $k=7$ \hspace{-0.4cm} & \hspace{-0.4cm}$k=8$ 
\hspace{-0.4cm} & \hspace{-0.4cm} $k=9$ \hspace{-0.4cm} & \hspace{-0.4cm} $k=10$ \\ \hline
\hline 
\end{tabular}
\caption{Simulated data --- Plot of ratio $\uprho_\phi(\mbox{$k$-means++})$ in eq. (\ref{defratio1}) as a function
  of $\phi_s^F$. Points
  \textit{below} the green line correspond to (average) runs in which
  \protectedKMPP~\textit{beats} $k$-means++.
}\label{t-expDKM1}
\end{center}
\end{table}

\begin{table}[t]
    \centering
\begin{center}
\begin{tabular}{ccccccc}
\hline 
 \includegraphics[width=0.15\textwidth]{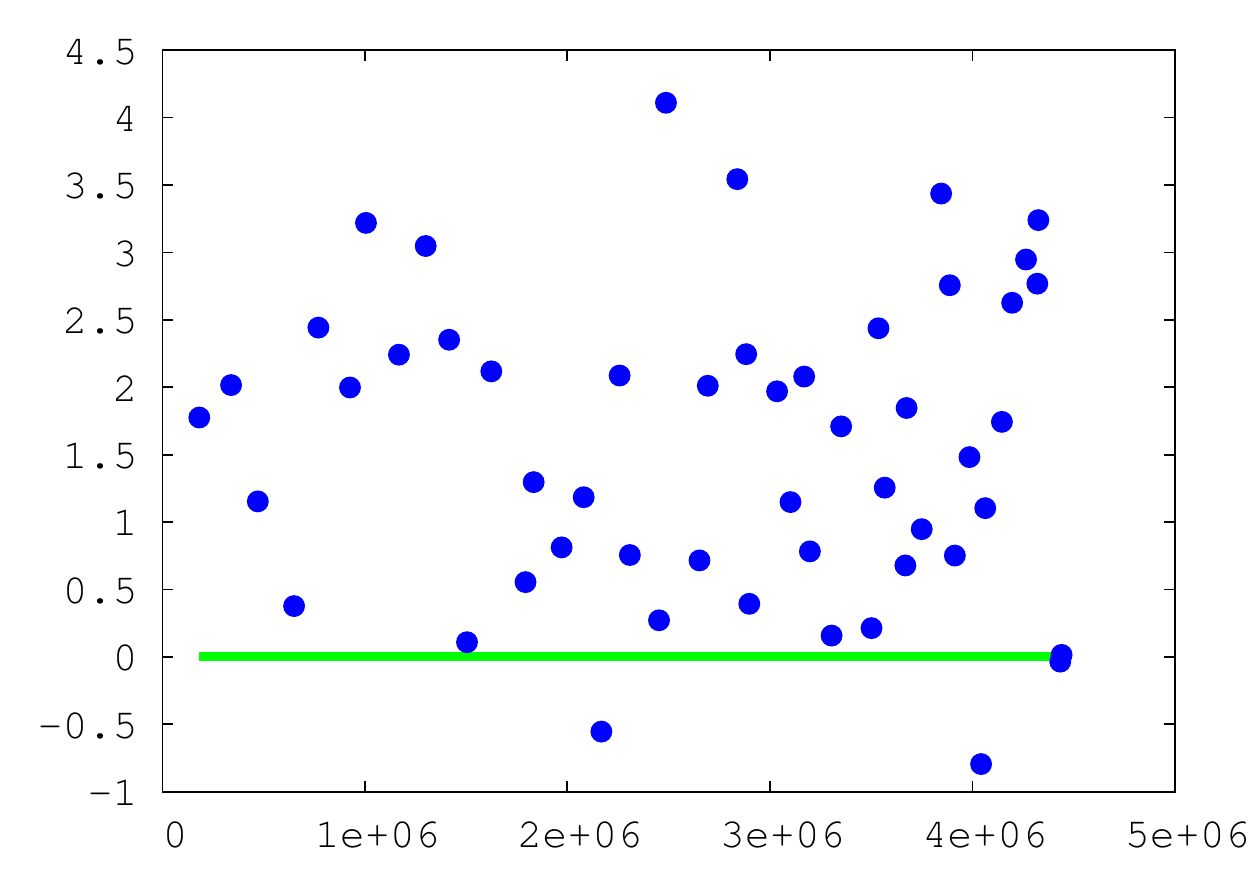}
 \hspace{-0.4cm} & \hspace{-0.4cm} \includegraphics[width=0.15\textwidth]{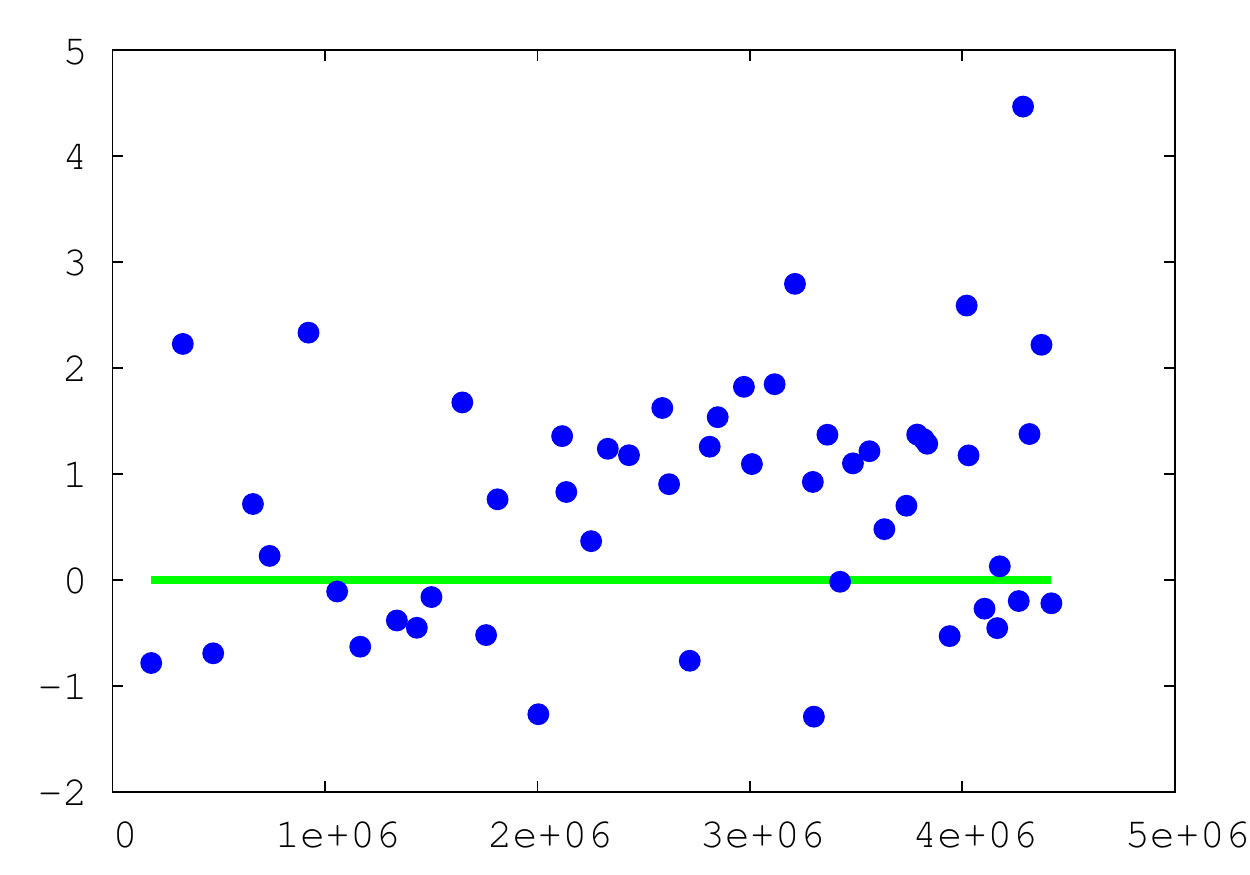}
 \hspace{-0.4cm} & \hspace{-0.4cm} \includegraphics[width=0.15\textwidth]{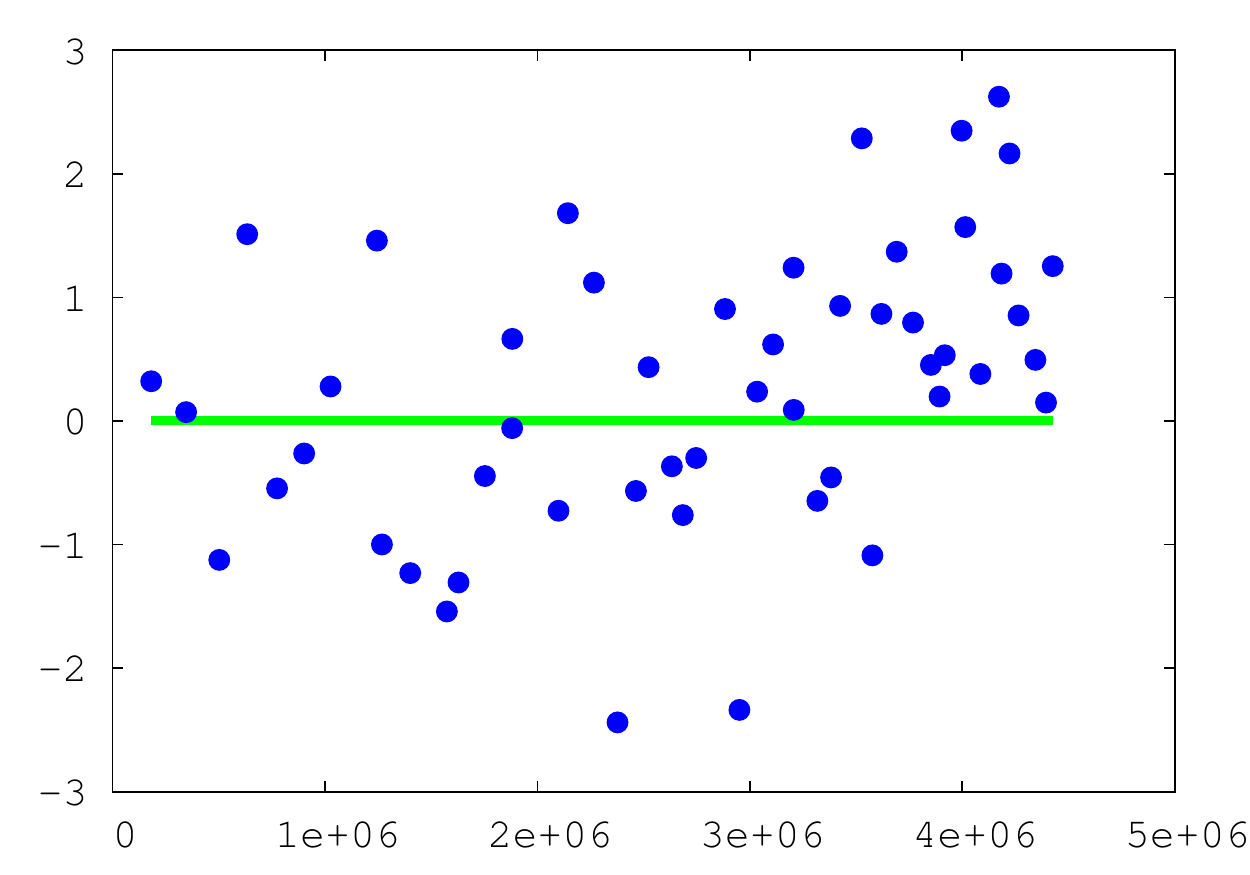}
 \hspace{-0.4cm} & \hspace{-0.4cm} \includegraphics[width=0.15\textwidth]{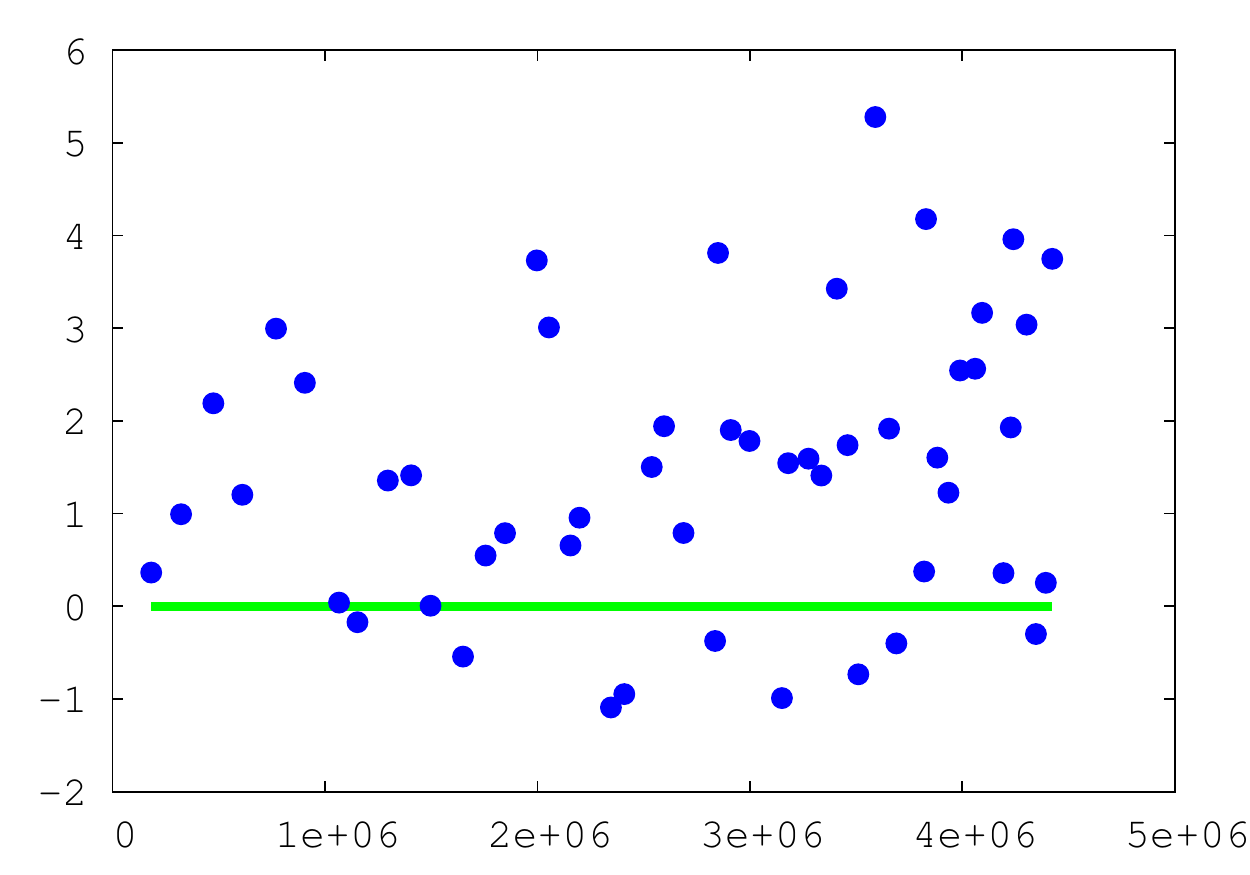}
 \hspace{-0.4cm} & \hspace{-0.4cm} \includegraphics[width=0.15\textwidth]{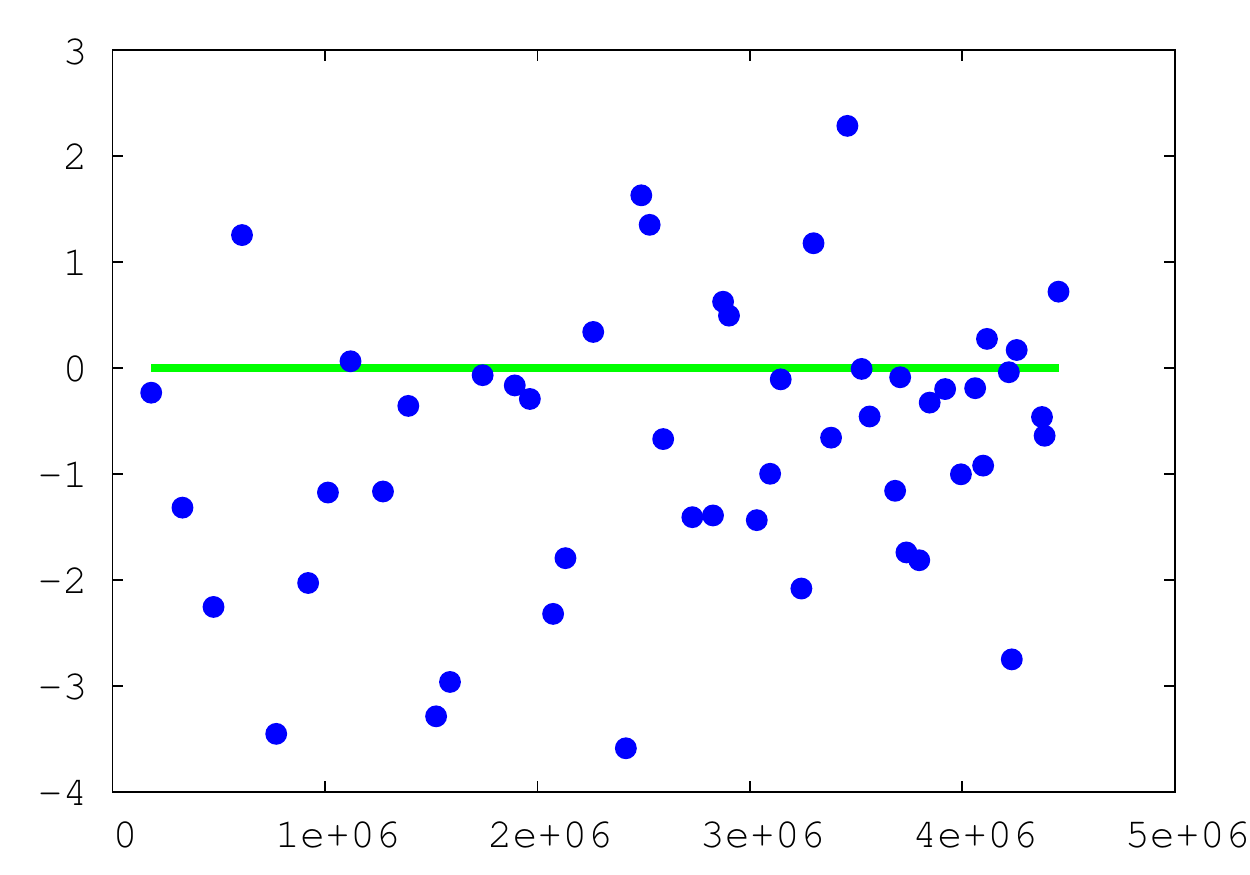}
 \hspace{-0.4cm} & \hspace{-0.4cm} \includegraphics[width=0.15\textwidth]{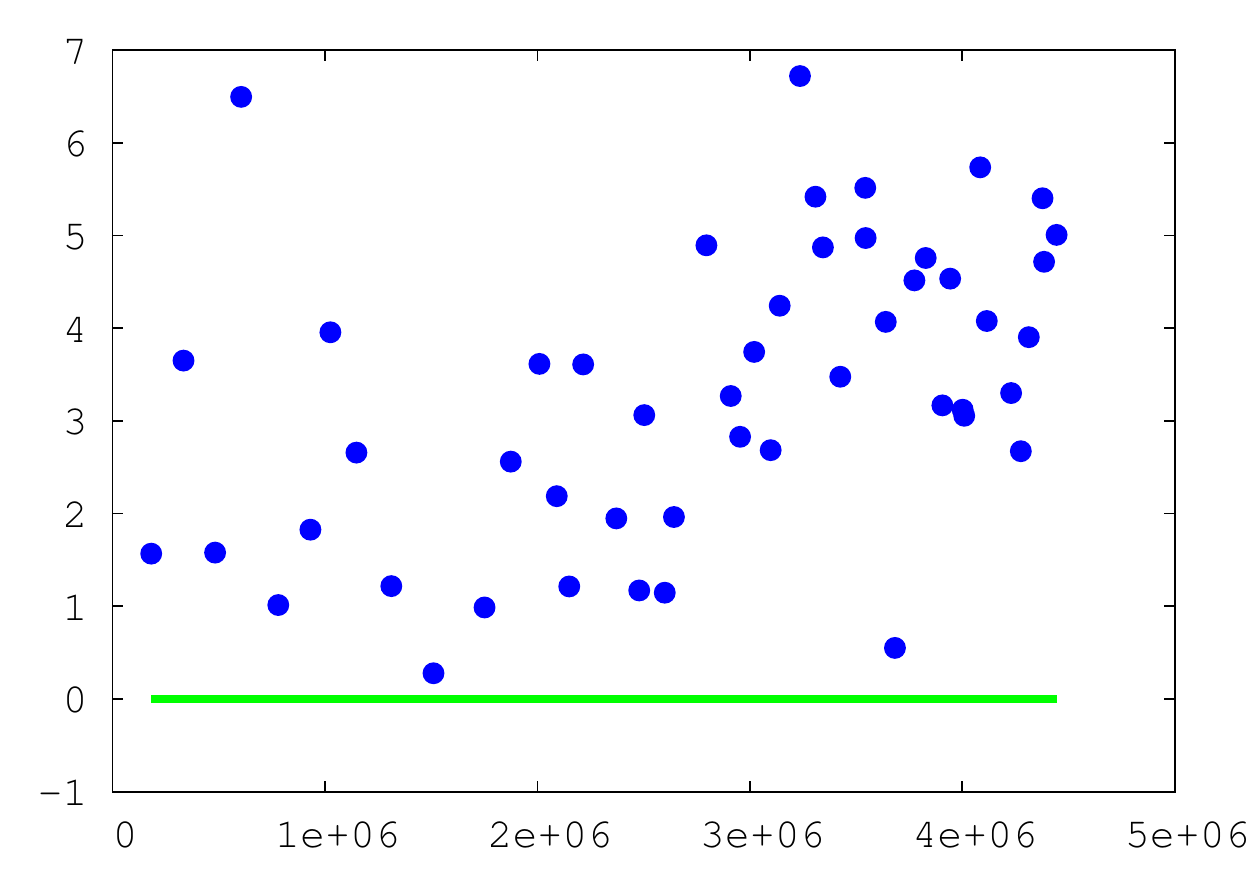}
 \hspace{-0.4cm} & \hspace{-0.4cm} \includegraphics[width=0.15\textwidth]{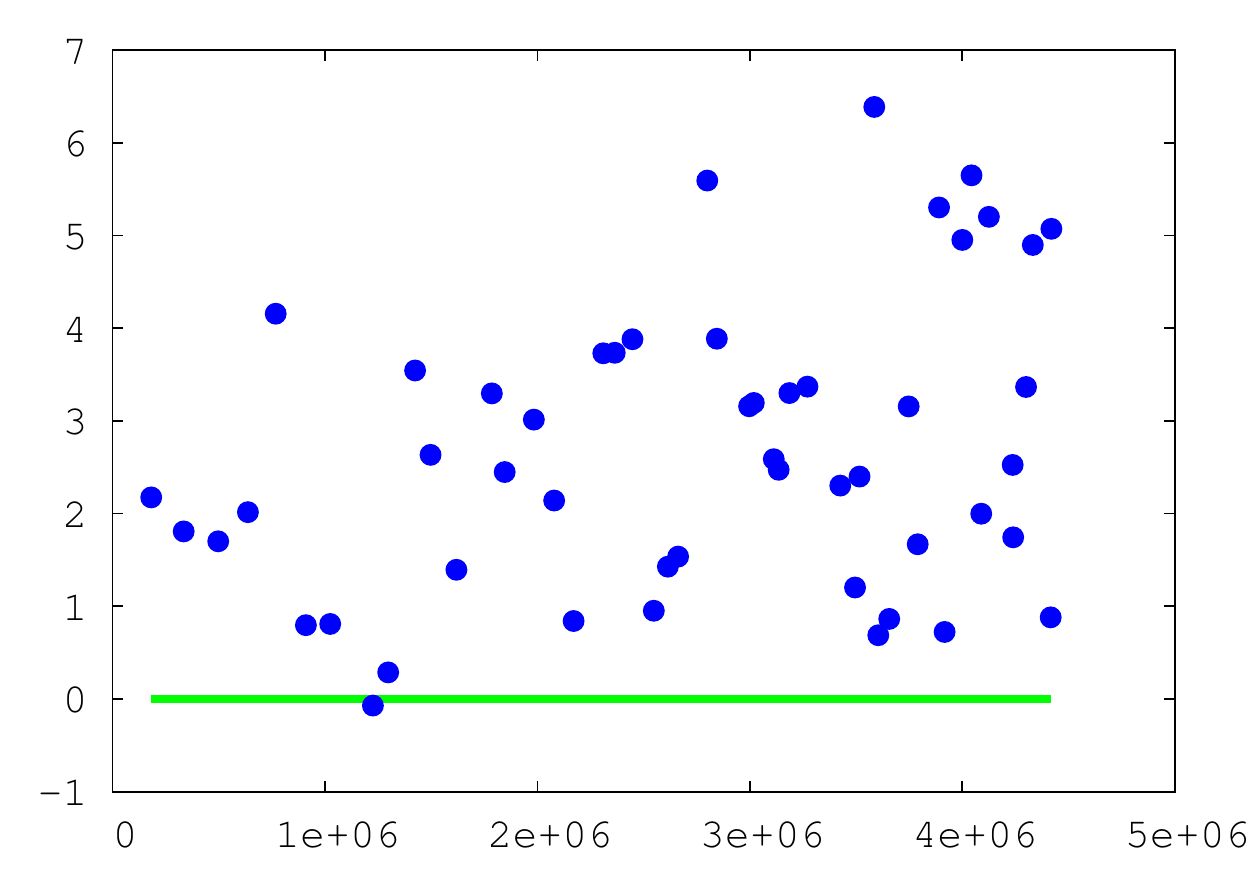}
 \\ 
$k=4$ \hspace{-0.4cm} & \hspace{-0.4cm} $k=5$ \hspace{-0.4cm} & \hspace{-0.4cm}$k=6$ 
\hspace{-0.4cm} & \hspace{-0.4cm} $k=7$ \hspace{-0.4cm} & \hspace{-0.4cm}$k=8$ 
\hspace{-0.4cm} & \hspace{-0.4cm} $k=9$ \hspace{-0.4cm} & \hspace{-0.4cm} $k=10$ \\ \hline
\hline 
\end{tabular}
\caption{Simulated data --- Plot of ratio $\uprho_\phi(\mbox{$k$-means$_{\tiny{\mbox{$\|$}}}$})$ in eq. (\ref{defratio1}) as a function
  of $\phi_s^F$. Points
  \textit{below} the green line correspond to (average) runs in which
  \protectedKMPP~\textit{beats} $k$-means$_{\tiny{\mbox{$\|$}}}$.
}\label{t-expDKM2}
\end{center}
\end{table}

\begin{figure}[t]
    \centering
\includegraphics[width=0.5\textwidth]{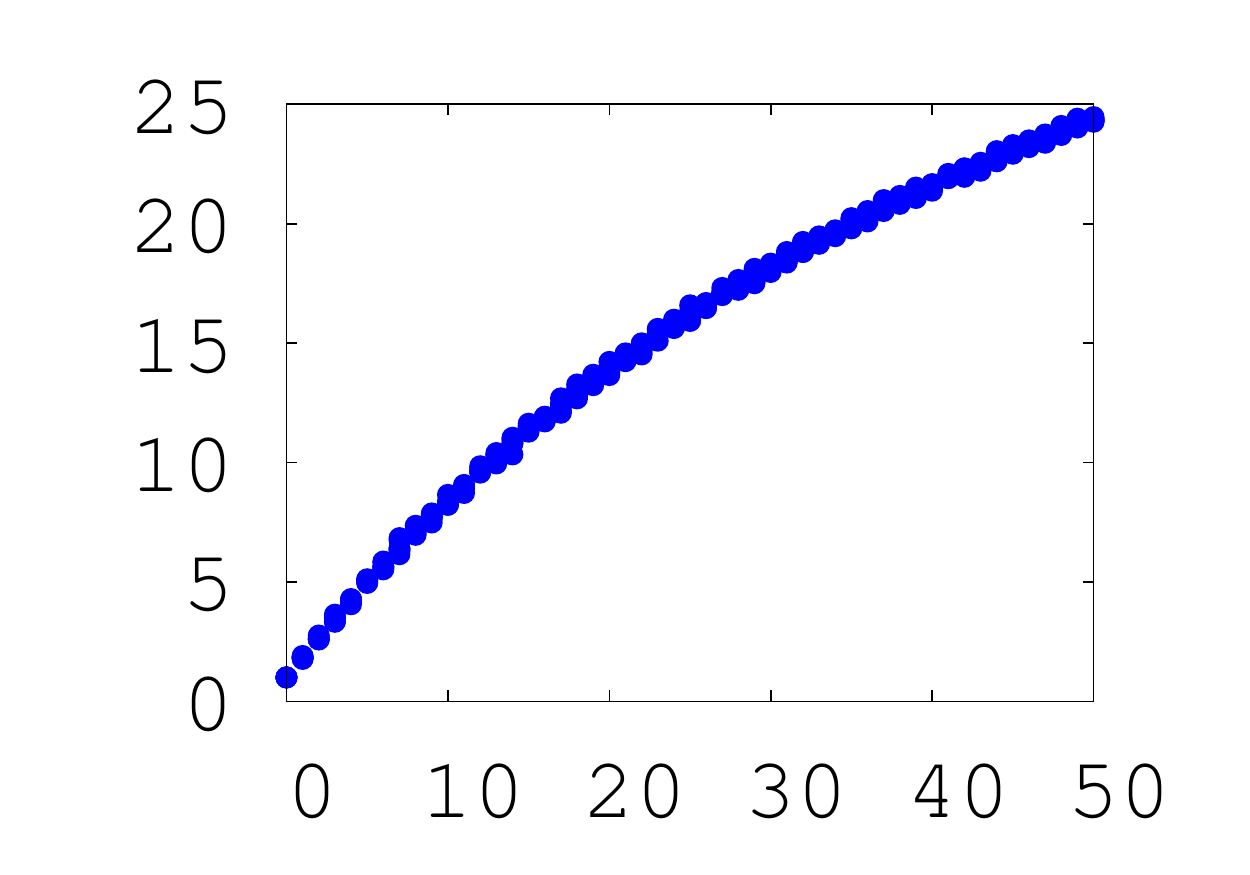}
\caption{Simulated data --- Relative increase of spread, $\phi_s^F(p)/\phi_s^F(0)$, through the runs, as a function
  of $p$.}\label{f-spread1}
\end{figure}

\paragraph{$\star$ Experiments on synthetic data} We have generated a set of $m\approx$20 000 points using the same kind of
clusters as in the experiments related to Theorem \ref{thdp1}: we add
"true" clusters until the total number of points exceeds 20 000. To
simulate the spread of data among peers (Forgy nodes) and evaluate the influence of
the spread of Forgy nodes ($\phi_s^F$) for \protectedKMPP, we have
devised the following protocol: let us name "true" clusters the
hyperrectangle clusters used to build the dataset. Each true cluster
corresponds to the data held by a peer. Then, for some $p \in [0,100]$
($\%$),
\textit{each} point in \textit{each} true cluster moves into another
cluster, with probability $p$. The choice of the target cluster is
made uniformly at random. Thus, as $p$ increases, we get a clustering
problem in which the data held by peers is more and more spread, and
for which the spread of Forgy nodes $\phi_s^F$ increases. Figure
\ref{f-expProtected} presents a typical example of the spread for
$p=50\%$. Notice that in this case many Forgy nodes have data spreading
through a much larger domain than the initial, true clusters. Figure
\ref{f-spread1} displays that this happens indeed, as $\phi_s^F$ is multiplied by a factor
exceeding 20 (compared to $\phi_s^F$ at $p=0$) for the largest values of $p$.

We have compared \protectedKMPP~to $k$-means++ and
$k$-means$_{\tiny{\mbox{$\|$}}}$ \citep{bmvkvSK}. In the case of that
latter algorithm, we follow the paper's statements and pick the number of outer iterations to be
$\lceil \log \phi_1\rceil$, where $\phi_1$ is the potential for one
Forgy-chosen center. We also pick $\ell = 2k$, considering that it is
a value which gives some of the best experimental results in
\citep{bmvkvSK}. Finally, we recluster the points at the end of the
algorithm using $k$-means++. For each algorithm ${\mathcal{H}} \in
\{$k$\mbox{-means++}, \mbox{$k$-means$_{\tiny{\mbox{$\|$}}}$}\}$, we
run it on the complete dataset and its results are
averaged over 10 runs. We run \protectedKMPP~for each $p\in \{0\%,
1\%, ..., 50\%\}$. More precisely, for each $p$, we average the
results of \protectedKMPP~over 10 runs. We use as metric the relative increase
in the potential of \protectedKMPP~compared to ${\mathcal{H}}$:
\begin{eqnarray}
\uprho_\phi({\mathcal{H}}) & \defeq & \frac{\phi(\mbox{\protectedKMPP}) - \phi({\mathcal{H}})}{\phi({\mathcal{H}})}\cdot 100\:\:.\label{defratio1}
\end{eqnarray}
that we plot as a function of $\phi_s^F$, or surface plot as a
function of $(k,p)$. The intuition for the former plot is that the larger $\phi_s^F$, the
larger should be this ratio, since the data held by peers spreads
across the domain and each peer is constrained to pick its centers
with uniform seeding.

\paragraph{$\hookrightarrow$ \protectedKMPP~vs $k$-means++}
Figure \ref{t-expDKM1} presents results for $\uprho_\phi(\mbox{$k$-means++}) =
f(\phi_s^F)$ obtained for various
$k$. First, the intuition is indeed confirmed for $k=8, 9, 10$, but an
interesting phenomenon appears for $k=5$: \protectedKMPP~almost
consistently beats $k$-means++. The decrease in the average potential
ranges up to $3\%$. Furthermore, this happens even for large values of
$\phi_s^F$. Finally, for \textit{all but one} value of $k$, there
exists spread values for which \protectedKMPP~beats $k$-means++. The
surface plot in Figure \ref{f-par1} displays that superior performances
of \protectedKMPP~are probably not random.
One
possible explanation to this phenomenon relies on the expression of
$\phi_{\bias}$ given in the proof of Theorem \ref{thDKM}
(eq. (\ref{spreadex})), recalled here:

\begin{eqnarray}
\phi_{\bias} & \defeq & \sum_{\ve{a} \in {\mathcal{A}}} \|\ve{\mu}_{\ve{a}}
- \ve{c}_{\opt}(\ve{a})\|_2^2\nonumber\\
 & = & \sum_{i \in [n]}\sum_{\ve{a} \in
  {\mathcal{A}}_i} \|\ve{c}({\mathcal{A}}_i) -
\ve{c}_{\opt}(\ve{a})\|_2^2 \nonumber\:\:.\\
\end{eqnarray}
Recall that $\phi_{\bias}$ can be $<\phi_{\opt}$, and it can even be
zero, in which case Theorem \ref{thmain} says that the approximation
bound may actually be \textit{better} than that of $k$-means++ in
\citep{avKM} (furthermore, $\upeta = 0$ for \protectedKMPP). Hence,
what happens is pobably that in several cases, there exists a union of
peers data (the number of peers is larger than $k$) that gives a at least
reasonably good
approximation of the global optimum. In all our experiments indeed, we
obtained a number of peers larger than 30.

\paragraph{$\hookrightarrow$ \protectedKMPP~vs
  $k$-means$_{\tiny{\mbox{$\|$}}}$} Figure \ref{f-par1} appear
to display performances for \protectedKMPP~that are even more in favor
of \protectedKMPP, compared to $k$-\means. Figure \ref{t-expDKM2} presents results for $\uprho_\phi(\mbox{$k$-means$_{\tiny{\mbox{$\|$}}}$}) =
f(\phi_s^F)$ obtained for various
$k$. The fact that each of them is a vertical translation of a picture
in Figure \ref{t-expDKM1} comes from the fact that the results of
$k$-means$_{\tiny{\mbox{$\|$}}}$ and $k$-means++ do not depend on the
spread of the neighbors $\phi_s^F$.

\paragraph{$\star$ Experiments on real world data} 

We consider the EuropeDiff
dataset\footnote{\label{urlf}http://cs.joensuu.fi/sipu/datasets/} (Dataset
characteristics provided in Table \ref{t-comp}). Figures
\ref{f_ED} and \ref{f-spread_ED} give the results for the equivalent
settings of the experimental data. To simulate $N$ peers with real
data, reasonably spread geographically, 
we have sampled $N$ points ("peer centers") with $k$-means++ seeding in data and then aggregated for each
peer the subset of data in the corresponding Voronoi 1-NN cell. We
then simulate the spread for parameter $p$ as in the simulated data. Figures
\ref{f_ED} and \ref{f-spread_ED} globally display (and confirm) the
same trends as for the simulated data. They, however, clearly
emphasize this time that the spread of Forgy nodes $\phi_s^F$ is one
key parameter that drives the performances of \protectedKMPP. Notice also
that \protectedKMPP~remains on this dataset competitive up to $p\geq
30\%$, which means that it remains competitive when a significant proportion of peers' data is scattered
without any constraint.

\begin{figure}[t]
    \centering
\begin{tabular}{cc}
\hspace{-1cm} \includegraphics[width=0.55\textwidth]{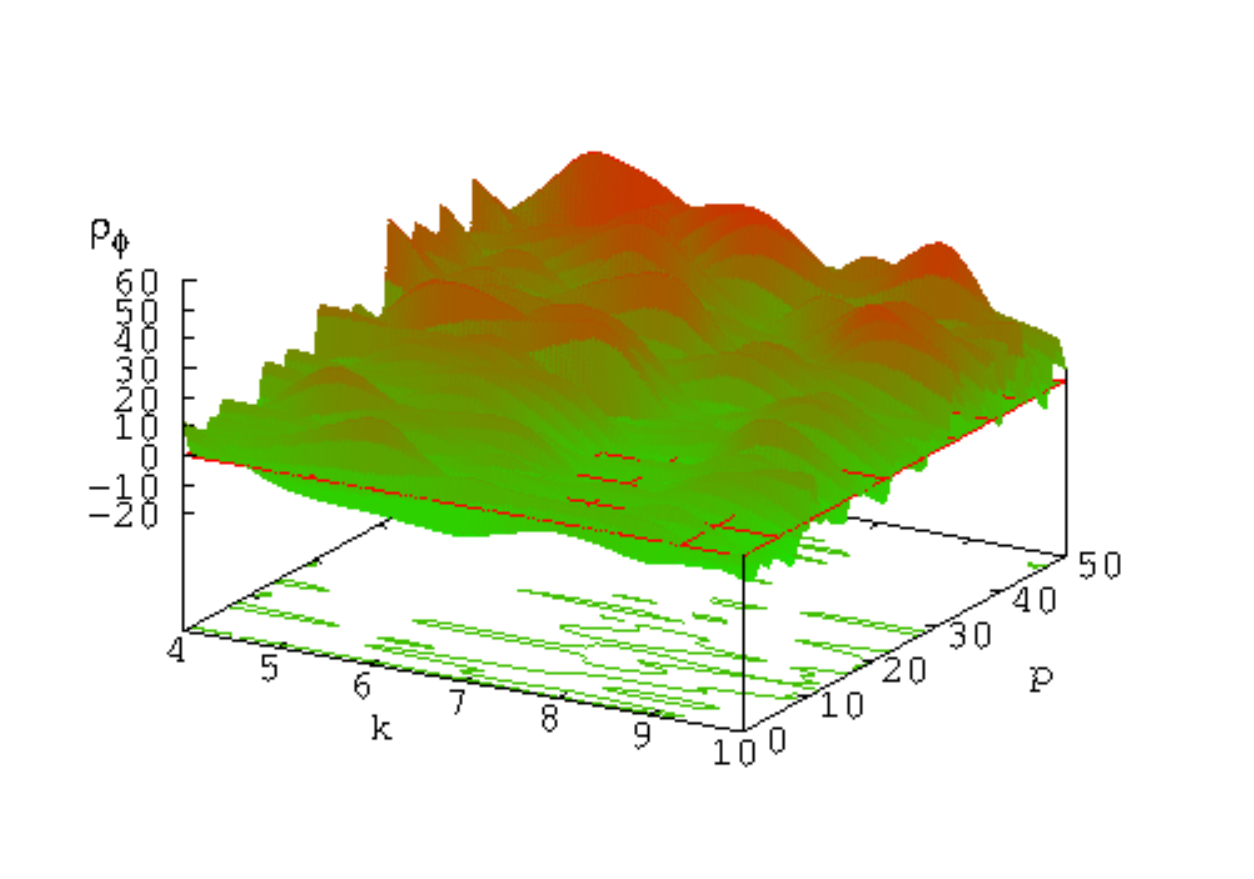}
\hspace{-1cm} & \hspace{-1cm} \includegraphics[width=0.55\textwidth]{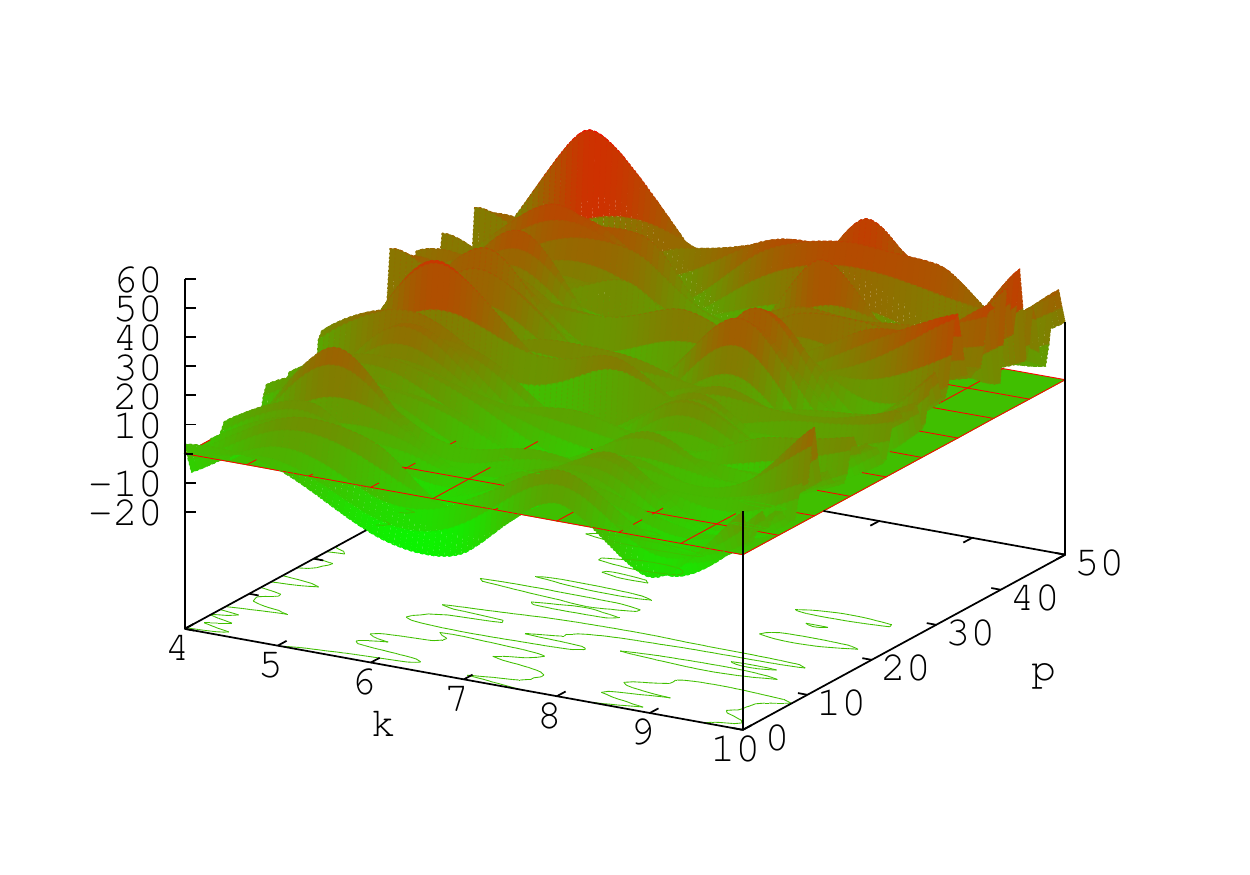}\\
 \hspace{-1cm} $\uprho_\phi(\mbox{$k$-means++})$ \hspace{-1cm} &
 \hspace{-1cm}
 $\uprho_\phi(\mbox{$k$-means$_{\tiny{\mbox{$\|$}}}$})$\\ 
\end{tabular}
\begin{tabular}{ccccccc}
\hline 
 \includegraphics[width=0.15\textwidth]{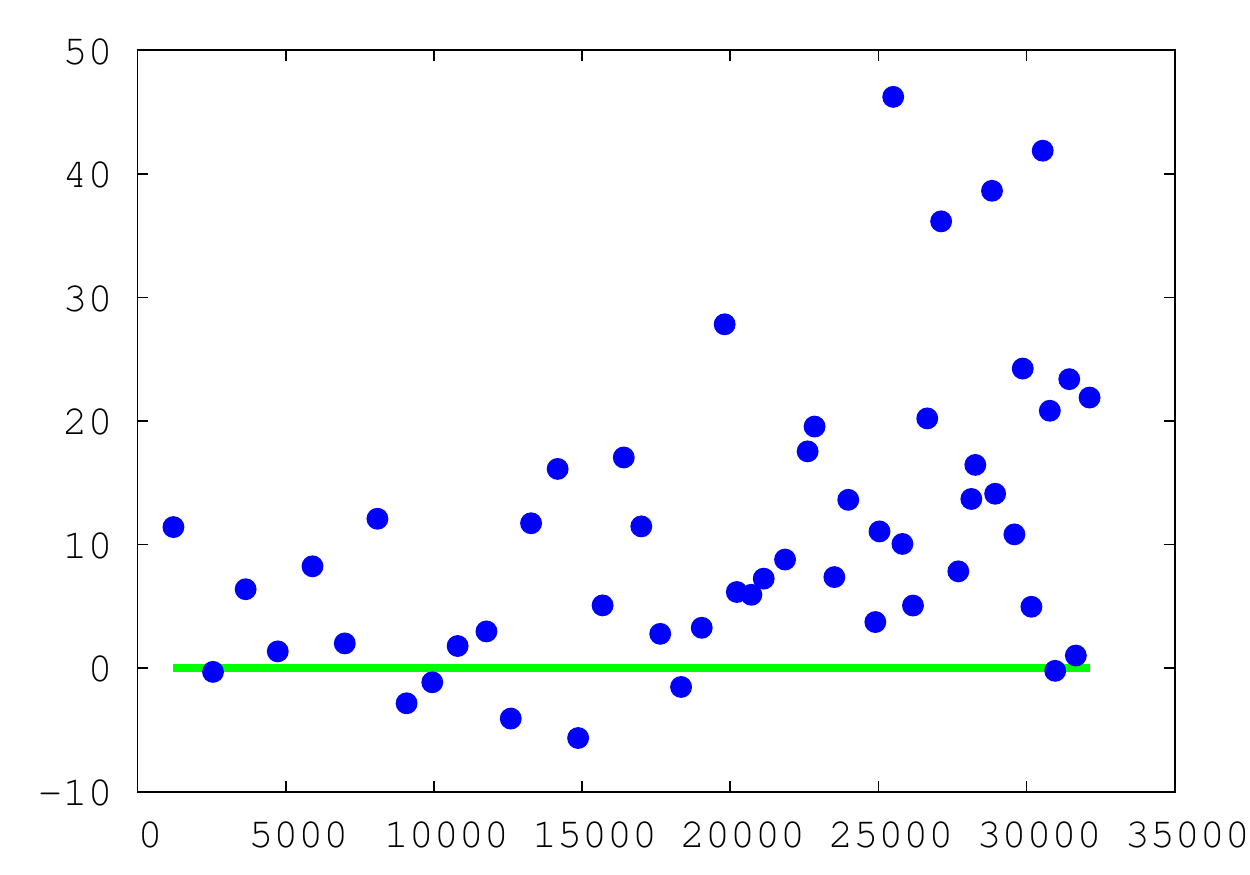}
 \hspace{-0.4cm} & \hspace{-0.4cm} \includegraphics[width=0.15\textwidth]{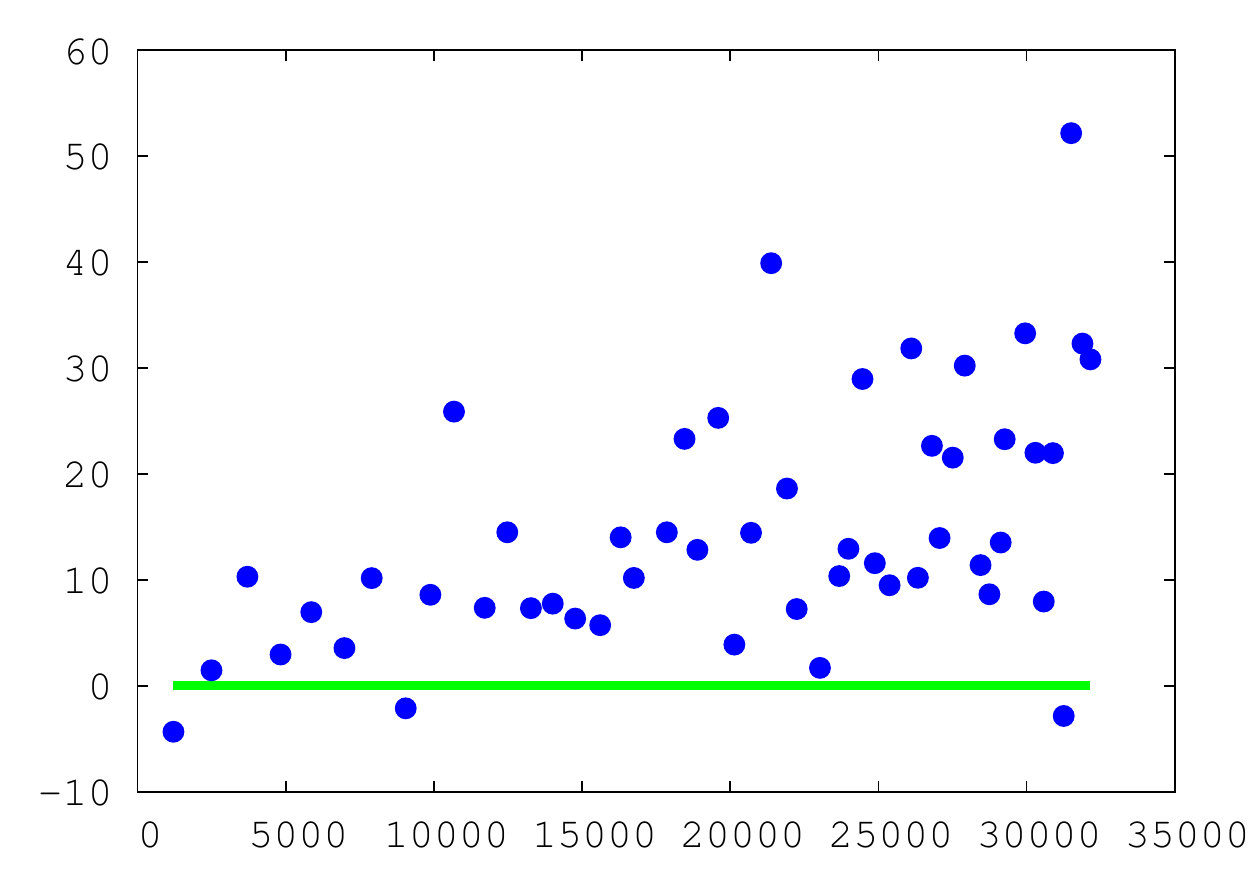}
 \hspace{-0.4cm} & \hspace{-0.4cm} \includegraphics[width=0.15\textwidth]{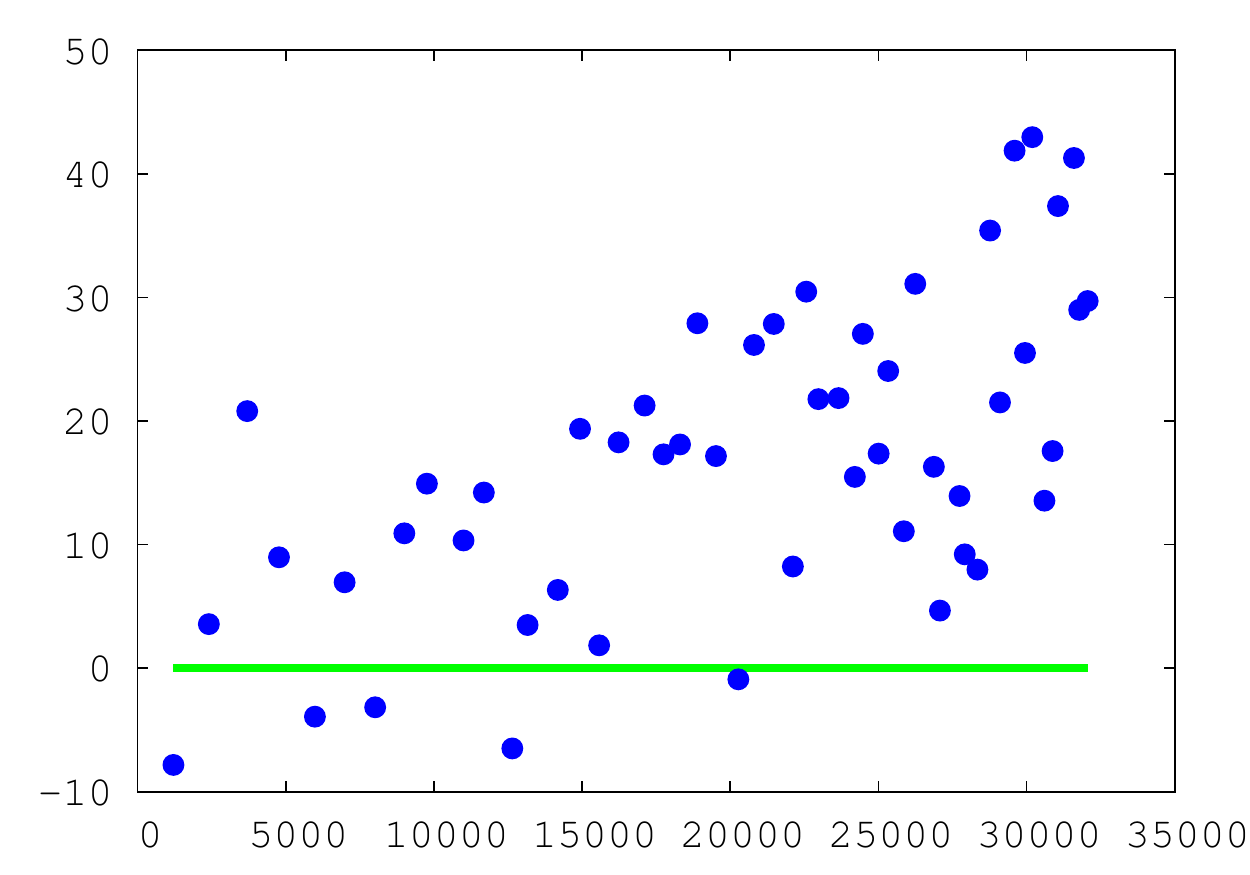}
 \hspace{-0.4cm} & \hspace{-0.4cm} \includegraphics[width=0.15\textwidth]{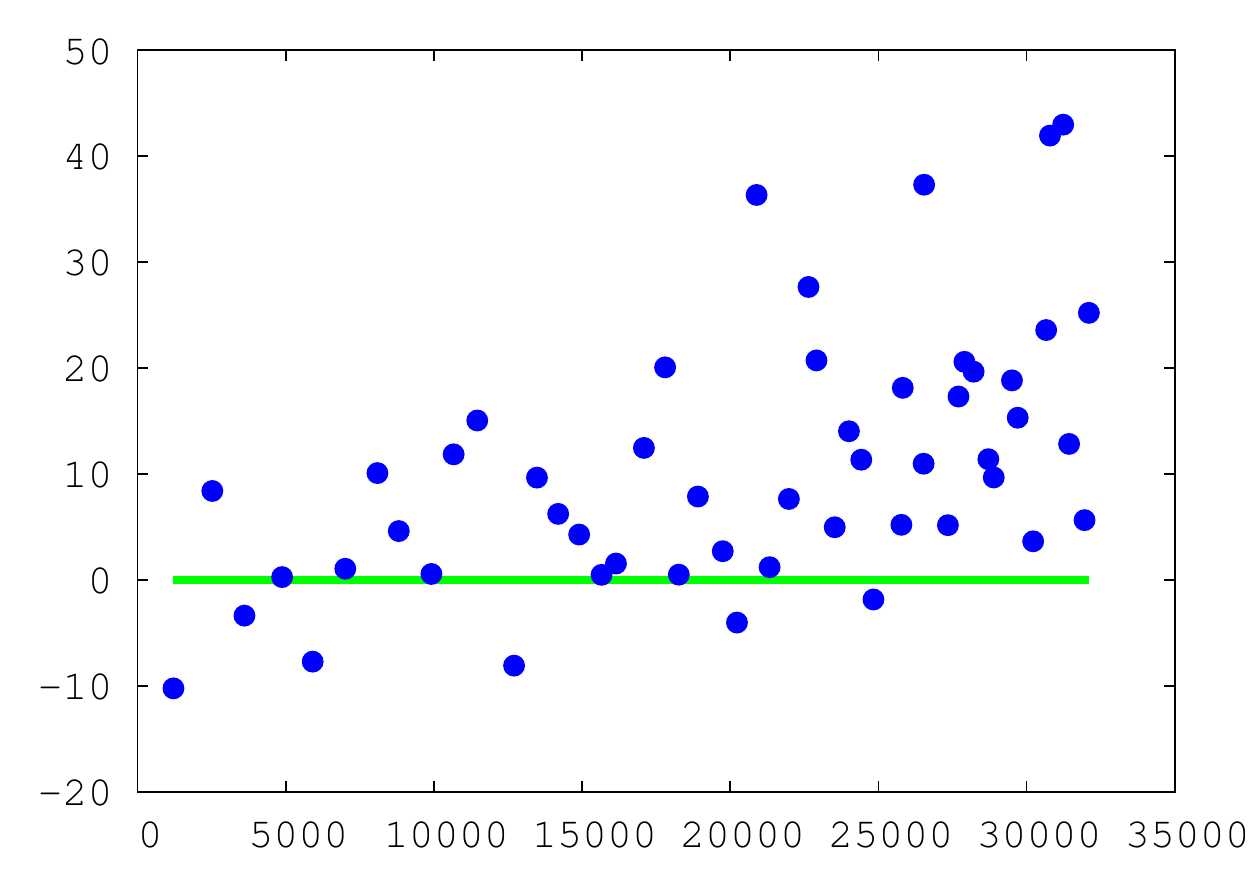}
 \hspace{-0.4cm} & \hspace{-0.4cm} \includegraphics[width=0.15\textwidth]{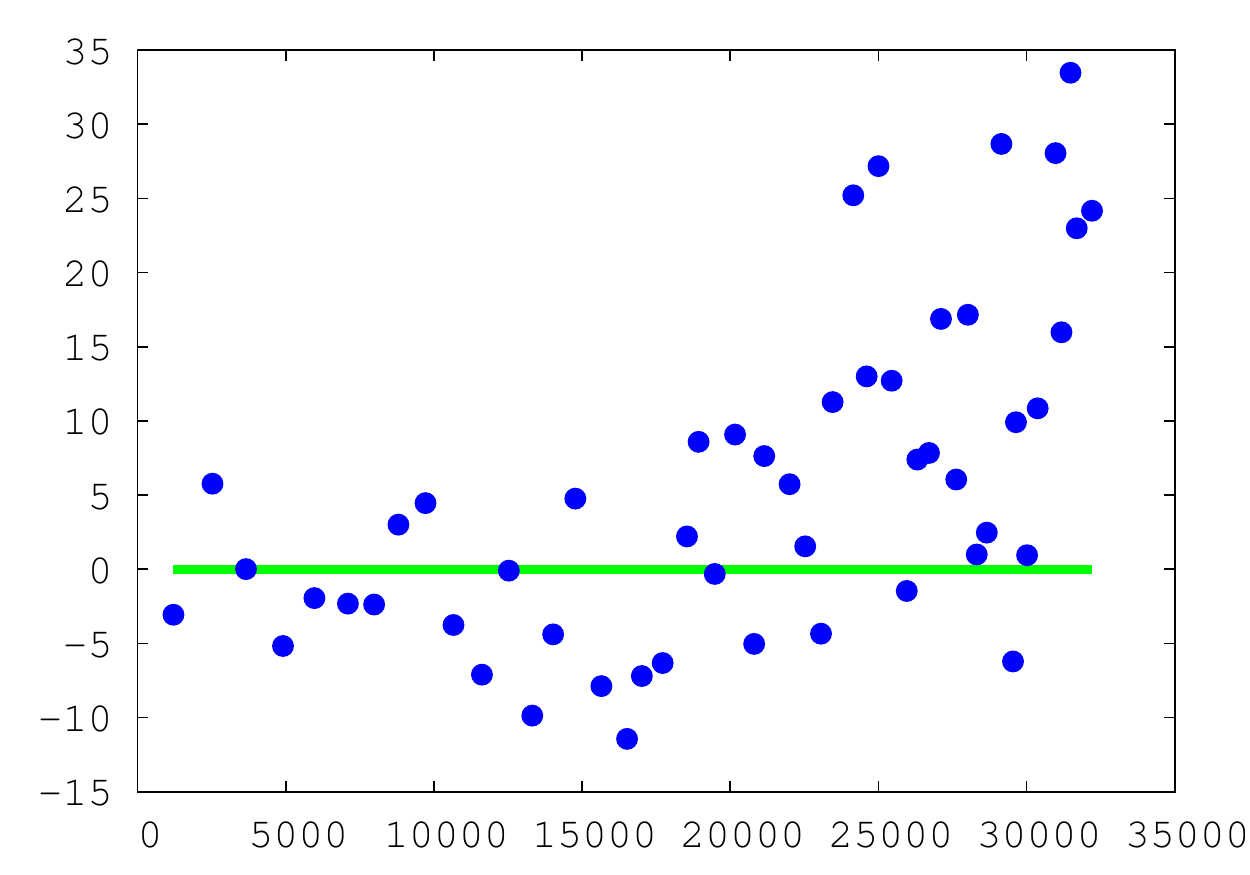}
 \hspace{-0.4cm} & \hspace{-0.4cm} \includegraphics[width=0.15\textwidth]{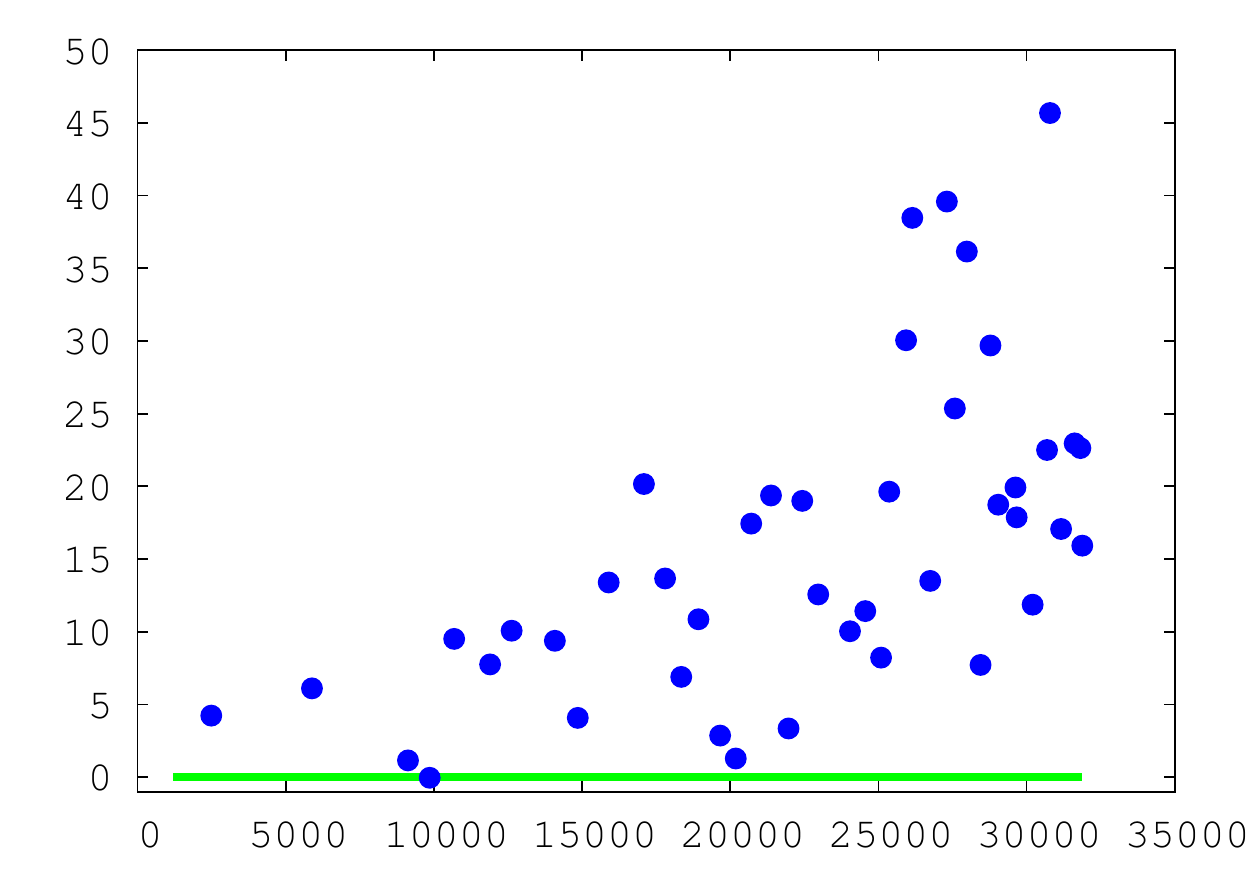}
 \hspace{-0.4cm} & \hspace{-0.4cm} \includegraphics[width=0.15\textwidth]{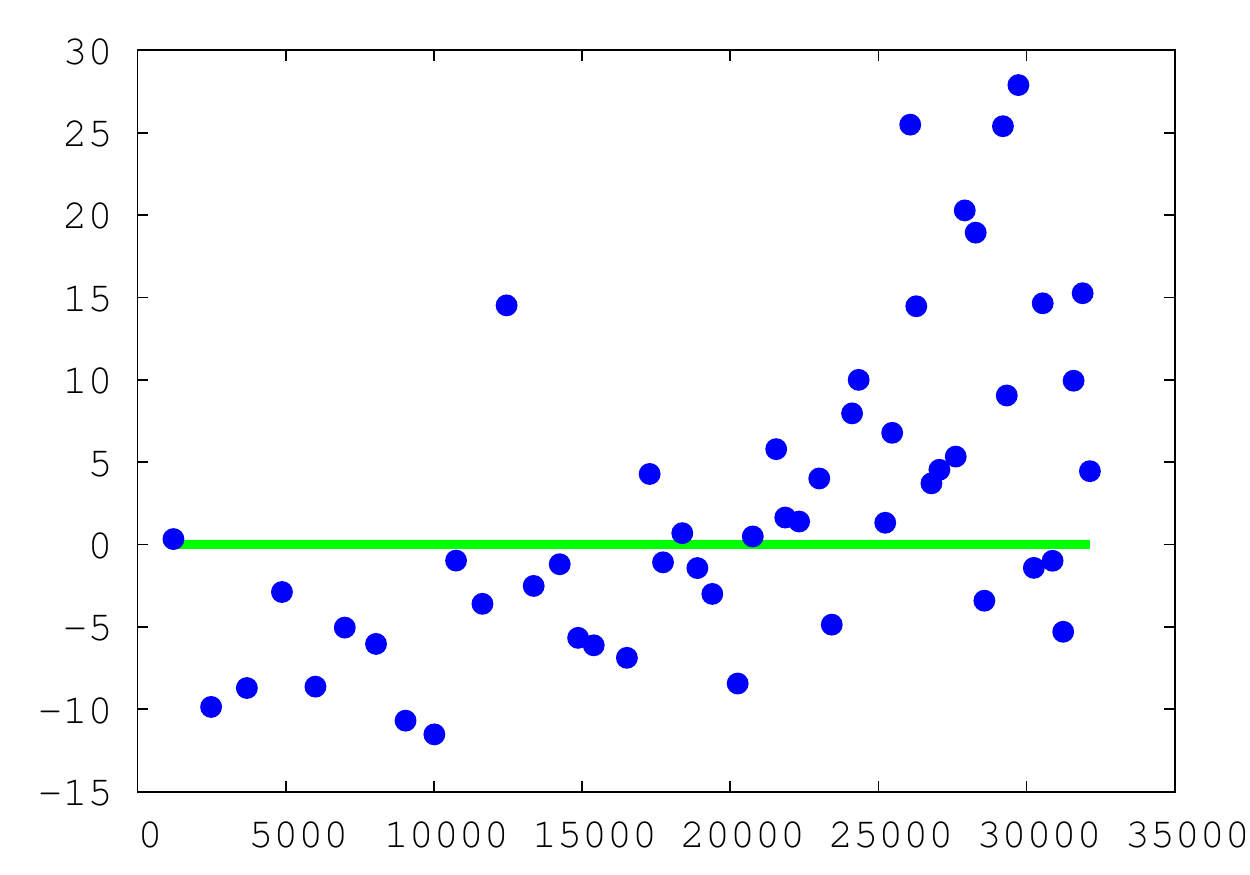}
 \\ 
$k=4$ \hspace{-0.4cm} & \hspace{-0.4cm} $k=5$ \hspace{-0.4cm} & \hspace{-0.4cm}$k=6$ 
\hspace{-0.4cm} & \hspace{-0.4cm} $k=7$ \hspace{-0.4cm} & \hspace{-0.4cm}$k=8$ 
\hspace{-0.4cm} & \hspace{-0.4cm} $k=9$ \hspace{-0.4cm} & \hspace{-0.4cm} $k=10$ \\ \hline
\hline 
\end{tabular}
\begin{tabular}{ccccccc}
\hline 
 \includegraphics[width=0.15\textwidth]{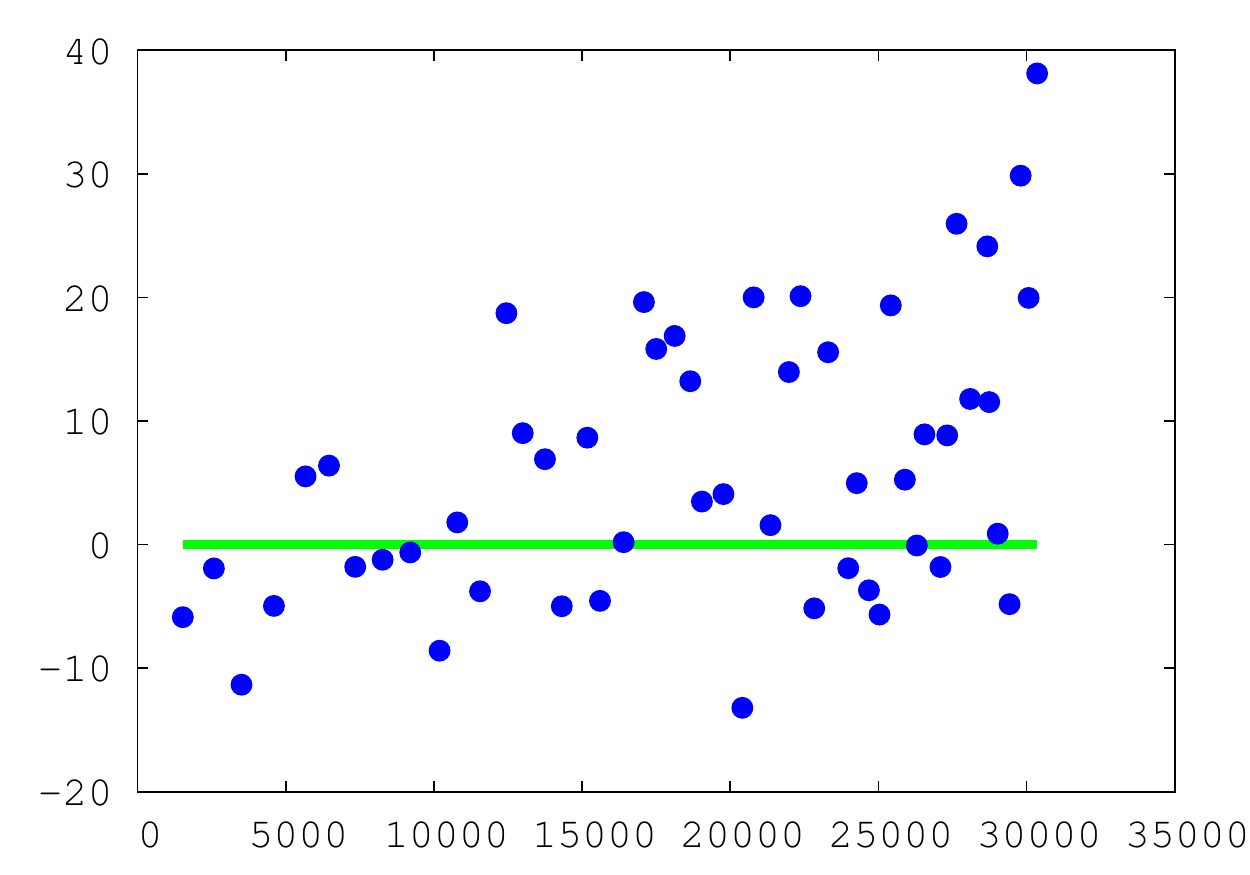}
 \hspace{-0.4cm} & \hspace{-0.4cm} \includegraphics[width=0.15\textwidth]{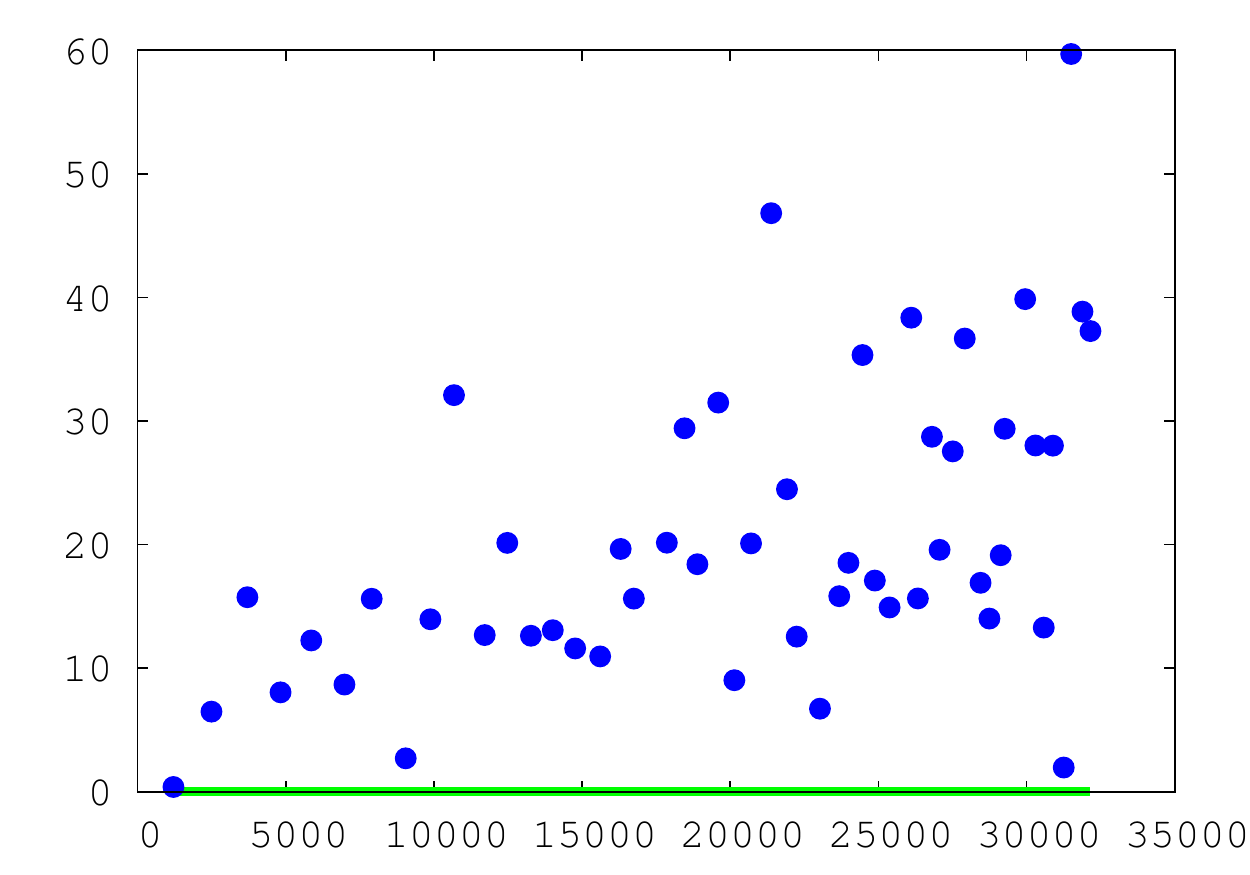}
 \hspace{-0.4cm} & \hspace{-0.4cm} \includegraphics[width=0.15\textwidth]{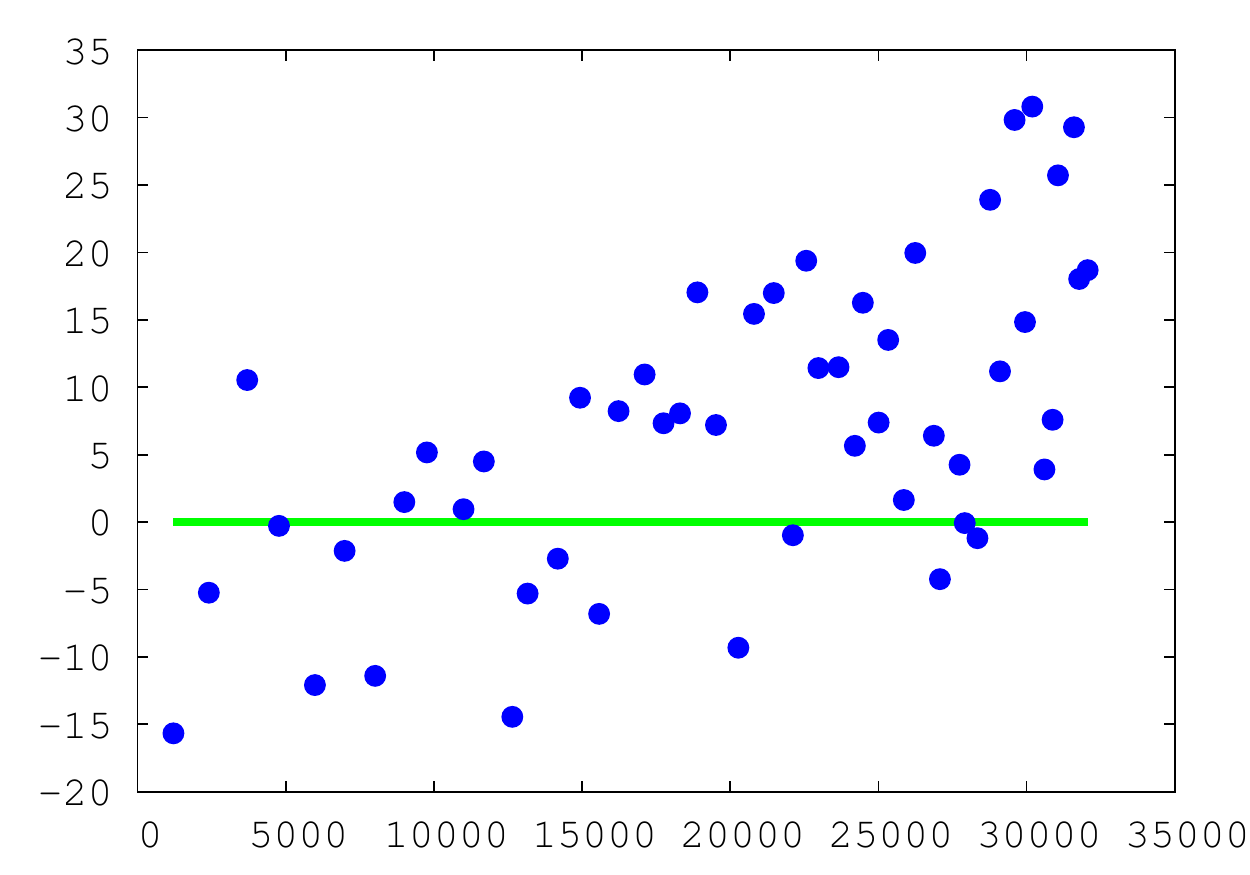}
 \hspace{-0.4cm} & \hspace{-0.4cm} \includegraphics[width=0.15\textwidth]{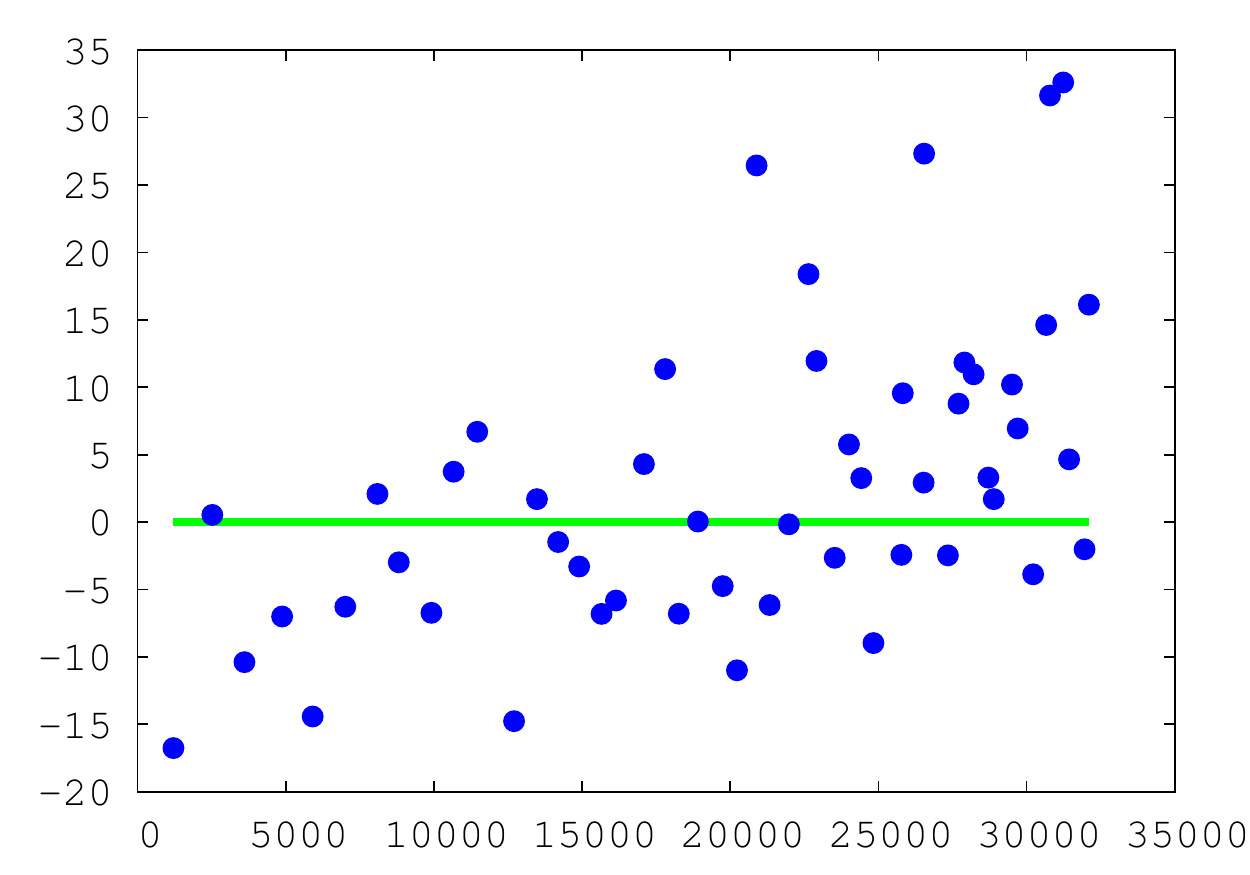}
 \hspace{-0.4cm} & \hspace{-0.4cm} \includegraphics[width=0.15\textwidth]{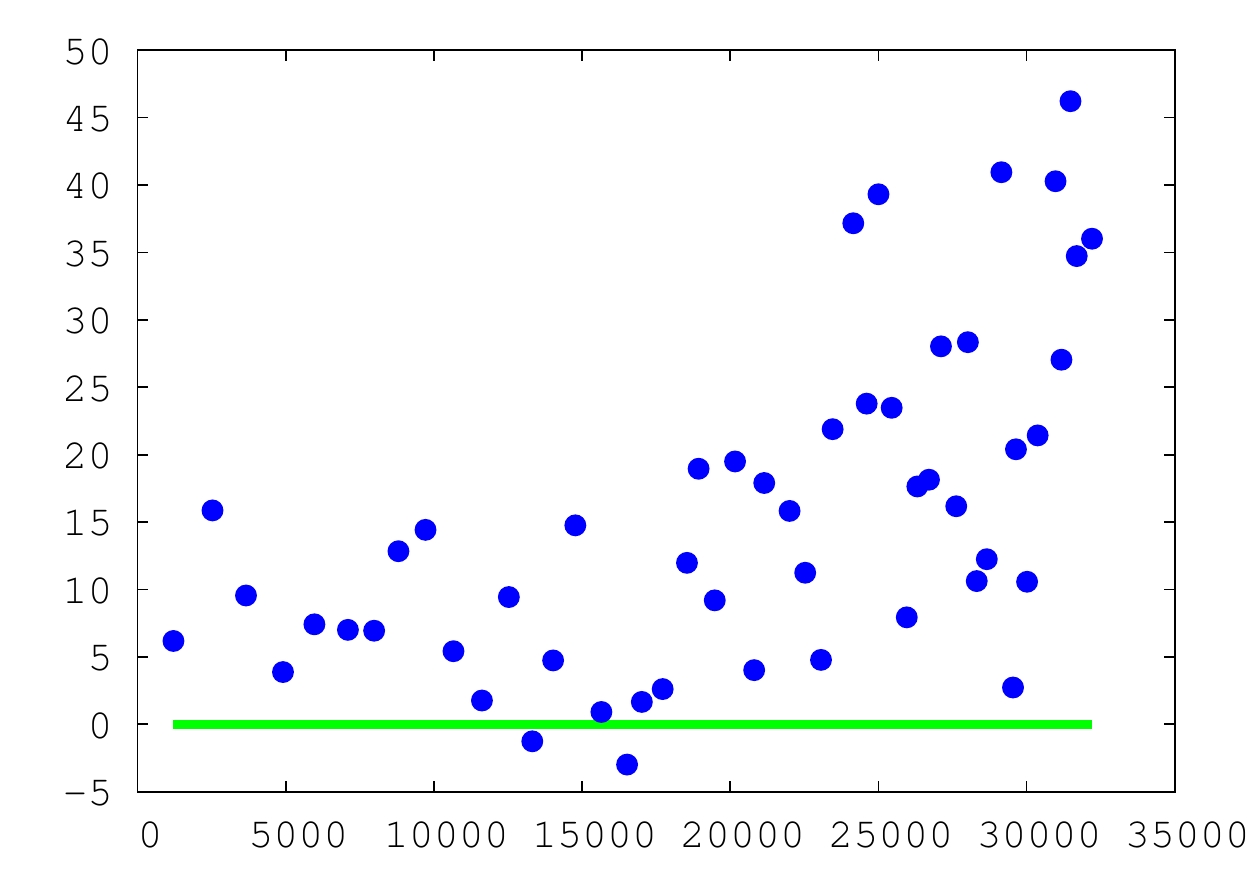}
 \hspace{-0.4cm} & \hspace{-0.4cm} \includegraphics[width=0.15\textwidth]{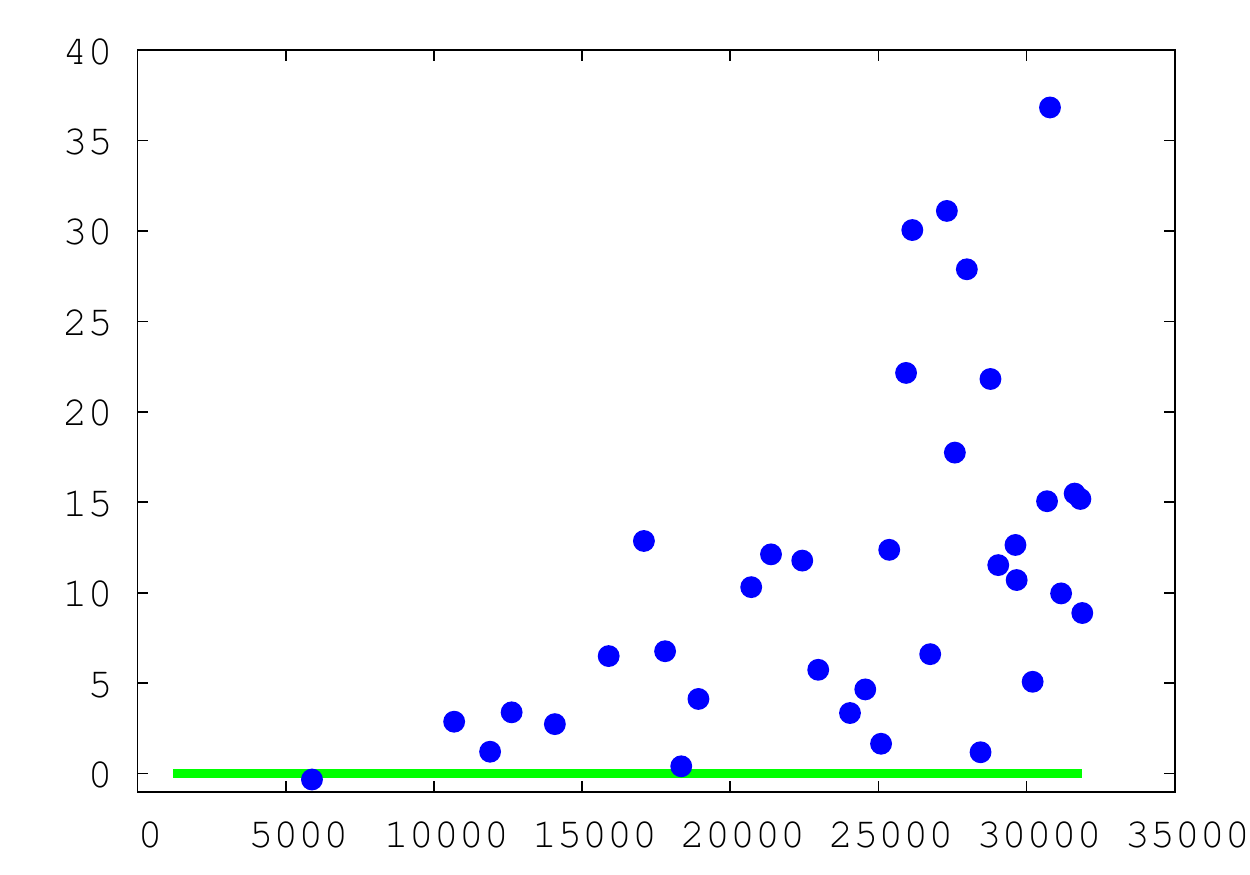}
 \hspace{-0.4cm} & \hspace{-0.4cm} \includegraphics[width=0.15\textwidth]{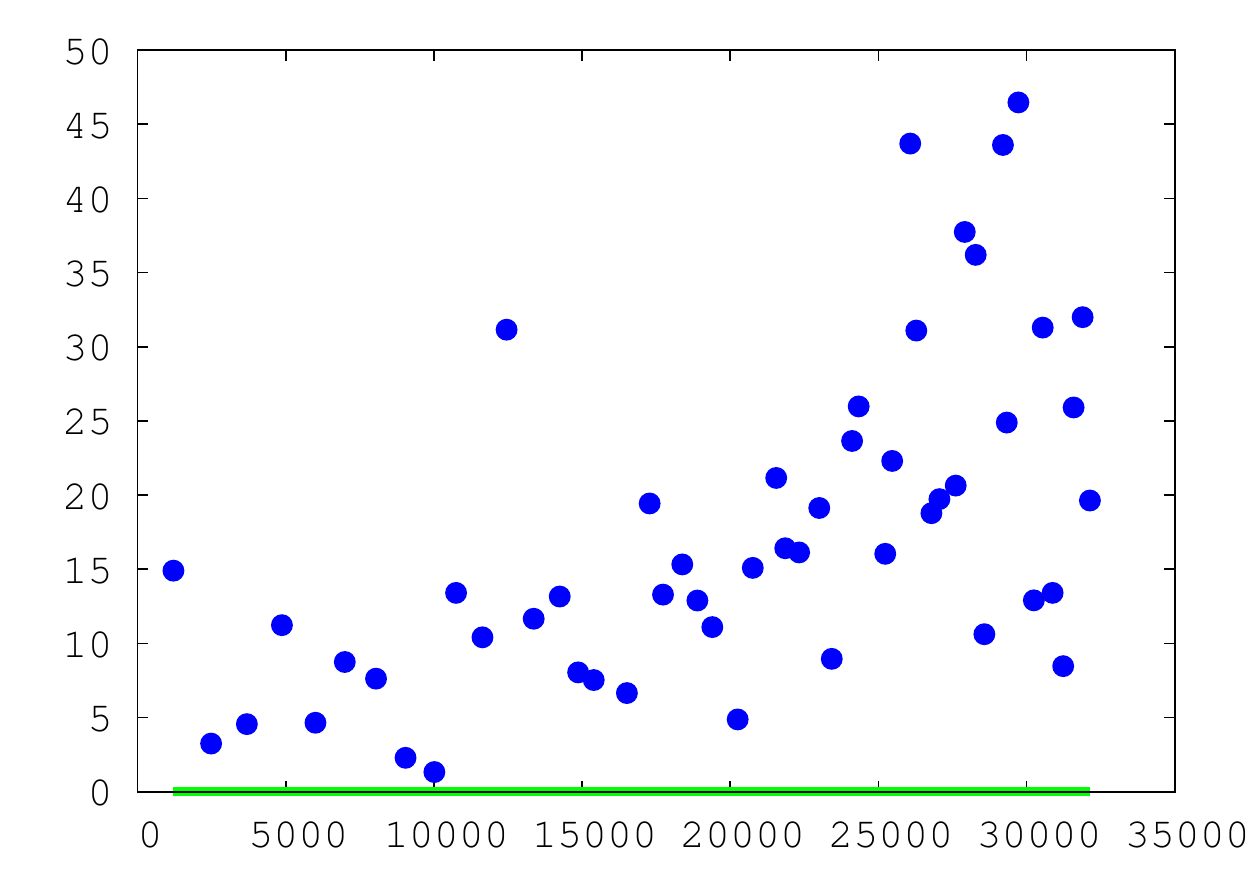}
 \\ 
$k=4$ \hspace{-0.4cm} & \hspace{-0.4cm} $k=5$ \hspace{-0.4cm} & \hspace{-0.4cm}$k=6$ 
\hspace{-0.4cm} & \hspace{-0.4cm} $k=7$ \hspace{-0.4cm} & \hspace{-0.4cm}$k=8$ 
\hspace{-0.4cm} & \hspace{-0.4cm} $k=9$ \hspace{-0.4cm} & \hspace{-0.4cm} $k=10$ \\ \hline
\hline 
\end{tabular}
\caption{Experiments on real world data "EuropeDiff" with $N=30$
  simulated peers. Top plot: Plots
  corresponding to Figure \ref{f-par1}. Middle and bottom plot ranges: plots
  corresponding respectively to Figures \ref{t-expDKM1} and \ref{t-expDKM2}.}\label{f_ED}
\end{figure}

\begin{figure}[t]
    \centering
\includegraphics[width=0.5\textwidth]{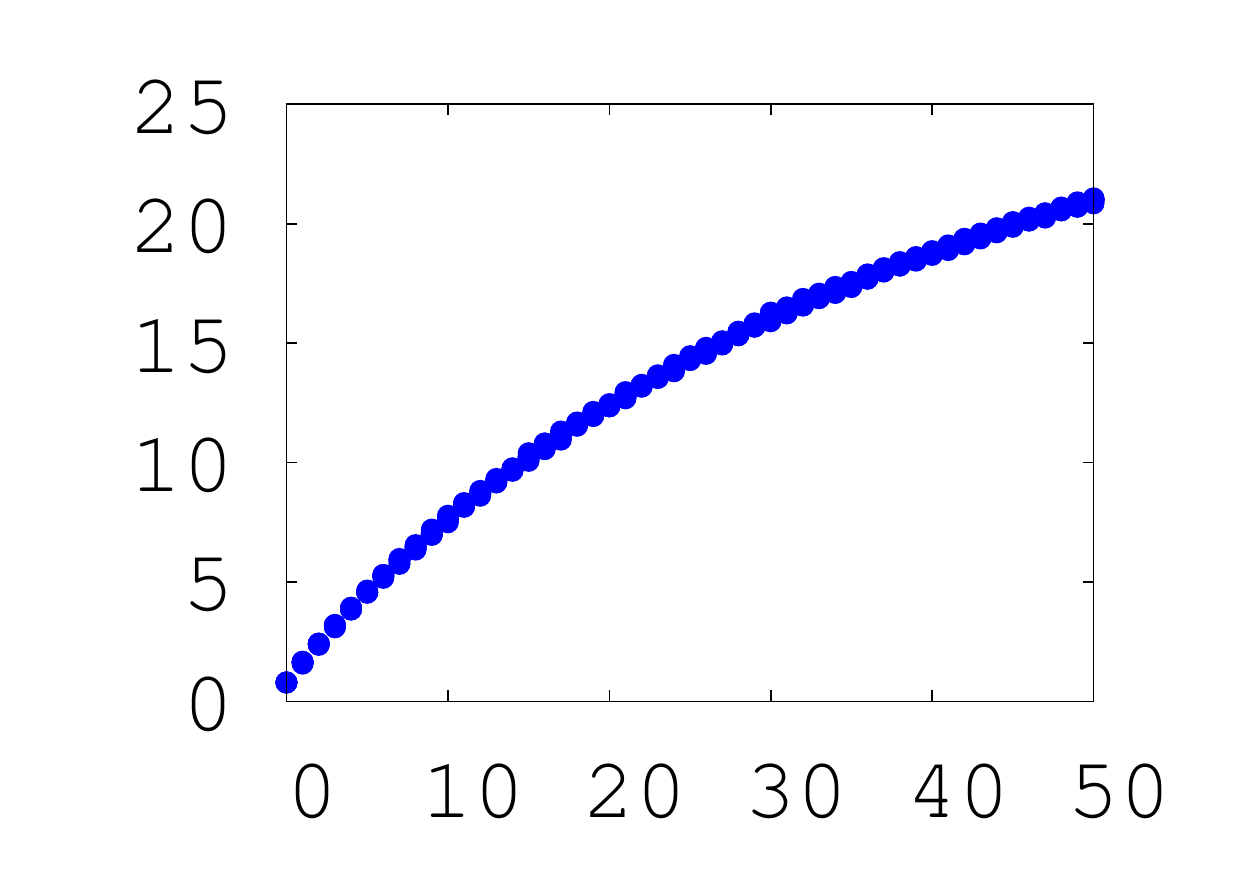}
\caption{Experiments on real world data "EuropeDiff" with $N=40$
  simulated peers. Plot corresponding to Figure \ref{f-spread1}.}\label{f-spread_ED}
\end{figure}

To further address the way the spread of Forgy nodes affects results,
we have used another real world data with highly non-uniform
distribution, Mopsi-Finland locations\footnoteref{urlf} ($m = 13467, d =
2$). We have sampled peers using two different schemes for the peer
centers: $k$-means++ and Forgy. In this latter initialisation, we just
pick peer centers at random.
In the former $k$-means++ initialisation, the
initial peer centers are much more evenly geographically spread before we complete
the peers data with the closest points. They remain more spread after
the $p\%$ uniform displacement of data between peers, as shown on the
top plots of Figure \ref{f-Finland}. What is interesting about this
data is that it displays that if peers' data are indeed geographically
located, then \protectedKMPP~is competitive up to quite reasonable values of
$p \leq 20\%$ (depending on $k$). That, is \protectedKMPP~works well
when each peer aggregates 80 $\%$ data which is reasonably
"localized in the domain" and 20 $\%$ data which can be located
everywhere in the domain.

\begin{figure}[t]
    \centering
\begin{tabular}{c|c}\hline\hline
Forgy initial peer centers & $k$-means++ initial peer centers\\ \hline \hline
\includegraphics[width=0.5\textwidth]{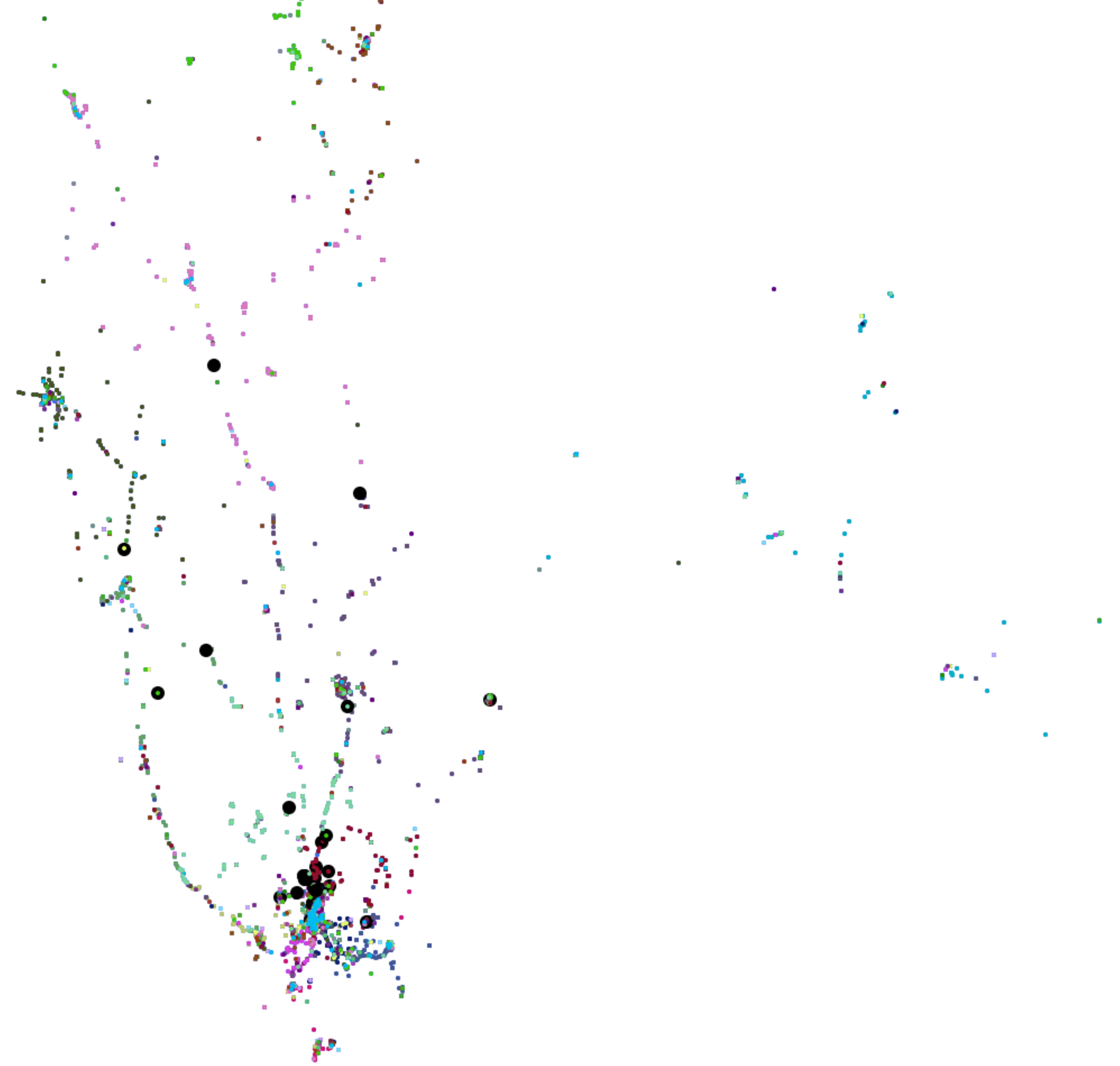}
& \includegraphics[width=0.5\textwidth]{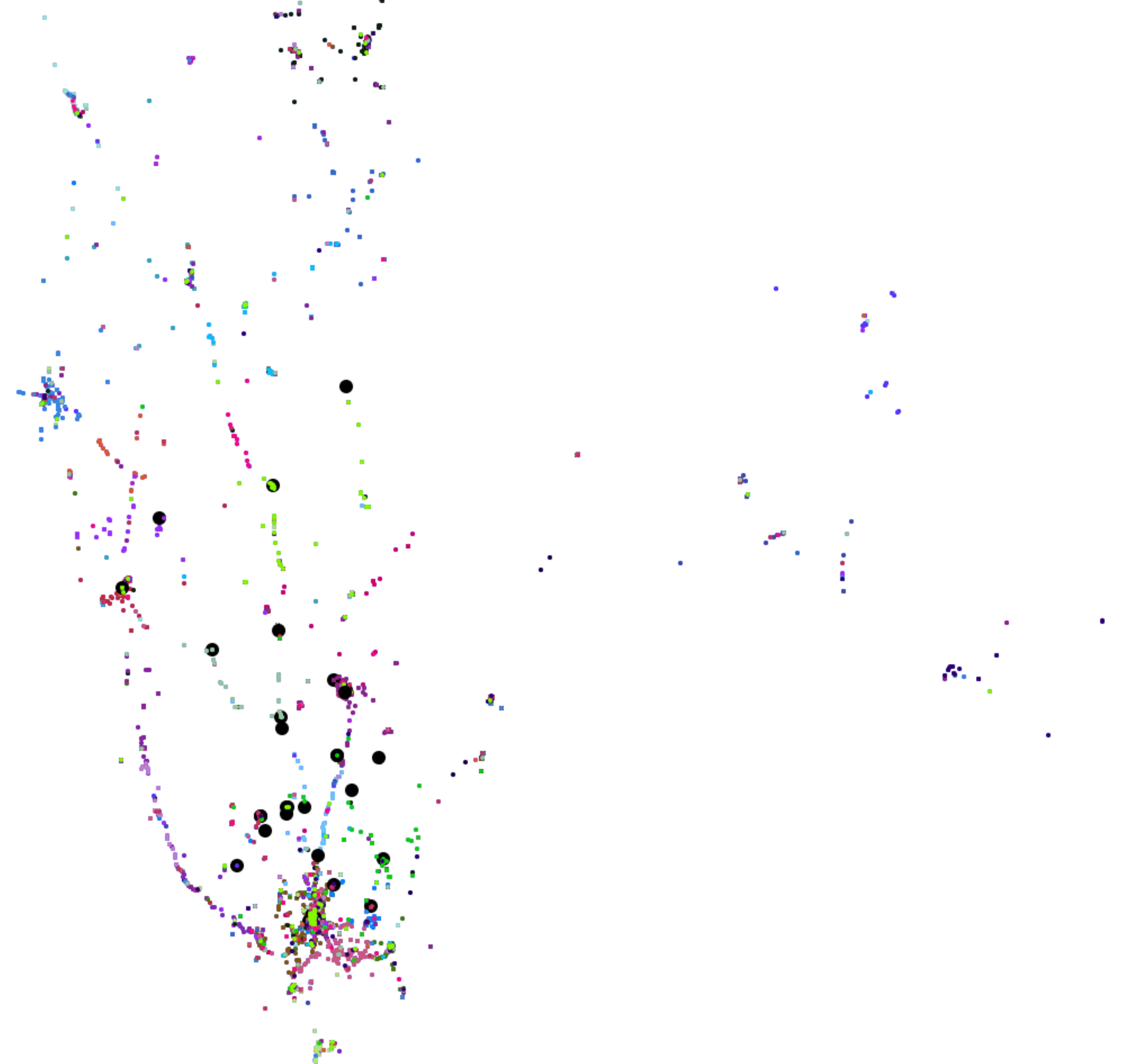}\\
\hline
\includegraphics[width=0.5\textwidth]{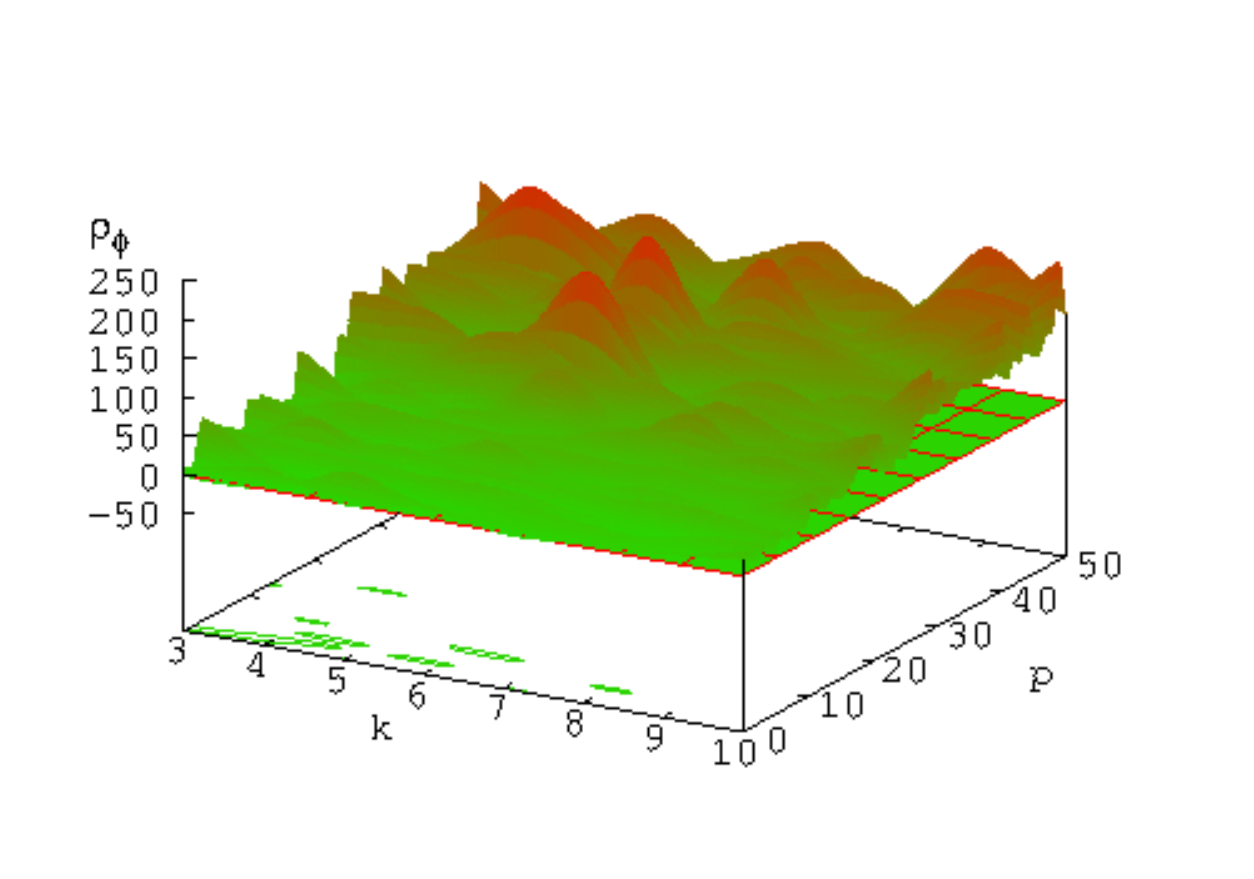}
& \includegraphics[width=0.5\textwidth]{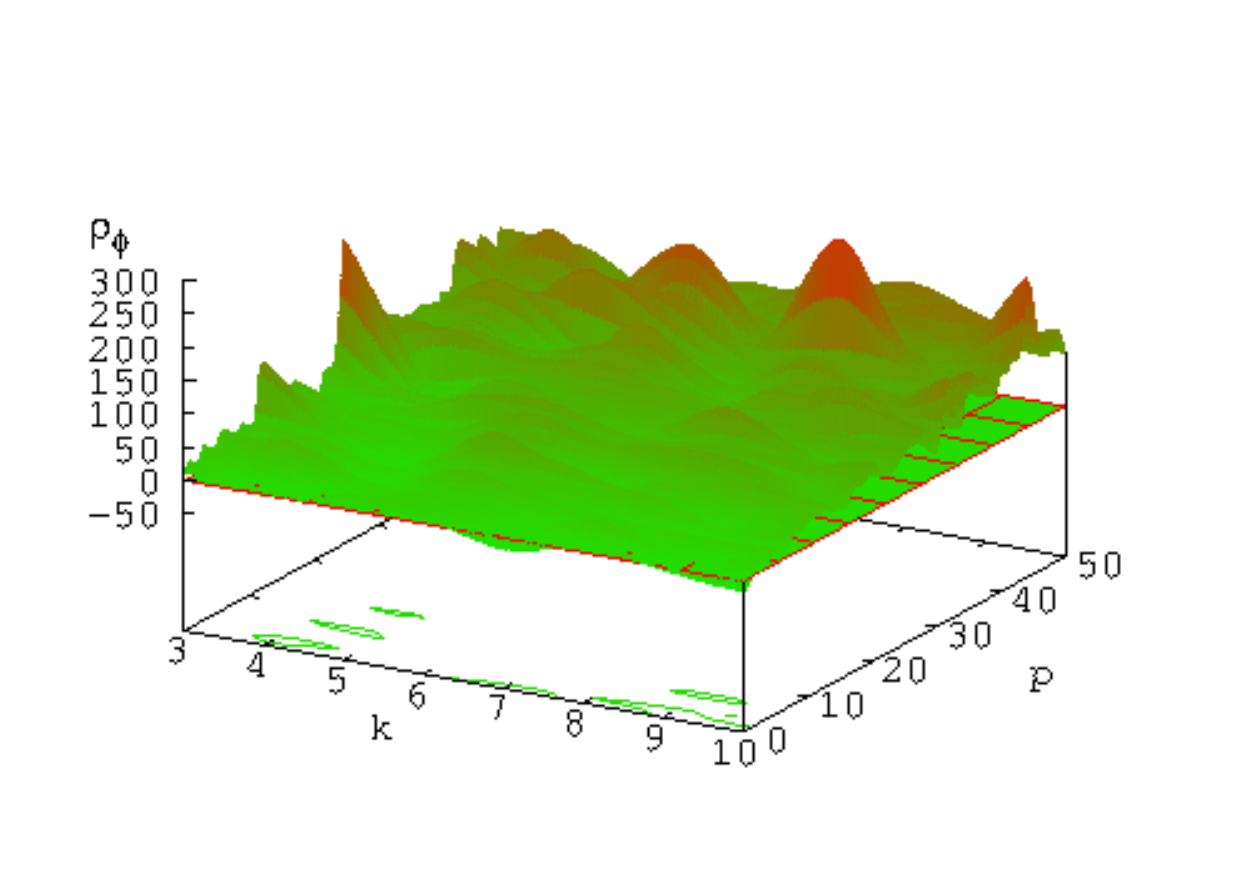}\\
\hline
\includegraphics[width=0.5\textwidth]{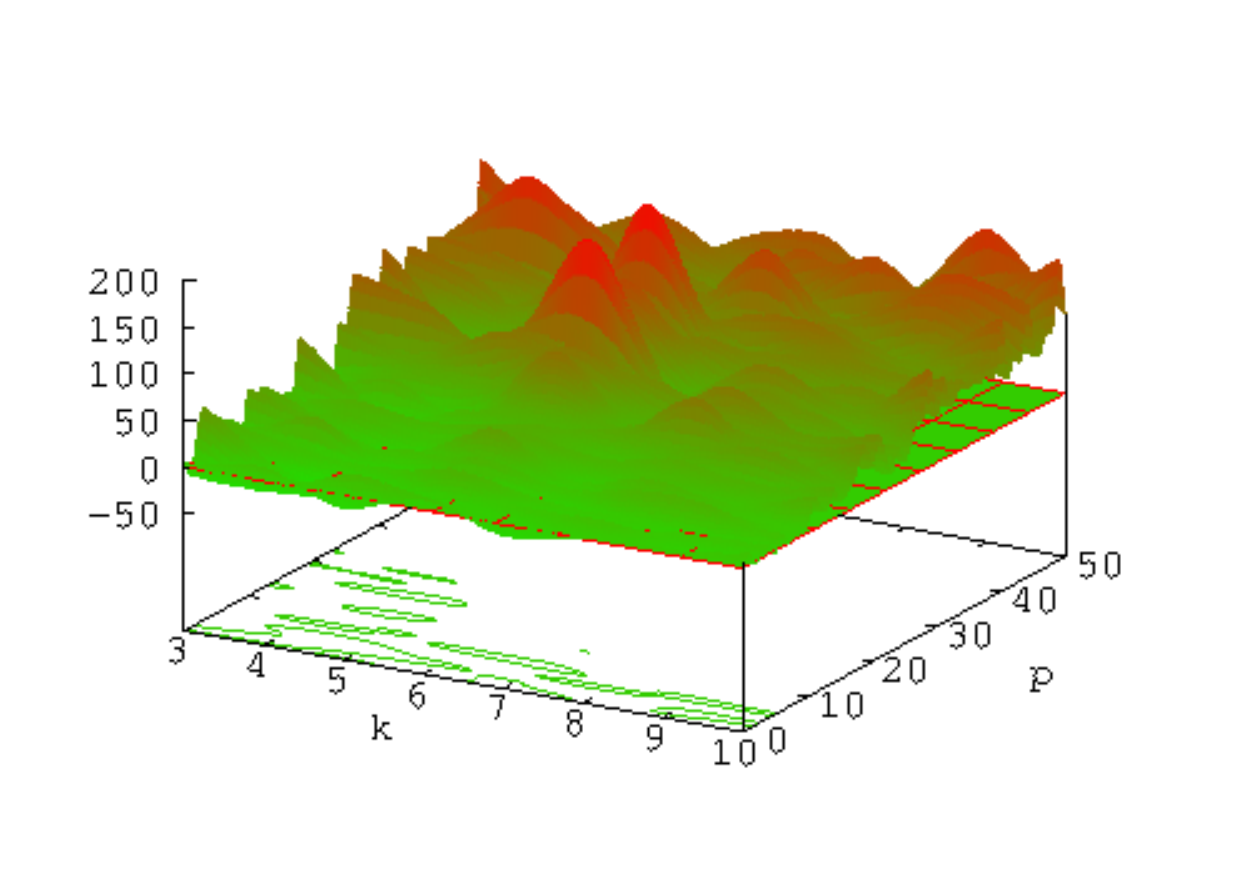}
& \includegraphics[width=0.5\textwidth]{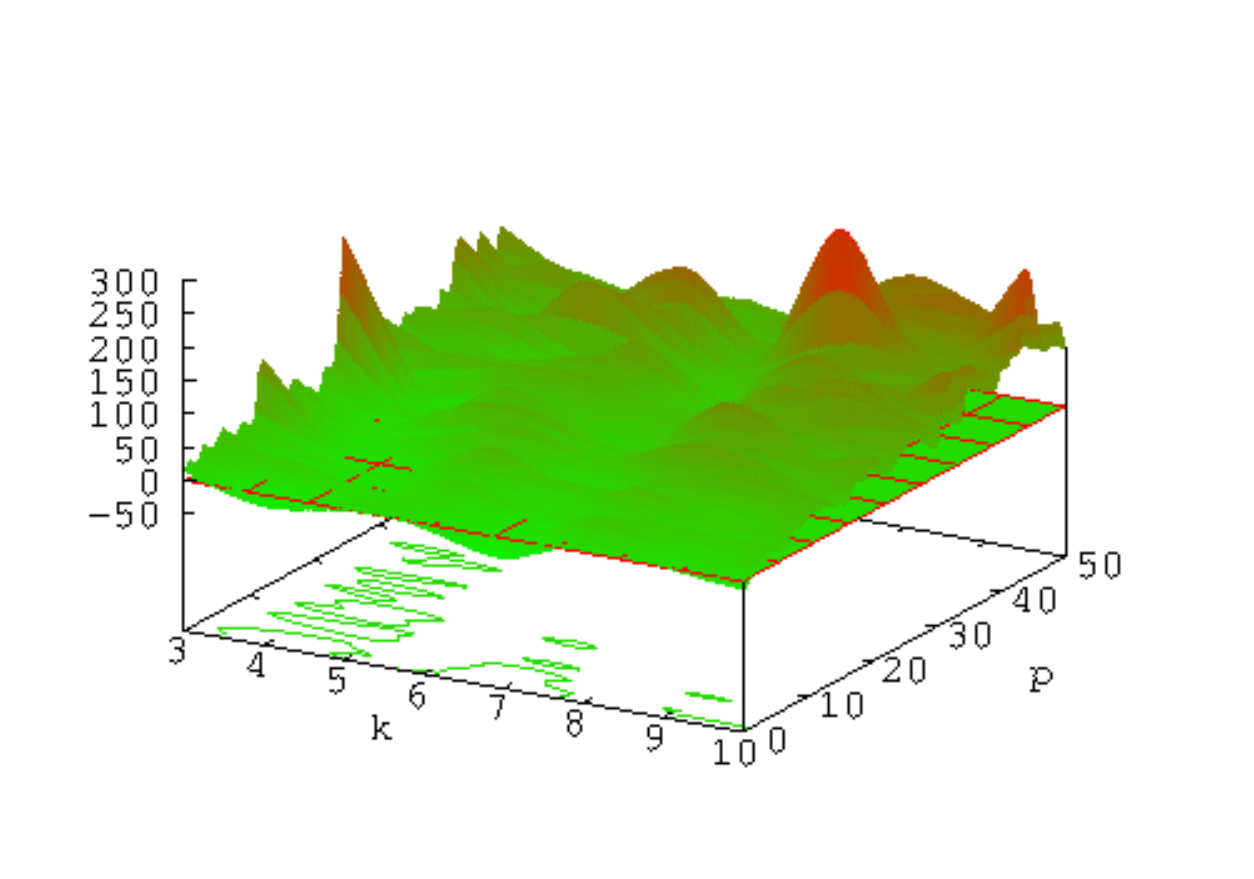}\\ \hline
\end{tabular}
\caption{Mopsi-Finland locations data --- Top: peer centers (big black dots) after $p=50\%$ moving
  probability changes in data. Remark from the right plot ($k$-means++
initial peer centers) that peer data are less "attracted"
towards the highest density regions. Center: plots of
$\uprho_\phi(\mbox{$k$-means++})$. Bottom: plots of
$\uprho_\phi(\mbox{$k$-means$_{\tiny{\mbox{$\|$}}}$})$. }\label{f-Finland}
\end{figure}

\subsection*{Experiments with $k$-\means~and GUPT}\label{exp_dp1}

Among the state-of-the-art approaches against which we could compare
$k$-\means, there are two major contenders, PINQ \citep{mPI} and GUPT
\citep{mtsscGP}. Even when PINQ is a broad system, we switched our
preferences to GUPT for the following reasons. The performance of
$k$-means based on PINQ relies on two principal factors: 
the initialisation (like in the non differentially private version)
and the number of iterations. To compete against heavily tuned
specific applications, like $k$-\means, this scheme requires
substantial work for its optimisation. For example, if one allocates
part of the privacy budget to release a differential private
initialisation, the noise has to be proportional to the domain width,
which would
release poor centers. Also, generating points
uniformly at random from the domain, to obtain data-independent initial
centers, yields to a poor initialisation. Finally, the number of
iterations has to be tuned very carefully: if too small, the algorithm
keeps poor solutions; if too large, the number of iteration increase
the added noise for privacy and harms PINQ's final accuracy. We thus
chose GUPT. 
$k$-means implemented in the GUPT proceeds the following way: 
the dataset is cut in a certain number of blocks $\ell$ (following
\citep{mtsscGP}, we fix $\ell = m^{0.4}$ in our experiments), the usual $k$-means
algorithm is performed on each block. 
Before releasing the final centroids, results are aggregated and a
noise is applied. Finally, we also compare against the vanilla
approach of {\em Forgy Initialisation} using the Laplace
mechanism. The noise rate (\textit{i.e.}, standard deviation) is then proportional to $\propto kR/\epsilon$ (we
do not run $k$-means afterwards, hence the privacy budget remains
``small''). In comparison, GUPT adds noise $\propto kR/(\ell \epsilon)$
at the end of this aggregation process. Note that we disregard the
fact that our data are multidimensional, which should require a
finer-grained tuning of $\ell$, and choose to rely on the $\ell =
m^{0.4}$ suggestion from \citep{mtsscGP}.

\begin{table}[t]
    \centering
\begin{center}
\begin{tabular}{|c|c|c|c|r|r|r|}
\hline 
Dataset & $m$ & $d$ & $k$ & \multicolumn{1}{|c|}{$\tilde{\epsilon}$} &
\multicolumn{1}{|c|}{${\uprho'_\phi(\mbox{F-\DP})}$} & \multicolumn{1}{|c|}{${\uprho'_\phi(\mbox{GUPT})}$} \tabularnewline
\hline 
\hline 
\multirow{3}{*}{LifeSci} & \multirow{3}{*}{26733} & \multirow{3}{*}{10} & $2$ & $8.5$ & $311$ & $1.6$\tabularnewline
\cline{4-7} 
 &  &  & $3$ & $4.4$ & $172$ & $0.4$\tabularnewline
\cline{4-7} 
 &  &  & $4$ & $0.6$ & $6$ & $0.02$\tabularnewline
\hline 
\multirow{2}{*}{Image} & \multirow{2}{*}{34112} & \multirow{2}{*}{3} & $2$ & $12.6$ & $300$ & $4.8$\tabularnewline
\cline{4-7} 
 &  &  & $3$ & $3.2$ & $77$ & 0.9\tabularnewline
\hline 
\multirow{7}{*}{EuropeDiff} & \multirow{7}{*}{$169308$} & \multirow{7}{*}{2} & $2$ & $19.0$ & $1200$ & $46.1$\tabularnewline
\cline{4-7} 
 &  &  & $3$ & $21.0$ & $3120$ & $66.5$\tabularnewline
\cline{4-7} 
 &  &  & $4$ & $18.0$ & $3750$ & $55.0$\tabularnewline
\cline{4-7} 
 &  &  & $5$ & $14.0$ & $4000$ & $51.0$\tabularnewline
\cline{4-7} 
 &  &  & $6$ & $10.4$ & $5000$ & $36.0$\tabularnewline
\cline{4-7} 
 &  &  & $7$ & $6.6$ & $2600$ & $26.0$\tabularnewline
\cline{4-7} 
 &  &  & $8$ & $1.8$ & $350$ & $2.0$\tabularnewline
\hline 
\end{tabular}
\caption{Comparison of $k$-\means, Forgy Initialisation differentially
  private (F-\DP) and GUPT on the real world domains. On each domain, we
  compute ratio $\uprho'_\phi$ of the clustering potential of the contender to that of
  $k$-\means, a value $>1$ indicating that $k$-\means~is better. The
  potential of each algorithm has been averaged over 30
  runs. $\tilde{\epsilon}$ is given in eq. (\ref{defepsilont}).}\label{t-comp}
\end{center}
\end{table}

\paragraph{$\hookrightarrow$ Comparison on real world domains}  Our domains consist of $3$ real-world datasets\footnoteref{urlf}. Lifesci contains
the value of the top 10 principal components for a chemistry or biology
experiment. Image is a 3D dataset with RGB vectors, and finally
EuropeDiff is the differential coordinates of Europe map.

 Table \ref{t-comp} presents the extensinve results obtained, that are
 averaged in the paper's body. We have fixed
$\epsilon = 1$ in the differentially privacy parameters. The column
$\tilde{\epsilon}$ (eq. (\ref{defepsilont})) provides the differential privacy
parameter which is equivalent from the protection
standpoint, but exploits the computation of $\updeltaw,
\updeltas$ (which we compute exactly, and not in a randomized way like
in the experiments on Theorem \ref{thdp1} above) and
ineq. (\ref{eqqq1}). Therefore,
each time $\tilde{\epsilon} > \epsilon$ (=1 in our applications), it means
that our analysis brings a sizeable advantage over ``raw protection''
by Laplace
mechanism (in our
application we chose for $p_{{\ve{\mu}}_{\ve{a}}, {\ve{\theta}}_{\ve{a}}}$ a Laplace distribution). $R$ is
computed from the data by an upperbound of the smallest enclosing ball
radius. The results display several interesting
patterns. First, the largest the domain, the better we compare with
respect to the other algorithms. On EuropeDiff for example, we often
have the ratio of the potentials $\phi(\mathrm{GUPT}) /
\phi(\mbox{$k$-\means})$ of the order of \textit{dozens}. Also, the
performances of $k$-\means~degrade if $k$ increases, which is again
consistent with the ``good'' regime of Theorem \ref{thlabf}. 

\begin{sidewaysfigure}[t]
    \centering
\begin{center}
\begin{tabular}{|c|c||c|c|}\hline\hline
\includegraphics[width=0.2\textheight]{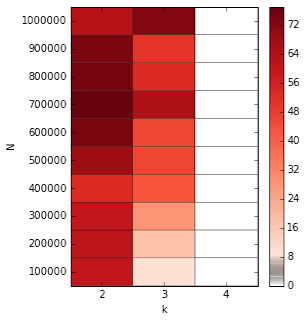}
&  \includegraphics[width=0.2\textheight]{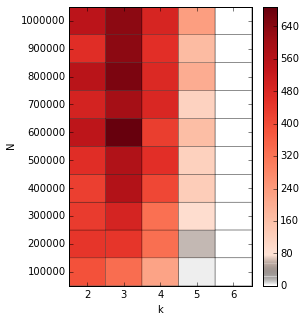}
&  \includegraphics[width=0.2\textheight]{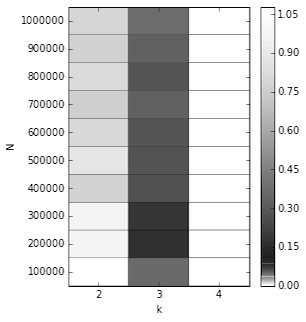}
&  \includegraphics[width=0.2\textheight]{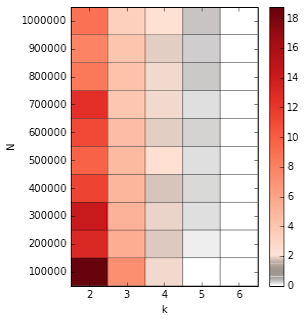}
\\ 
vs F, $d=2$ & vs F, $d=15$ & vs GUPT, $d=2$ & vs GUPT, $d=15$\\ \hline\hline
\end{tabular}
\caption{$k$-\means~vs Forgy initialisation differentially private
  and GUPT. We use ratio $\uprho'_\phi$ between the potential of the contender in
  (F-\DP, GUPT) over the potential of $k$-\means~(potentials are averaged
  30 times). The more red, the better is $k$-\means~with respect to
  the contender. Grey values indicate less positive outcomes for
  $k$-\means; white values indicate that $k$-\means~does not manage to
  find an $\epsilon'$ larger than $\epsilon$, and thus does not manage
  to put smaller noise rate than in the Laplace mechanism.}\label{f-heat}
\end{center}
\end{sidewaysfigure}

\paragraph{$\hookrightarrow$ Comparison on synthetic domains} The synthetic datasets contain points uniformly sampled on a unit
$d$-ball, in low dimension $d=2$ and higher dimension $d=15$ , we
generated datasets with size in $\left\{ 10^{5},10^{6}\right\}
$.

\end{document}